\documentclass[english,american]{IEEEtran}
\usepackage[T1]{fontenc}
\usepackage[utf8]{inputenc}
\usepackage{color}
\usepackage{array}
\usepackage{float}
\usepackage{textcomp}
\usepackage{bm}
\usepackage{multirow}
\usepackage{amsmath}
\usepackage{amsthm}
\usepackage{amssymb}
\usepackage{graphicx}

\makeatletter

\providecommand{\tabularnewline}{\\}
\floatstyle{ruled}
\newfloat{algorithm}{tbp}{loa}
\providecommand{\algorithmname}{Algorithm}
\floatname{algorithm}{\protect\algorithmname}

\theoremstyle{plain}
\newtheorem{thm}{\protect\theoremname}
\theoremstyle{plain}
\newtheorem{lem}[thm]{\protect\lemmaname}

\usepackage{amsmath,amssymb}

\usepackage{color}  
\usepackage{graphicx,psfrag,cite,subfigure}

\usepackage{flushend}

\author{

\IEEEauthorblockN{Zheng~Xing and Junting~Chen}

\IEEEauthorblockA{School of Science and Engineering (SSE) and Shenzhen Future Network of Intelligence Institute (FNii-Shenzhen) \\ The Chinese University of Hong Kong, Shenzhen, Guangdong 518172, China}
}


\usepackage[acronym]{glossaries}
\newcommand{\newac}{\newacronym}
\newcommand{\ac}{\gls}
\newcommand{\Ac}{\Gls}
\newcommand{\acpl}{\glspl}

\makeglossaries

\usepackage{enumitem}

\newac{ips}{IPS}{indoor positioning system}
\newac{wlan}{WLAN}{wireless local area network}
\newac{ap}{AP}{access point}
\newac{sc}{SC}{subspace clustering}
\newac{uwb}{UWB}{ultrawideband}
\newac{hmm}{HMM}{hidden Markov model}
\newac{ema}{EM}{expectation maximization}
\newac{lbs}{LBS}{location-based service}
\newac{gpc}{GPC}{Gaussian process classification}
\newac{gpca}{GPCA}{generalized PCA}
\newac{msl}{MSL}{multistage learning}
\newac{mppca}{MPPCA}{mixture of probabilistic PCA}
\newac{alc}{ALC}{agglomerative lossy compression}
\newac{ransac}{RANSAC}{random sample consensus}
\newac{fa}{FA}{factorization-based affinity}
\newac{gpcaa}{GPCAA}{gpca-based affinity}
\newac{slbf}{SLBF}{spectral local best-fit flats}
\newac{lsa}{LSA}{local subspace affinity}
\newac{llmc}{LLMC}{locally linear manifold clustering}
\newac{ssc}{SSC}{sparse subspace clustering}
\newac{lrr}{LRR}{low-rank representation}
\newac{scc}{SCC}{spectral curvature clustering}
\newac{mfb}{MFB}{matrix factorization-based}
\newac{mssc}{MSSC}{multi-view low-rank sparse subspace clustering}
\newac{dsc}{DSC}{deep subspace clustering}
\newac{sssc}{SSSC}{Stochastic sparse subspace clustering}
\newac{mdssc}{MDSSC}{}
\newac{nmi}{NMI}{normalized mutual information}
\newac{svm}{SVM}{support vector machine}
\newac{tdoa}{TDOA}{time difference of arrival}
\newac{gmm}{GMM}{Gaussian mixture model}
\newac{pca}{PCA}{Principal Component Analysis}
\newac{dnn}{DNN}{deep neural network}
\newac{mlp}{MLP}{multilayer perceptron}
\newac{admm}{ADMM}{alternating direction method of multipliers}
\newac{gan}{GAN}{generative adversarial networks}
\newac{rsrp}{RSRP}{reference signal received power}
\newac{ss}{SS}{synchronization signal}
\newac{fim}{FIM}{Fisher information matrix}
\newac{crlb}{CRLB}{Cramer-Rao Lower Bound}
\newac{1nn}{1-NN}{1 Nearest-Neighbor}
\newac{dft}{DFT}{discrete Fourier transform}
\newac{evd}{EVD}{Eigen Value Decomposition}
\newac{ssb}{SSB}{secondary synchronization block}
\newac{lstm}{LSTM}{long short-term memory}
\newac{ppp}{PPP}{Poisson Point Process}
\newac{nn}{NN}{Neural Network}
\newac{csi}{CSI}{channel state information}
\newac{mimo}{MIMO}{multiple-input multiple-output}
\newac{nlos}{NLOS}{non-line-of-sight}
\newac{snr}{SNR}{signal-to-noise ratio}
\newac{bs}{BS}{base station}
\newac{rssi}{RSSI}{received signal strength indicator}
\newac{aoa}{AoA}{angle of arrival}
\newac{mle}{MLE}{maximum likelihood estimation}
\newac{mse}{MSE}{mean squared error}
\newac{wrt}{w.r.t.}{with respect to}
\newac{uav}{UAV}{unmanned aerial vehicle}
\newac{rss}{RSS}{received signal strength}
\newac{sinr}{SINR}{signal-to-interference-and-noise ratio}
\newac{los}{LOS}{line-of-sight}
\newac{pdf}{PDF}{probability density function}
\newac{gps}{GPS}{global positioning system}
\newac{mae}{MAE}{mean absolute error}

\usepackage{babel}

\usepackage{array}

\usepackage{varwidth}

\usepackage{amsthm}

\theoremstyle{plain}

\theoremstyle{remark}

\theoremstyle{proposition}
\newtheorem{myprop}{Proposition}

\makeatother

\usepackage{babel}
\addto\captionsamerican{\renewcommand{\lemmaname}{Lemma}}
\addto\captionsamerican{\renewcommand{\theoremname}{Theorem}}
\addto\captionsenglish{\renewcommand{\algorithmname}{Algorithm}}
\addto\captionsenglish{\renewcommand{\lemmaname}{Lemma}}
\addto\captionsenglish{\renewcommand{\theoremname}{Theorem}}
\providecommand{\lemmaname}{Lemma}
\providecommand{\theoremname}{Theorem}

\begin{document}
\title{Blind Construction of Angular Power Maps in Massive MIMO Networks\thanks{Manuscript submitted December 21, 2024, revised June 7, 2025, and accepted September 26, 2025. (Corresponding author: Junting Chen. E-mail: juntingc@cuhk.edu.cn)}
\thanks{The work was supported in part by NSFC with Grant No. 62293482, the Basic Research Project No. HZQB-KCZYZ-2021067 of Hetao Shenzhen-HK S\&T Cooperation Zone, the NSFC with Grant No. 62171398, the Guangdong Basic and Applied Basic Research Foundation 2024A1515011206, the Shenzhen Science and Technology Program under Grant No. JCYJ20220530143804010 and No. KJZD20230923115104009, the Shenzhen Outstanding Talents Training Fund 202002, the Guangdong Research Projects No. 2017ZT07X152 and No. 2019CX01X104, the Guangdong Provincial Key Laboratory of Future Networks of Intelligence (Grant No. 2022B1212010001), and the Shenzhen Key Laboratory of Big Data and Artificial Intelligence (Grant No. SYSPG20241211173853027).} 
\thanks{Zheng~Xing and Junting~Chen are with the School of Science and Engineering (SSE), Shenzhen Future Network of Intelligence Institute (FNii-Shenzhen), and Guangdong Provincial Key Laboratory of Future Networks of Intelligence, The Chinese University of Hong Kong, Shenzhen, Guangdong 518172, China.}}
\maketitle
\begin{abstract}
\Ac{csi} acquisition is a challenging problem in massive \ac{mimo}
networks. Radio maps provide a promising solution for radio resource
management by reducing online \ac{csi} acquisition. However, conventional
approaches for radio map construction require location-labeled \ac{csi}
data, which is challenging in practice. This paper investigates unsupervised
angular power map construction based on large timescale \ac{csi}
data collected in a massive \ac{mimo} network without location labels.
A \ac{hmm} is built to connect the hidden trajectory of a mobile
with the \ac{csi} evolution of a massive \ac{mimo} channel. As a
result, the mobile location can be estimated, enabling the construction
of an angular power map. We show that under uniform rectilinear mobility
with Poisson-distributed \acpl{bs}, the \ac{crlb} for localization
error can vanish at any \acpl{snr}, whereas when \acpl{bs} are confined
to a limited region, the error remains nonzero even with infinite
independent measurements. Based on \ac{rsrp} data collected in a
real multi-cell massive \ac{mimo} network, an average localization
error of 18 meters can be achieved although measurements are mainly
obtained from a single serving cell.
\end{abstract}

\begin{IEEEkeywords}
Angular power map, massive \ac{mimo} network, trajectory recovery,
localization, \ac{csi} prediction
\end{IEEEkeywords}

\glsresetall

\section{Introduction}

Acquiring \ac{csi} is essential in beamforming, resource allocation,
and inter-cell interference mitigation in \ac{mimo} networks. However,
it becomes more and more challenging in estimating the \ac{csi} as
the number of antennas scale up in trending massive \ac{mimo} networks
\cite{DuDen:J22,LiuBha:J23,xing2022spectrum}. Radio maps, which associate
each mobile location with the corresponding \ac{csi}, provide a new
paradigm for \ac{csi} acquisition, tracking, and prediction \cite{ZenChe:J24,xing2024constructing,xing2022integrated,xing2025hmm}.
For example, the work \cite{KalTob:J23} enhances low-latency \ac{mimo}
communications using statistical radio maps that predict and select
communication parameters for reliability. Another work \cite{WuZen:J23}
introduces an environment-aware hybrid beamforming for mmWave massive
\ac{mimo} systems, utilizing channel knowledge maps to reduce real-time
training needs and improve communication rates with location accuracy
flexibility.

Nonetheless, it is very challenging to construct radio maps for multi-cell
\ac{mimo} communications, and one of the main challenges is the lack
of location-labeled \ac{mimo} channel data collected from real scenarios.
In reality, accurate location information of the mobile users is hard
to obtain. First, it requires the users to continuously report the
location to the network, which may not always be feasible due to privacy
concerns. Second, the localization accuracy is significantly affected
by \ac{nlos} conditions. Third, dedicated measurement campaign using
drive test is costly and less timely. Nevertheless, existing approaches
for radio map construction mostly require a massive amount of location-labeled
\ac{csi} measurement data. For instance, the work \cite{TimSub:J23}
necessitated gathering a large volume of \ac{csi} with location labels
to train a deep generative model. The research \cite{RogSan:J23}
involved collecting \ac{csi} with location labels via vehicles and
employs an \ac{lstm}-based neural network to construct road radio
maps. The work \cite{LiuChe:J23} collected a substantial amount of
\ac{csi} data with location labels in ground node-\ac{uav} networks
to construct radio maps. Moreover, the study \cite{ChiMas:J22} utilized
Kriging and covariance tapering techniques to construct radio maps
in massive \ac{mimo} systems using a small amount of location-labeled
\ac{csi} data. To relieve the requirement on the massive amount of
location-labeled \ac{csi} data for radio map construction, the work
\cite{WanZhu:J24} employed sparse sampling and Bayesian learning
inference techniques to construct radio maps using a limited amount
of location-labeled \ac{csi} data.

Is it possible to learn the radio map for \ac{mimo} communications
in a completely {\em unsupervised} manner based on the \ac{csi}
data without accurate location information? \ac{csi} captures how
the signal propagation is affected by the environment \cite{WanZhaQht:J24,ZhoChe:J24,xing2024calibration,VucHos:J24,Trajmat:J25,XinChe:C25,GuoLv:J24}.
Thus, although a single \ac{csi} realization tends to be random,
there exists some the geographical pattern of the \ac{csi} distribution,
which is determined by the environment and cellular \ac{mimo} network
topology \cite{GaoWan:J24,xing2024block,ZhaLei:J24,xing2024unsupervised}.
However, while the current cellular networks have been capturing extensive
volumes of \ac{csi} from numerous mobile users for \ac{mimo} transmission
and radio resource management, the instantaneous \ac{csi} is usually
discarded immediately after the transmission.

This paper proposes to leverage the unlabeled instantaneous \ac{csi}
to train a radio map model before the \ac{csi} data is discarded.
The ultimate goal is to build a data-driven model that can describe
the evolution of the \ac{csi} process. The physical intuition is
that although the \ac{csi} in a massive \ac{mimo} network is of
high dimension, the mobility of the corresponding user is contained
in a 3D physical world, and the dominant scattering environment, such
as building and vegetation, remains roughly and temporarily static.
Mathematically, while the \ac{csi} process evolves in a high dimensional
space, it can be embedded in a low dimensional latent space, which
represents the geographic environment. Similar ideas were attempted
in the channel charting literature. The work \cite{Stud:J18} proposed
a channel charting approach to reduce the dimensionality of \ac{csi}
data to two or three dimensions. The location labeling can be achieved
by rotating the 2D or 3D data in the latent space based on a limited
amount of location labels. Subsequent studies \cite{FerRau:J21,FerGui:J23}
have introduced various channel charting methods, focusing on techniques
such as auto-encoders, and Siamese networks. Yet, in the channel charting
literature, the latent space does not have to represent the physical
world, and hence, it is still of high interest to investigate whether
it is possible to embed a geographical model to describe the low-dimensional
latent space of the \ac{csi} sequence data.

More specifically, the aim of this paper is to investigate the following
fundamental questions: {\em (i) Can we establish a radio map model that maps the \ac{csi} to a latent space that has a clear physical meaning as the geographic area},
and {\em (ii) Is there any theoretical guarantee in recovering the user location from the \ac{csi}}.

For elaboration purpose, a multi-cell massive \ac{mimo} network is
considered, where a mobile user travels along an arbitrary trajectory
in the coverage area of the network and the network keeps acquiring
some partial \ac{csi} of the user. We propose a \ac{hmm} to establish
the connection between the mobility of the user and the \ac{csi}
evolution partially observed by the network. Under this context, to
our best knowledge, the most related work appeared in \cite{JunMoo:J15}
for indoor localization employs an \ac{hmm}-based method to reconstruct
trajectories from indoor \ac{rss} measurements. This method, however,
necessitates extensive \ac{rss} collection at each location within
the indoor layout and encounters mapping failures in symmetric building
configurations. Our prior work \cite{XinChe:C24} employed an \ac{hmm}
to extract coarse user locations from measurements and recover their
trajectories up to a region-level accuracy. While this approach allows
for the construction of a radio map using the measurements and the
estimated trajectory without any location labels, it does not address
trajectory recovery in \ac{mimo} scenarios.

To summarize, the following contributions are made:
\begin{itemize}
\item We develop a framework to recover the location labels from a sequence
of \ac{csi} measurements using Bayesian methods. Our approach does
not require the continuity of the \ac{csi} process, and hence, can
be adopted in \ac{nlos} scenario.
\item We establish theoretical results to show that under uniform rectilinear
mobility, the \ac{crlb} of the localization error can asymptotically
approach zero at any \ac{snr} if the \ac{bs} topology can be modeled
by a \ac{ppp} over a large enough region. By contrast, if the \acpl{bs}
are only deployed in a limited region, localization error cannot approach
zero even using infinite amount of independent measurement data.
\item We design efficient algorithms to solve for the joint trajectory recovery
and propagation parameter estimation problem, where an iterative log-transformation
technique is developed to solve a nonlinear regression problem that
contains coupled polynomial terms and exponential terms.
\item We conduct experiments on real dataset collected from a commercial
5G massive \ac{mimo} network. A mean localization error below 18
meters from \ac{ssb} \ac{rsrp} measurements is demonstrated, although
measurements from neighboring cells are mostly missing in the dataset.
When richer measurements are available, a localization error of 7
meters is demonstrated from a synthetic dataset. To demonstrate the
application of the angular power map, we predict the \ac{rsrp}, \ac{sinr},
and \ac{rssi} of the \ac{csi} beam using the constructed radio map.
The proposed radio map-assisted method achieves the lowest errors
in comparison with existing methods.
\end{itemize}
\selectlanguage{english}%
$\quad$The remainder of this paper is structured as follows. Section
\ref{sec:System-Model} presents the propagation model, measurement
mode, mobility model, and an \foreignlanguage{american}{\ac{hmm}}
formulation. Section \ref{sec:Fundamental-Limits} derives the theoretical
results for the \foreignlanguage{american}{\ac{crlb}} of the localization
error, considering both a limited region and an unlimited region for
the deployment of the \foreignlanguage{american}{\ac{bs}}. Section
\ref{sec:HMM-Based-CSI-Embedding} focuses on the development of the
trajectory recovery algorithm. Experimental results are provided in
Section \ref{sec:Numerical-Results}, and the paper concludes in Section
\ref{sec:Conclusion}.
\selectlanguage{american}%

\section{System model}

\label{sec:System-Model}

Consider a mobile user traveling along an arbitrary trajectory as
shown in Figure~\ref{fig:background}. Denote $\mathbf{x}\in\mathbb{R}^{2}$
as the location of the mobile user. The communication signal of the
mobile user can be acquired by $Q$ \acpl{bs} each equipped with
$N_{\text{t}}$ antennas, where the location of the $q$th \ac{bs}
is denoted as $\mathbf{o}_{q}\in\mathbb{R}^{2}$, $q=1,2,\dots,Q$.
Denote $d(\mathbf{x},\mathbf{o}_{q})=\|\mathbf{x}-\mathbf{o}_{q}\|_{2}$
as the distance between the user at location $\mathbf{x}$ and the
\ac{bs} at $\mathbf{o}_{q}$. While this work adopts a topology model
in 2D, the extension to 3D is straight-forward.

Here, we consider a power angular map for a narrowband \ac{mimo}
communication system as for easy elaboration. Note that a similar
methodology may apply to constructing a radio map for other \ac{csi}
statistics such as the \ac{aoa} and power delay profile.

\subsection{Propagation and Measurement Models}

\begin{figure}[t]
\begin{centering}
\includegraphics[width=1\columnwidth]{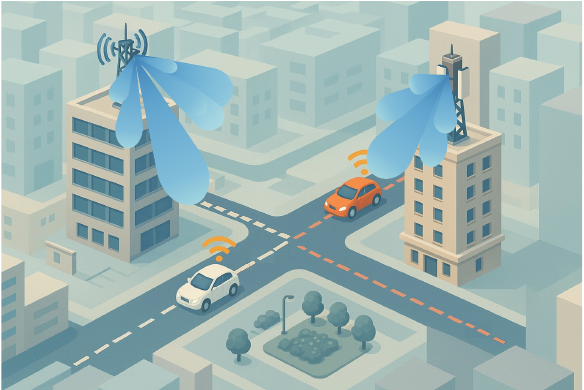}\vspace{-0.1in}
\par\end{centering}
\centering{}\caption{The mobile user moves along roads in a \ac{mimo} environment. For
each beam, multipath components predominantly arrive from the direction
of the mobile user, while paths from other angles exhibit significantly
lower probabilities. \label{fig:background}}
\vspace{-0.1in}
\end{figure}

\subsubsection{Path Loss Model}

Let $\mathbf{h}_{q}(\mathbf{x})=\sqrt{a_{q}(\mathbf{x})}\bar{\mathbf{h}}_{q}(\mathbf{x})$
be the narrowband \ac{mimo} channel between the user at position
$\mathbf{x}$ and the $q$th \ac{bs}, where $a_{q}(\mathbf{x})$
captures the channel gain, and $\bar{\mathbf{h}}_{q}(\mathbf{x})$
is the normalized channel vector that is assumed to be independent
of the channel gain. We adopt a path loss model for the channel gain
$a_{q}(\mathbf{x})$ in its logarithmic value, i.e., $10\log_{10}(a_{q}(\mathbf{x}))$,
as
\begin{equation}
[a_{q}(\mathbf{x})]_{\text{{\scriptsize dB}}}=\beta_{q}+\alpha_{q}\mathrm{log}_{10}d(\mathbf{x},\mathbf{o}_{q})+\xi'\label{eq:path-loss}
\end{equation}
where $\beta_{q}$ is a constant that depends on the propagation environment
surrounding the $q$th \ac{bs} and the path loss exponent $\alpha_{q}$
characterizes the rate at which the signal power diminishes with distance.
The random variable $\xi'$ models the log-normal shadowing of \ac{bs}
$q$ which is assumed to follow a zero mean Gaussian distribution.
We assume independent shadowing, although incorporating shadowing
correlation could improve model accuracy, it would require more parameters
and increase complexity.

\subsubsection{\ac{mimo} Pattern Model}

The realization of the normalized channel vector $\bar{\mathbf{h}}_{q}(\mathbf{x})$
depends on the user location $\mathbf{x}$, the multipath, and the
array response, and we model its statistical property via a set of
{\em sensing vectors} $\mathbf{g}_{q,m}$ of the $q$th \ac{bs},
where $m=1,2,\dots,M$ is the number beams on each BS. Specifically,
consider the \ac{aoa} of the paths is statistically centered around
the direct path for each user position, exhibiting small variance
under \ac{los} conditions and large variance under \ac{nlos} conditions.
We assume that the sensing vectors $\mathbf{g}_{q,m}$ are designed
in such a way that the antenna array at the $q$th \ac{bs} has a
statistically large response for signals arriving from a reference
direction $c_{q,m}\in[0,2\pi)$, and the array has a statistically
decreasing response for signals increasingly deviating from the reference
direction $c_{q,m}$, where the parameters $c_{q,m}$ is to be estimated.
Mathematically, consider the transmit-side \ac{ssb} beamforming gain
for each \ac{bs}, we use an exponential function to model such a
pattern
\begin{align}
 & 10\mathrm{log}_{10}\mathbb{E}\{|\mathbf{g}_{q,m}^{\text{H}}\bar{\mathbf{h}}_{q}(\mathbf{x})|^{2}\}\label{eq:beamForm}\\
 & =\omega_{q,m}\mathrm{exp}[-\eta_{q,m}(\phi(\mathbf{x},\mathbf{o}_{q})-c_{q,m})^{2}]+\xi''\nonumber 
\end{align}
where the expectation is taken over the randomness due to the small-scale
fading. The parameter $\omega_{q,m}$ models the beamforming gain
at the reference direction $c_{q,m}$, and $\eta_{q,m}$ models the
spread of the beam. The function $\phi(\mathbf{x},\mathbf{o}_{q})$
denotes the angle from \ac{bs} at $\mathbf{o}_{q}$ to user at $\mathbf{x}$.
The variable $\xi''$ captures the offset and the model mismatch.

A naive example of the sensing vectors $\mathbf{g}_{q,m}$ for a uniform
linear array is a set of $M$-dimensional \ac{dft} vectors applied
to a sub-array of $M$ consecutive antenna elements \cite{LotJas:J17}.
It is known that the received beamforming gain is maximized at a certain
direction, and the gain decreases if the incident signal arrives from
other directions.

The pattern model (\ref{eq:beamForm}) captures the spatial response
of the sensing vector $\mathbf{g}_{q,m}$ according to the \ac{aoa}
distribution of the multipaths arrived at the \ac{bs}. First, if
the direction of the mobile user $\phi(\mathbf{x},\mathbf{o}_{q})$
aligns with the reference direction $c_{q,m}$ of the beam, then a
maximum beamforming gain $\omega_{q,m}+\xi''$ is attained. By contrast,
if the user direction $\phi(\mathbf{x},\mathbf{o}_{q})$ substantially
deviates from the direction $c_{q,m}$, only a small gain $\xi''$
is attained. The parameter $\xi''$ provides an offset to the gain,
which can be absorbed in the parameter $\beta_{q}$ in the path loss
model (\ref{eq:path-loss}). Second, as shown in Figure~\ref{fig:background},
the model (\ref{eq:beamForm}) implicitly assumes that the multipaths
mainly arrive from the direction of the mobile user $\phi(\mathbf{x},\mathbf{o}_{q})$,
and hence, the $\mathbf{g}_{q,m}$ pointing to $c_{q,m}=\phi(\mathbf{x},\mathbf{o}_{q})$
receives the maximum gain, whereas, the paths arriving from other
angles have a substantially smaller probability. This model aligns
with the scenario where \acpl{bs} are placed on a high tower with
only a few local scatters.

\subsubsection{Measurement Model}

Motivated by the models (\ref{eq:path-loss}) and (\ref{eq:beamForm}),
we arrive at our measurement model as follows. Consider a slotted
system, $t=1,2,...,T$, and denote $\mathbf{x}_{t}\in\mathbb{R}^{2}$
as the location of the user at time $t$. For a mobile location $\mathbf{x}_{t}$,
denote $\tilde{y}_{t,q,m}=\mathbb{E}\{|\mathbf{g}_{q,m}^{\text{H}}\mathbf{h}_{q}(\mathbf{x}_{t})|^{2}\}$
as the average received power from sensing vector $\mathbf{g}_{q,m}$.
From (\ref{eq:path-loss}) and (\ref{eq:beamForm}), the logarithmic
value $y_{q,m,t}=10\log_{10}\tilde{y}_{q,m,t}$ is modeled as
\begin{align}
y_{q,m,t} & =\beta_{q}+\alpha_{q}\mathrm{log}_{10}d(\mathbf{x}_{t},\mathbf{o}_{q})\label{eq:measurement-model}\\
 & +\omega_{q,m}\mathrm{exp}[-\eta_{q,m}(\phi(\mathbf{x}_{t},\mathbf{o}_{q})-c_{q,m})^{2}]+\xi_{q,t}\nonumber 
\end{align}
\textcolor{black}{where $\xi_{q,t}\sim\mathcal{N}(0,\sigma_{q}^{2})$
models the randomness due to the shadowing and the pattern model mismatch
from $\xi'$ and $\xi''$ in (\ref{eq:path-loss}) and (\ref{eq:beamForm})
respectively, }$\xi''$ is absorbed into the constant $\beta_{q}$
 as a general offset for each  to ensure model parsimony, $\xi_{q,t}$
is assumed to be independent across \acpl{bs} and different time
slot $t$.

Thus, the \ac{pdf} of $y_{q,m,t}$ given the mobile user location
$\mathbf{x}_{t}$ can be given by
\begin{align}
 & p(y_{q,m,t}|\mathbf{x}_{t};\bm{\theta}_{q,m})\label{eq:prob-observ}\\
 & =\frac{1}{(2\pi)^{1/2}\sigma_{q}}\mathrm{exp}\Big\{-\frac{1}{2\sigma_{q}^{2}}\Big[y_{q,m,t}-\beta_{q}-\alpha_{q}\nonumber \\
 & \times\mathrm{log}_{10}d(\mathbf{x}_{t},\mathbf{o}_{q})-\omega_{q,m}\mathrm{exp}[-\eta_{q,m}(\phi(\mathbf{x}_{t},\mathbf{o}_{q})-c_{q,m})^{2}]\Big]^{2}\Big\}\nonumber 
\end{align}
\textcolor{black}{where $\bm{\theta}_{q,m}=\{\alpha_{q},\beta_{q},\sigma_{q},\omega_{q,m},\eta_{q,m},c_{q,m}\}$
is a collection of propagation parameters for the $m$th beam of the
$q$th }\ac{bs}\textcolor{black}{.}

Assume that, given the location, measurements across beams and \textcolor{black}{\acpl{bs}}
are statistically independent, which is a standard and practically
justified simplification.\textcolor{black}{{} Denote $\mathbf{y}_{t}$
as the collection of measurements of all \acpl{bs} $q$ over all
sensing vectors $m$ at time slot $t$. The \ac{pdf} of $\mathbf{y}_{t}$
is given by
\begin{equation}
p(\mathbf{y}_{t}|\mathbf{x}_{t};\bm{\Theta}_{\text{{p}}})=\prod_{q=1}^{Q}\prod_{m=1}^{M}p(y_{q,m,t}|\mathbf{x}_{t})\label{eq:prob-yt-xt}
\end{equation}
where $\bm{\Theta}_{\text{{p}}}=\{\bm{\theta}_{q,m}\}$ is a collection
of propagation parameters for beams $m$ and }\acpl{bs}\textcolor{black}{{}
$q$.}

In a more general case, if we only have sparse observations for a
few selected sensing vectors $\mathbf{g}_{q,m}$ indexed by the set
$\mathcal{M}_{q,t}$ for the $q$th \ac{bs}, and for a few selected
\acpl{bs} denoted by the index set $\mathcal{Q}_{t}$ at time slot
$t$, the model in (\ref{eq:prob-yt-xt}) can be expressed as \textcolor{black}{
\begin{equation}
p(\mathbf{y}_{t}|\mathbf{x}_{t};\bm{\Theta}_{\text{{p}}})=\prod_{q\in\mathcal{Q}_{t}}\prod_{m\in\mathcal{M}_{q,t}}p(y_{q,m,t}|\mathbf{x}_{t}).\label{eq:prob-yt-xt-1}
\end{equation}
}It is observed that the formulation is not affected by the situation
of missing data.

\subsection{Mobility Model}

\label{subsec:Vehicle-Trajectory-Model}

We adopt the Gauss-Markov model \cite{Lia:J99,He:J18} for the dynamics
of user mobility $\mathbf{x}_{t}$. Denote $\delta$ as the slot duration.
The movement at the $t$th time slot is modeled as:
\begin{equation}
\mathbf{x}_{t}-\mathbf{x}_{t-1}=\gamma(\mathbf{x}_{t-1}-\mathbf{x}_{t-2})+(1-\gamma)\delta\bar{\mathbf{v}}+\sqrt{1-\gamma^{2}}\delta\bm{\epsilon}.\label{eq:mobility-model}
\end{equation}
Here, the velocity $(\mathbf{x}_{t}-\mathbf{x}_{t-1})/\delta$ at
time slot $t$ depends on the velocity from the previous time slot
following an auto-regressive model with parameter $0<\gamma\leq1$
and randomness $\bm{\epsilon}\sim\mathcal{N}(\bm{0},\sigma_{\text{v}}^{2}\mathbf{I})$.
This is to capture the fact that acceleration is bounded in practice.
The parameter $\bar{\mathbf{v}}$ models the average velocity. A higher
$\gamma$ value indicates a stronger correlation between consecutive
velocities, resulting in smoother movement. When $\gamma=1$, the
mobile user maintains a constant velocity.

\subsection{An \ac{hmm} Formulation}

\label{subsec:An-HMM-Formulation}

Denote $\mathcal{X}_{t}=(\mathbf{x}_{1},\mathbf{x}_{2},\dots,\mathbf{x}_{t})$
and $\mathcal{Y}_{t}=(\mathbf{y}_{1},\mathbf{y}_{2},\dots,\mathbf{y}_{t})$
as the trajectory of the mobile user and the accumulated measurements
up to time $t$, respectively. The goal here is to estimate the trajectory
$\mathcal{X}_{t}$ based on the measurements $\mathcal{Y}_{t}$.

Based on the Gauss-Markov model (\ref{eq:mobility-model}), the \ac{pdf}
of the location $\mathbf{x}_{t}$ at time slot $t$ can be written
as $p(\mathbf{x}_{t}|\mathbf{x}_{t-1},\mathbf{x}_{t-2};\bm{\Theta}_{\text{m}})$,
where $\bm{\Theta}_{\text{{m}}}=\{\bar{\mathbf{v}},\sigma_{\text{v}}^{2}\}$
is the collection of mobility parameters. The Bayes' rule of probability
suggests that
\begin{align}
p(\mathcal{Y}_{t},\mathcal{X}_{t}) & =p(\mathbf{y}_{t}|\mathcal{Y}_{t-1},\mathbf{x}_{t},\mathcal{X}_{t-1})p(\mathcal{Y}_{t-1},\mathbf{x}_{t},\mathcal{X}_{t-1})\nonumber \\
 & =p(\mathbf{y}_{t}|\mathbf{x}_{t})p(\mathcal{Y}_{t-1},\mathbf{x}_{t},\mathcal{X}_{t-1})\nonumber \\
 & =p(\mathbf{y}_{t}|\mathbf{x}_{t})p(\mathbf{x}_{t}|\mathbf{x}_{t-1},\mathbf{x}_{t-2})p(\mathcal{Y}_{t-1},\mathcal{X}_{t-1}),\label{eq:p_Yt_Xt}
\end{align}
which is due to the fact that\textbf{ }$\mathbf{y}_{t}$ is independent
of $\mathcal{Y}_{t-1}$ and $\mathcal{X}_{t-1}$ given $\mathbf{x}_{t}$,
and $\mathbf{x}_{t}$ is independent of $\mathcal{Y}_{t-1}$ given
$\mathbf{x}_{t-1}$ and $\mathbf{x}_{t-2}$.

Recursively applying (\ref{eq:p_Yt_Xt}), one arrives at 
\[
p(\mathcal{Y}_{T},\mathcal{X}_{T})=\prod_{t=1}^{T}p(\mathbf{y}_{t}|\mathbf{x}_{t})\prod_{t=3}^{T}p(\mathbf{x}_{t}|\mathbf{x}_{t-1},\mathbf{x}_{t-2}).
\]

Consider to maximize the log-likelihood $\log p(\mathcal{Y}_{T},\mathcal{X}_{T})$.
We arrive at the following problem\textcolor{black}{
\begin{align}
\underset{\mathcal{X}_{T},\bm{\Theta}_{\text{{p}}},\bm{\Theta}_{\text{{m}}}}{\mathrm{maximize}} & \;\;\sum_{t=1}^{T}\log p(\mathbf{y}_{t}|\mathbf{x}_{t})+\sum_{t=3}^{T}\log p(\mathbf{x}_{t}|\mathbf{x}_{t-1},\mathbf{x}_{t-2})\label{eq:prob-J}
\end{align}
}which jointly fits the parameters $\bm{\Theta}_{\text{p}}$ of the
observation model (\ref{eq:prob-yt-xt}) and the parameters $\Theta_{\text{m}}$
of the mobility model to the data $\mathcal{Y}_{T}$ for the recovery
of the trajectory $\mathcal{X}_{T}$. Due to the assumption on independent
measurements, the models for multiple users are identical. Thus, it
is straight-forward to apply the algorithm to the case of multiple
users and multiple trajectories.

\section{Fundamental Limits}

\label{sec:Fundamental-Limits}

We study the fundamental limit on recovering the trajectory $\mathcal{X}_{T}$
from the power measurements under the most challenging case $N_{\text{t}}=1$
with a single antenna at each \ac{bs}, where we allow to scale the
number of \acpl{bs} to infinity. Increasing $N_{\text{t}}$ enables
the extraction of richer spatial features, thereby improving estimation
accuracy and trajectory identifiability. Specifically, we will examine
the \ac{crlb} for estimating $\mathbf{x}_{t}$. It is known that
the \ac{crlb} provides a lower bound on the \ac{mse} of any unbiased
estimator of a parameter. Under certain regularity conditions, the
\ac{mle} asymptotically achieves the \ac{crlb} as the number of
measurements increases to infinity \cite{KaySte:J93}. Thus, our goal
is to understand the best possible performance for recovering the
trajectory $\mathbf{x}_{t}$ and identify the critical parameters
that affect the performance.

To understand the performance upper bound, consider the case of $\gamma=1$
in the mobility model (\ref{eq:mobility-model}), where the Markovian
mobility model degenerates to a chain of states with transition probability
1, which corresponds to the user moving at a constant speed. Consequently,
the mobility $\mathbf{x}_{t}$ reduces to a deterministic trajectory:
\begin{equation}
\mathbf{x}_{t}=\mathbf{x}+t\mathbf{v}\label{eq:constant-speed-mobility}
\end{equation}
which can be fully determined by the parameters $(\mathbf{x},\mathbf{v})\in\mathbb{R}^{4}$,
i.e., the starting position $\mathbf{x}$ and the velocity $\mathbf{v}$.
It naturally follows that decreasing $\gamma$ will increase the randomness
of $\mathbf{x}_{t}$, and more information is needed to determine
$\mathbf{x}_{t}$. In the limiting case where $\gamma=0$, the trajectory
recovery problem degenerates to a series of conventional Bayesian
estimations for the positions $\mathbf{x}_{t}$ based on \ac{rss}
measurements. Hence, the case $\gamma=1$ corresponds to the best
possible performance we can obtain for the \ac{rss}-based trajectory
recovery problem.

Based on the constant speed mobility model (\ref{eq:constant-speed-mobility}),
the function $d(\mathbf{x}_{t},\mathbf{o}_{q})$ for the distance
between the mobile location $\mathbf{x}_{t}$ and the $q$th \ac{bs}
location $\mathbf{o}_{q}$ is simplified as $d_{t,q}(\mathbf{x},\mathbf{v})\triangleq\|\bm{l}_{q}(\mathbf{x})+t\mathbf{v}\|_{2}$,
where $\bm{l}_{q}(\mathbf{x})=\mathbf{x}-\mathbf{o}_{q}$ is the direction
from the $q$th \ac{bs} to the initial position $\mathbf{x}$ of
the trajectory. From the observation model (\ref{eq:prob-observ})
under $N_{\text{t}}=1$ and the mobility model (\ref{eq:mobility-model})
under $\gamma=1$, the log-likelihood function (\ref{eq:prob-J})
becomes 
\begin{align}
f(\bm{\phi},\bm{\psi}) & =\sum_{t=1}^{T}\sum_{q=1}^{Q}\bigg[-\ln2\pi\sigma_{q}^{2}\label{eq:log-likelihood-constant-speed}\\
 & \qquad-\frac{1}{2\sigma_{q}^{2}}\left(y_{t,q}-\beta_{q}-\alpha_{q}\log d_{t,q}(\mathbf{x},\mathbf{v})\right)^{2}\Big]\nonumber 
\end{align}
where $\bm{\phi}=\{\alpha_{q},\beta_{q}\}$ is the propagation parameters
and $\bm{\psi}=(\mathbf{x},\mathbf{v})$ is the mobility parameters,
and the term $\log p(\mathbf{x}_{t}|\mathbf{x}_{t-1},\mathbf{x}_{t-2})$
in (\ref{eq:prob-J}) disappears under constant speed mobility.

From the log-likelihood function (\ref{eq:log-likelihood-constant-speed}),
it is not surprising that once the mobility parameter $\bm{\psi}$
is available, the \ac{crlb} of the estimator for the propagation
parameter $\bm{\phi}$ approaches zero as $T\to\infty$. This is because
as $T$ approaches infinity, the number of independent measurements
under the observation model (\ref{eq:measurement-model}) approaches
infinity, and since equation (\ref{eq:measurement-model}) establishes
a linear model on the propagation parameter $\bm{\phi}$ with Gaussian
noise, the \ac{crlb} of estimating $\bm{\phi}$ approaches zero \cite{KaySte:J93}.

Hence, the focus here is to understand the fundamental limit of estimating
the mobility parameter $\bm{\psi}=(\mathbf{x},\mathbf{v})$.

\subsection{The Fisher Information Matrix}

The \ac{fim} $\mathbf{F}_{T,\psi}$ of $\bm{\psi}=(\mathbf{x},\mathbf{v})\in\mathbb{R}^{4}$
from the measurements over a duration $T$ can be computed as
\begin{align*}
\mathbf{F}_{T,\psi} & \triangleq\mathbb{E}\{-\nabla_{\bm{\psi}\bm{\psi}}^{2}f(\bm{\phi},\bm{\psi})\}\\
 & =\sum_{t,q}\frac{\alpha_{q}^{2}}{\sigma_{q}^{2}d_{t,q}^{2}(\mathbf{x},\mathbf{v})}\nabla_{\bm{\psi}}d_{t,q}(\mathbf{x},\mathbf{v})(\nabla_{\bm{\psi}}d_{t,q}(\mathbf{x},\mathbf{v}))^{\mathrm{T}}
\end{align*}
where the derivative $\nabla_{\bm{\psi}}d_{t,q}(\mathbf{x},\mathbf{v})$
is derived as
\begin{align*}
\nabla_{\bm{\psi}}d_{t,q}(\mathbf{x},\mathbf{v}) & =\left[\begin{array}{c}
\frac{\partial}{\partial\mathbf{x}}\|\bm{l}_{q}(\mathbf{x})+t\mathbf{v}\|_{2}\\
\frac{\partial}{\partial\mathbf{v}}\|\bm{l}_{q}(\mathbf{x})+t\mathbf{v}\|_{2}
\end{array}\right]=\frac{\bm{l}_{q}(\mathbf{x})+t\mathbf{v}}{d_{t,q}(\mathbf{x},\mathbf{v})}\left[\begin{array}{c}
1\\
t
\end{array}\right]^{\text{T}}.
\end{align*}
Thus, the \ac{fim} can be expressed as
\begin{align}
\mathbf{F}_{T,\psi} & =\sum_{t,q}\frac{\alpha_{q}^{2}}{\sigma_{q}^{2}d_{t,q}^{4}(\mathbf{x},\mathbf{v})}\left[\begin{array}{cc}
1 & t\\
t & t^{2}
\end{array}\right]\label{eq:F-Txtil}\\
 & \qquad\otimes\left((\bm{l}_{q}(\mathbf{x})+t\mathbf{v})(\bm{l}_{q}(\mathbf{x})+t\mathbf{v})^{\text{T}}\right)\nonumber 
\end{align}
in which, $\otimes$ is the Kronecker product.

For an unbiased estimator $\hat{\bm{\psi}}$, the \ac{mse} is lower
bounded by $\mathbb{E}\{\|\hat{\bm{\psi}}-\bm{\psi}\|^{2}\}\geq\mathrm{tr}\{\mathbf{F}_{T,\psi}^{-1}\}$,
where $\mathrm{tr}\{\mathbf{F}_{T,\psi}^{-1}\}$ is the \ac{crlb}
of estimating $\bm{\psi}=(\mathbf{x},\mathbf{v})$. Similarly, we
define the \acpl{fim} $\mathbf{F}_{T,x}=\mathbb{E}\left\{ -\nabla_{\mathbf{x}\mathbf{x}}^{2}f(\bm{\phi},\bm{\psi})\right\} $
and $\mathbf{F}_{T,v}=\mathbb{E}\left\{ -\nabla_{\mathbf{v}\mathbf{v}}^{2}f(\bm{\phi},\bm{\psi})\right\} $,
which are the diagonal blocks of $\mathbf{F}_{T,\psi}$ and are associated
with the \ac{crlb} $B(\mathbf{x})=\mathrm{tr}\{\mathbf{F}_{T,x}^{-1}\}$
and \ac{crlb} $B(\mathbf{v})=\mathrm{tr}\{\mathbf{F}_{T,v}^{-1}\}$
for the parameters $\mathbf{x}$ and $\mathbf{v}$, respectively.

\subsection{\ac{bs} Deployed in a Limited Region}

We first investigate the case where the \acpl{bs} are deployed in
a limited region, but the measurement trajectory is allowed to go
unbounded as $T$ goes to infinity. Signals can always be collected
by the \acpl{bs} regardless of the distance. As a result, an infinite
amount of measurements can be collected as $T\to\infty$.

It is observed that $\mathbf{F}_{T,\psi}\prec\mathbf{F}_{T+1,\psi}$,
indicating that the Fisher information is strictly increasing. This
is because each term in the summation in (\ref{eq:F-Txtil}) is positive
definite, provided that $\bm{l}_{q}(\mathbf{x})$ and $\mathbf{v}$
are linear independent for at least one $q$.

However, it is somewhat surprising that the \ac{crlb} for $\mathbf{x}_{}$
and $\mathbf{v}$ does not decrease to zero as $T\to\infty$, despite
the infinitely increasing amount of independent data.

Specifically, assume that the trajectory $\mathbf{x}_{t}$ does not
pass any of the \ac{bs} location $\mathbf{o}_{q}$, and hence, $d_{\min,q}=\min_{t}\{d_{t,q}(\mathbf{x},\mathbf{v})\}>0$
for all $q$. Define $\alpha_{\max}^{2}=\max_{q}\{\alpha_{q}^{2}\}$,
$\sigma_{\min}^{2}=\min_{q}\{\sigma_{q}^{2}\}$.
\begin{thm}
\label{thm:LB-F-x}The \ac{crlb} of $\mathbf{x}$ satisfies $B(\mathbf{x})=\mathrm{tr}\{\mathbf{F}_{T,x}^{-1}\}\geq\bar{\Delta}_{T,x}$,
where equality can be achieved when $\sigma_{\mathrm{\min}}^{2}=\sigma_{q}^{2}$
and $\alpha_{\mathrm{max}}^{2}=\alpha_{q}^{2}$ for all $q$. In addition,
$\bar{\Delta}_{T,x}$ is strictly decreasing in $T$, provided that
at least two vectors in $\{\bm{l}_{1}(\mathbf{x}),\bm{l}_{2}(\mathbf{x}),\dots,\bm{l}_{Q}(\mathbf{x}),\mathbf{v}\}$
are linear independent, but $\bar{\Delta}_{T,x}$ converges to a strictly
positive number as $T\to\infty$.
\end{thm}
\begin{proof}See Appendix \ref{sec:app-theo1}.\end{proof}

Theorem \ref{thm:LB-F-x} suggests that the \ac{crlb} of $\mathbf{x}$
cannot decrease to zero even when we estimate only two parameters
for the initial location $\mathbf{x}\in\mathbb{R}^{2}$ based on {\em infinite}
measurement samples collected over an infinite geographical horizon
as $T\to\infty$.

Through the development of the proof, a physical interpretation of
Theorem \ref{thm:LB-F-x} can be given as follows. As $T$ increases,
the distances $d_{t,q}(\mathbf{x},\mathbf{v})=\|\mathbf{x}_{t}-\mathbf{o}_{q}\|_{2}$
grow larger because the user moves away from the \acpl{bs}. For a
position $\mathbf{x}_{t}$ at a sufficiently large distance, the term
$\mathbf{x}_{t}-\mathbf{o}_{q}$ approximates to $\mathbf{x}_{t}$
since $\|\mathbf{x}_{t}\|\gg\|\mathbf{o}_{q}\|$. Consequently, all
measurements point to approximately the same direction independent
of the \ac{bs} locations $\mathbf{o}_{q}$, i.e., the incremental
information provided by each new measurement diminishes rapidly. Therefore,
although the \ac{fim} $\mathbf{F}_{T,x_{}}$ strictly increases with
$T$, the increment $\mathbf{F}_{T+1,x_{}}-\mathbf{F}_{T,x_{}}$ decreases
quickly. As a result, the \acpl{crlb} strictly decreases but only
approaches a non-zero lower bound, similar to how the series $\sum_{t=1}^{\infty}1/t^{r}$
converges for $r>1$.

While it is not possible to perfectly estimate the starting point
$\mathbf{x}$ under $T=\infty$ with a finite number of \acpl{bs},
one might expect that estimating the velocity $\mathbf{v}$ could
be more promising. However, we have the following result.
\begin{thm}
\label{thm:LB-F-v}The \ac{crlb} of $\mathbf{v}$ satisfies $B(\mathbf{v})=\mathrm{tr}\{\mathbf{F}_{T,v}^{-1}\}\geq\bar{\Delta}_{T,v}$
with equality achieved when $\sigma_{\mathrm{\min}}^{2}=\sigma_{q}^{2}$,
$\alpha_{\mathrm{max}}^{2}=\alpha_{q}^{2}$ for all $q$. In addition,
\[
\bar{\Delta}_{T,v}\rightarrow C_{v}=\left(\frac{\alpha_{\max}^{2}}{\sigma_{\min}^{2}}\sum_{q=1}^{Q}s_{\infty,q}^{(2)}\|\mathbf{P}_{v}^{\bot}\bm{l}_{q}(\mathbf{x})\|^{2}\right)^{-1}
\]
 as $T\rightarrow\infty$, where
\[
s_{\infty,q}^{(2)}=\lim_{T\to\infty}\sum_{t=1}^{T}\frac{t^{2}}{d_{t,q}^{4}(\mathbf{x},\mathbf{v})},\:\mathbf{P}_{v}^{\bot}=\mathbf{I}-\mathbf{v}\mathbf{v}^{T}/\|\mathbf{v}\|^{2}
\]
in which, the parameter $s_{\infty,q}^{(2)}$ is upper bounded by
$1/\rho^{4}\lim_{T\to\infty}\sum_{t=1}^{T}1/t^{2}\approx\pi^{2}/(6\rho^{4})$,
where $\rho>0$ is sufficiently small such that $d_{t,q}(\mathbf{x},\mathbf{v})>\rho t$
for all $t\geq1$.
\end{thm}
\begin{proof}See Appendix \ref{sec:app-theo2}.\end{proof}

While traditional work emphasizes \ac{bs} geometric diversity for
single-point localization, Theorems \ref{thm:LB-F-x} and \ref{thm:LB-F-v}
suggest that, under a finite number of \acpl{bs} in a limited region,
neither the initial position $\mathbf{x}_{}$ nor the velocity $\mathbf{v}$
can be perfectly estimated by merely increasing the observation time
$T$.

Theorem \ref{thm:LB-F-v} implies that the fundamental limit to the
estimation accuracy is affected by the spatial distribution of the
\acpl{bs} and the nature of the \ac{rss} measurements. Specifically,
the non-diminishing error lower bound $C_{v}$ in estimating the velocity
$\mathbf{v}\in\mathbb{R}^{2}$ is inversely affected by the norm of
$\mathbf{P}_{v}^{\bot}\bm{l}_{q}(\mathbf{x})$, i.e., the vector $\bm{l}_{q}(\mathbf{x})=\mathbf{x}-\mathbf{o}_{q}$
projected on the orthogonal direction of $\mathbf{v}$. The more spread
of the \acpl{bs} in the orthogonal direction of the moving direction
$\mathbf{v}$, the lower the estimation error of $\mathbf{v}$.

Another surprising observation that differentiates the trajectory
recovery problem (\ref{eq:prob-J}) from a conventional \ac{rss}-based
localization problem is that a large path loss exponent $|\alpha_{q}|$
is expected to decrease the \ac{crlb} for recovering the trajectory.
Recall that in the empirical path loss model, $\alpha_{q}=-20$ corresponds
to free space propagation, and $\alpha_{q}<-20$ usually corresponds
to propagations in the \ac{nlos} scenarios. Hence, for the same shadowing
deviation $\sigma$, a “deep fade” that empirically leads to a large
$|\alpha_{q}|$ seems to be preferred for a small \ac{crlb} in trajectory
recovery, or more rigorously, a large ratio $|\alpha_{q}/\sigma|$
is preferred. This phenomenon can be interpreted by the property that
a large $|\alpha_{q}|$ can better differentiate the movement distance
along ${\mathbf{x}_{t}}$ based on the \ac{rss} measurements in the
model (\ref{eq:measurement-model}).

\subsection{\ac{bs} Deployed in an Unlimited Region}

We now study the case where the \acpl{bs} follow a \ac{ppp} with
density $\kappa$ in an unlimited region, but the users can only connect
with a subset of \acpl{bs} within a radius of $R$. As a result,
the number of connected \acpl{bs} is still finite. We investigate
the \acpl{crlb} as the user trajectory goes unbounded as $T$ goes
to infinity.

Specifically, on average, measurements from $\bar{Q}=\kappa\pi R^{2}$
\acpl{bs} can be collected in each time slot. It turns out that,
in such a scenario, although we still have a limited number of active
\acpl{bs}, the estimation lower bound now can decrease to zero. Denote
$\alpha_{\min}^{2}=\min_{q}\{\alpha_{q}^{2}\}$, $\sigma_{\max}^{2}=\max_{q}\{\sigma_{q}^{2}\}$.
\begin{thm}
\label{thm:LB-F-x-un}Assume that the minimum distance to the nearest
\ac{bs} is greater than $r_{0}$ along the trajectory.\footnote{In practice, the parameter $r_0$ can be understood as the height of the  antenna. More rigorously, we should employ a 3D model to compute the distance $d_{t,q}$, but the asymptotic result would be the same.}
The \ac{crlb} of $\mathbf{x}$ satisfies $B(\mathbf{x})=\mathrm{tr}\{\mathbf{F}_{T,x}^{-1}\}\leq\tilde{\Delta}_{T,x}$
and as $T\to\infty$
\[
T\tilde{\Delta}_{T,x}\rightarrow\frac{2\sigma_{\max}^{2}}{\alpha_{\min}^{2}\kappa\pi\ln(R/r_{0})}.
\]
\end{thm}
\begin{proof}See Appendix \ref{sec:app-theo3}.\end{proof}
\begin{thm}
\label{thm:LB-F-v-un}Assume that the minimum distance to the nearest
\ac{bs} is greater than $r_{0}$ along the trajectory. The \ac{crlb}
of $\mathbf{v}$ satisfies $B(\mathbf{v})=\mathrm{tr}\{\mathbf{F}_{T,v}^{-1}\}\leq\tilde{\Delta}_{T,v}$
and as $T\to\infty$
\[
T(T+1)(2T+1)\tilde{\Delta}_{T,v}\rightarrow\frac{12\sigma_{\max}^{2}}{\alpha_{\min}^{2}\kappa\pi\ln(R/r_{0})}.
\]
\end{thm}
\begin{proof}See Appendix \ref{sec:app-theo4}.\end{proof}

It is observed from the above theorems that the \ac{crlb} of $\mathbf{x}$
decreases as $\mathcal{O}(1/T)$ and the \ac{crlb} of $\mathbf{v}$
decreases as $\mathcal{O}(1/T^{3})$. Estimating the velocity $\mathbf{v}$
is significantly easier than estimating the initial location $\mathbf{x}$.
Furthermore, a longer measurement range $R$ enhances the accuracy
of the estimation. Additionally, a higher density of \acpl{bs} $\kappa$
within the radius also leads to improved estimation performance.

To summarize, recall that $1/\sigma_{\max}^{2}$ is proportional to
the \ac{snr}. Our results show that it is possible to perfectly recover
a full trajectory from \ac{rss} measurements without any location
labels under all \ac{snr} conditions. However, a larger $\sigma_{\max}^{2}$
leads to a slower decrease in the \ac{crlb} of $\mathbf{x}$ and
$\mathbf{v}$ as $T$ increases.

\section{Algorithm Design}

\label{sec:HMM-Based-CSI-Embedding}

To solve the joint trajectory recovery and parameter estimation problem
(\ref{eq:prob-J}), it is observed that given $\mathcal{X}_{T}$,
the variables $\bm{\Theta}_{\text{{p}}}$ and $\bm{\Theta}_{\text{{m}}}$
are decoupled, because the first term in (\ref{eq:prob-J}) only depends
on $\bm{\Theta}_{\text{{p}}}$ and the second term in (\ref{eq:prob-J})
only depends on $\bm{\Theta}_{\text{{m}}}$. As a result, $\bm{\Theta}_{\text{{p}}}$
and $\bm{\Theta}_{\text{{m}}}$ can be solved by two parallel subproblems
from (\ref{eq:prob-J}) as follow\textcolor{black}{s
\begin{align*}
(\mathrm{P1}):\underset{\bm{\Theta}_{\text{{m}}}}{\mathrm{maximize}} & \;\;\sum_{t=3}^{T}\log p(\mathbf{x}_{t}|\mathbf{x}_{t-1},\mathbf{x}_{t-2};\bm{\Theta}_{\text{{m}}}),\\
(\mathrm{P2}):\underset{\bm{\Theta}_{\text{{p}}}}{\mathrm{maximize}} & \;\;\sum_{t=1}^{T}\log p(\mathbf{y}_{t}|\mathbf{x}_{t};\bm{\Theta}_{\text{{p}}}).
\end{align*}
}

On the other hand, given the variables $\hat{\bm{\Theta}}_{\text{{p}}}$
and $\hat{\bm{\Theta}}_{\text{{m}}}$ as the solutions to (P1) and
(P2), respectively, $\mathcal{X}_{T}$ can be solved b\textcolor{black}{y
\begin{align*}
(\mathrm{P3}):\underset{\mathcal{X}_{T}}{\mathrm{maximize}} & \;\;\sum_{t=1}^{T}\log p(\mathbf{y}_{t}|\mathbf{x}_{t};\hat{\bm{\Theta}}_{\text{{p}}})\\
 & \qquad+\sum_{t=3}^{T}\log p(\mathbf{x}_{t}|\mathbf{x}_{t-1},\mathbf{x}_{t-2};\hat{\bm{\Theta}}_{\text{{m}}}).
\end{align*}
}This naturally leads to an alternating optimization strategy, which
solves for $\mathcal{X}_{T}$ from problem (P3), and then for $\hat{\bm{\Theta}}_{\text{{p}}}$
and $\hat{\bm{\Theta}}_{\text{{m}}}$ from (P1) and (P2), in an iterative
manner. Since the corresponding iterations never decrease the objective
(\ref{eq:prob-J}) which is bounded above, the iteration must converge.

The solutions to these subproblems are derived as follows.

\subsection{Solution to (P1) for the Mobility Model}

According to the mobility model in (\ref{eq:mobility-model}), the
\ac{pdf} of the location $\mathbf{x}_{t}$ at time slot $t$ is given
by
\begin{align}
 & p(\mathbf{x}_{t}|\mathbf{x}_{t-1},\mathbf{x}_{t-2};\bm{\Theta}_{\text{{m}}})=\frac{1}{2\pi\sigma_{\text{v}}\sqrt{1-\gamma^{2}}}\label{eq:prob-velocity}\\
 & \times\mathrm{exp}\left\{ -\frac{\|\mathbf{x}_{t}-(1+\gamma)\mathbf{x}_{t-1}+\gamma\mathbf{x}_{t-2}-(1-\gamma)\delta\bar{\mathbf{v}}\|_{2}^{2}}{2(1-\gamma^{2})\delta^{2}\sigma_{\text{v}}^{2}}\right\} \nonumber 
\end{align}
for $\gamma\neq1$.

From (\ref{eq:prob-velocity}), the objective function of (P1) can
be expressed as
\begin{align}
 & -\sum_{t=3}^{T}\log\big(2\pi\sigma_{\text{v}}\sqrt{1-\gamma^{2}}\big)\label{eq:P1-obj}\\
 & \qquad-\sum_{t=3}^{T}\frac{\|\mathbf{x}_{t}-(1+\gamma)\mathbf{x}_{t-1}+\gamma\mathbf{x}_{t-2}-(1-\gamma)\delta\bar{\mathbf{v}}\|_{2}^{2}}{2(1-\gamma^{2})\delta^{2}\sigma_{\text{v}}^{2}}.\nonumber 
\end{align}

Setting the derivative of (\ref{eq:P1-obj}) \ac{wrt} $(\bar{\mathbf{v}},\sigma_{\text{v}}^{2})$
to zero, we find that the corresponding solution

\begin{align}
\bar{\mathbf{v}} & =\frac{\sum_{t=3}^{T}(\mathbf{x}_{t}-(1+\gamma)\mathbf{x}_{t-1}+\gamma\mathbf{x}_{t-2})}{(T-2)(1-\gamma)\delta}\label{eq:solution-v}\\
\sigma_{\text{v}}^{2} & =\frac{\sum_{t=3}^{T}\|\mathbf{x}_{t}-(1+\gamma)\mathbf{x}_{t-1}+\gamma\mathbf{x}_{t-2}-(1-\gamma)\delta\bar{\mathbf{v}}\|_{2}^{2}}{2(T-2)\delta^{2}}\label{eq:solution-v-sigma}
\end{align}
is unique. Since (P1) is an unconstrained optimization problem, (\ref{eq:solution-v})–(\ref{eq:solution-v-sigma})
give the optimal solution to (P1).

\subsection{Solution to (P2) via Separable Regression with Log-transformation}

Solving (P2) is very challenging because the regression model (\ref{eq:prob-observ})
and (\ref{eq:prob-yt-xt}) for (P2) is highly non-linear containing
both polynomial terms and exponential terms. Therefore, it is important
to investigate whether (P2) can be separated into easier subproblems.

It is observed that there are two groups of parameters, where one
group of parameters $\{\alpha_{q},\beta_{q},\sigma_{q}\}$ describe
the path loss model for each \ac{bs} $q$, and the other group $\{w_{q,m},\eta_{q,m},c_{q,m}\}$
describe the spatial pattern for each index $m$ of each \ac{bs}
$q$.

\subsubsection{Separability}

We first investigate the separability of the two groups of parameters.

Denote
\begin{equation}
y'_{q,m,t}=y_{q,m,t}-w_{q,m}\exp[-\eta_{q,m}(\phi(\mathbf{x}_{t},\mathbf{o}_{q})-c_{q,m})^{2}].\label{eq:y'}
\end{equation}
Using (\ref{eq:prob-observ})–(\ref{eq:prob-yt-xt}) and given the
group of pattern parameters $\{w_{q,m},\eta_{q,m},c_{q,m}\}$, problem
(P2) can be written into $Q$ parallel subproblems for $q=1,2,\dots,Q$,
\begin{equation}
\underset{\alpha_{q},\beta_{q},\sigma_{q}}{\text{minimize}}\quad\frac{1}{2\sigma_{q}^{2}}\sum_{m,t}(y'_{q,m,t}-\alpha_{q}\tilde{d}_{q,t}-\beta_{q})^{2}+\frac{QMT}{2}\ln(2\pi\sigma_{q}^{2})\label{eq:path-loss-parameter-estimation}
\end{equation}
where $\tilde{d}_{q,t}=\log_{10}d(\mathbf{x}_{t},\mathbf{o}_{q})$
is the log-distance between the $q$th \ac{bs} and the user at $\mathbf{x}_{t}$.
It follows that the solutions to (\ref{eq:path-loss-parameter-estimation})
depends on the pattern parameters $\{w_{q,m},\eta_{q,m},c_{q,m}\}$
via the variables $y'_{q,m,t}$ in (\ref{eq:y'}). Denote the solution
to (\ref{eq:path-loss-parameter-estimation}) as $\alpha_{q}^{(1)}$,
$\beta_{q}^{(1)}$, and $\sigma_{q}^{(1)}.$

Consider an {\em auxiliary} problem that does not explicitly depends
on the pattern parameters $\{w_{q,m},\eta_{q,m},c_{q,m}\}$ 
\begin{equation}
\underset{\alpha_{q},\beta_{q},\sigma_{q}}{\text{minimize}}\quad\frac{1}{2\sigma_{q}^{2}}\sum_{t}(\bar{y}_{q,t}-\alpha_{q}\tilde{d}_{q,t}-\beta_{q})^{2}+\frac{QMT}{2}\ln(2\pi\sigma_{q}^{2})\label{eq:path-loss-estimation-2}
\end{equation}
where $\bar{y}_{q,t}=\sum_{m}y_{q,m,t}$, which only depends on the
aggregated values from $y_{q,m,t}$ regardless of the parameters $\{w_{q,m},\eta_{q,m},c_{q,m}\}$
under some conditions to be specified later. This is because while
different beams have different energy distribution in space, it is
possible that the aggregate energy over all beams can be uniform over
the coverage area of a \ac{bs}. Denote the corresponding solution
as $\alpha_{q}^{(2)}$, $\beta_{q}^{(2)}$, and $\sigma_{q}^{(2)}.$

We find the conditions when the path loss parameters estimated from
the separated problem (\ref{eq:path-loss-estimation-2}) are identical
to the original problem (\ref{eq:path-loss-parameter-estimation}).

\begin{myprop}

[Separability]\label{prop:sparability} Suppose that the measurement
model (\ref{eq:measurement-model}) satisfies 
\begin{equation}
\sum_{m=1}^{M}w_{q,m}\exp[-\eta_{q,m}(\phi(\mathbf{x}_{t},\mathbf{o}_{q})-c_{q,m})^{2}]=\bar{C}_{q}\label{eq:condition-uniform-distritubted-beam-pattern}
\end{equation}
for all $\mathbf{x}_{t}$ and some constant $\bar{C}_{q}$. Then,
the solutions to the path loss parameters satisfy $\alpha_{q}^{(1)}=\alpha_{q}^{(2)}$
and $\beta_{q}^{(1)}=\beta_{q}^{(2)}-\bar{C}_{q}$.

\end{myprop}

\begin{proof}See Appendix \ref{sec:app-prop1}. \end{proof}

The condition (\ref{eq:condition-uniform-distritubted-beam-pattern})
requires that the aggregated beamforming gain over all beams is identical
for all locations $\mathbf{x}_{t}$. This implicitly requires that
the \ac{bs} distributes the beams uniformly within the coverage range
such that the aggregated beamforming gain is uniform at all possible
directions.

Proposition \ref{prop:sparability} delivers two important messages.
First, while the condition (\ref{eq:condition-uniform-distritubted-beam-pattern})
may be challenging to be met in practice, Proposition \ref{prop:sparability}
suggests that distributing the beams as uniform as possible can simplify
the parameter estimation, since the estimation problem (P2) can be
approximately decomposed into a subproblem of estimating the path
loss parameters separately. Second, when there is beamforming gain,
solving the separated path loss estimation problem (\ref{eq:path-loss-estimation-2})
tends to over estimate the path gain by $\bar{C}_{q}$, which equals
to the aggregated beamforming gain, but the estimation of the path
loss exponent $\alpha_{q}$ is not affected and accurate.

Motivated by Proposition \ref{prop:sparability}, one can easily initialize
the path loss parameters by solving the separated problem (\ref{eq:path-loss-estimation-2}),
and then, one fine tunes the estimate by iteratively estimating the
pattern parameters and the path loss parameters, leading to the proposed
strategy described as follows.

\subsubsection{Path Loss Parameters}

\label{subsec:Path-Loss-Parameters}

To solve problem (\ref{eq:path-loss-estimation-2}), denote $\bar{\mathbf{y}}_{q}\in\mathbb{R}^{T}$
as the collection of the variables $\bar{y}_{q,t}$ for the $q$th
\ac{bs} along the trajectory $\mathbf{x}_{t}$, $\mathbf{d}_{q}\in\mathbb{R}^{T}$
as the collection of all the log-distances $\tilde{d}_{q,t}$, $\mathbf{D}_{q}=[\mathbf{d}_{q},\mathbf{1}]\in\mathbb{R}^{T\times2}$
where $\mathbf{1}$ is a all-one vector with $T$ elements, and $\bm{\theta}_{q}=[\alpha_{q},\beta_{q}]^{\text{T}}$.
The first term of (\ref{eq:path-loss-estimation-2}) can be written
into the matrix form as $\frac{1}{2\sigma_{q}^{2}}\|\bar{\mathbf{y}}_{q}-\mathbf{D}_{q}\bm{\theta}_{q}\|^{2}$,
and the second term of (\ref{eq:path-loss-estimation-2}) does not
depend on $\bm{\theta}_{q}$. Such a least-squares problem (\ref{eq:path-loss-estimation-2})
has the solution in the matrix form
\begin{equation}
\hat{\bm{\theta}}_{q}=(\mathbf{D}_{q}^{\text{T}}\mathbf{D}_{q})^{-1}\mathbf{D}_{q}^{\text{T}}\bar{\mathbf{y}}_{q}.\label{eq:solution-path-loss-estimation-2}
\end{equation}

The solution to problem (\ref{eq:path-loss-parameter-estimation})
can be obtained in a similar way and is found as 
\begin{equation}
\hat{\bm{\theta}}_{q}=(\tilde{\mathbf{D}}_{q}^{\text{T}}\tilde{\mathbf{D}}_{q})^{-1}\tilde{\mathbf{D}}_{q}^{\text{T}}\mathbf{y}_{q}'\label{eq:solution-path-loss-estimation-1}
\end{equation}
where $\mathbf{y}_{q}'\in\mathbb{R}^{MT}$ is the collection of variable
$y_{q,m,t}'$ and $\tilde{\mathbf{D}}_{q}=\mathbf{D}_{q}\otimes\bm{1}$,
in which, $\mathbf{1}$ is an all-one vector with $M$ elements.

The solution to $\sigma_{q}$ can be found as setting the derivative
of the objective function in problem (\ref{eq:path-loss-parameter-estimation})
to zero, and is found as 
\begin{equation}
\hat{\sigma}_{q}=\sqrt{\frac{1}{MT}\sum_{m,t}(y'_{q,m,t}-\alpha_{q}d_{q,t}-\beta_{q})^{2}}.\label{eq:solution-path-loss-estimation-sigma}
\end{equation}

\subsubsection{Pattern Parameters via Log-transformation}

To estimate the pattern parameters $\{w_{q,m},\eta_{q,m},c_{q,m}\}$,
denote $y''_{q,m,t}=y_{q,m,t}-\alpha_{q}\tilde{d}_{q,t}-\beta_{q}$.
Given the path loss parameters $\{\alpha_{q},\beta_{q},\sigma_{q}\}$,
problem (P2) can be equivalently written as $QM$ parallel subproblems
for $q=1,2,\dots,Q$ and $m=1,2,\dots,M$, 
\begin{equation}
\underset{w_{q,m}\eta_{q,m},c_{q,m}}{\text{minimize}}\quad\sum_{t}\Big(y''_{q,m,t}-w_{q,m}\exp\big[-\eta_{q,m}(\phi_{q,t}-c_{q,m})^{2}\big]\Big)^{2}\label{eq:beam-pattern-estimation}
\end{equation}
where $\phi_{q,t}=\phi(\mathbf{x}_{t},\mathbf{o}_{q})$ captures the
direction from the $q$th \ac{bs} to the user at $\mathbf{x}_{t}$.

It is observed that even by removing the path loss components, the
regression problem (\ref{eq:beam-pattern-estimation}) is still difficult
to solve as the the regression model contains both exponential and
polynomial terms. Here, we propose to linearize the problem via log-transformation.

Note that the regression model $w_{q,m}\exp[-\eta_{q,m}(\phi_{q,t}-c_{q,m})^{2}$
in (\ref{eq:beam-pattern-estimation}) is positive, but it is possible
that $y_{q,m,t}''$ is negative due to the estimation error from the
path loss parameters $\{\alpha_{q},\beta_{q},\sigma_{q}\}$. In addition,
Proposition \ref{prop:sparability} suggests that the path gain is
overestimated by solving the separated subproblem (\ref{eq:path-loss-estimation-2}),
$i.e.$, $\beta^{(2)}=\beta_{q}^{(1)}+\bar{C}_{q}>\beta_{q}^{(1)}$,
resulting in a negative bias for all $y''_{q,m,t}$. As a result,
it is more reliable to focus on the set of data $\mathcal{T}_{q,m}^{\epsilon}=\{t:y''_{q,m,t}>\epsilon\}$
for the estimation of the pattern parameters $\{w_{q,m},\eta_{q,m},c_{q,m}\}$,
where $y''_{q,m,t}$ is greater than some positive threshold $\epsilon$.
A positive threshold $\epsilon$ is used to exclude noisy, low-power
measurements and ensure robust parameter estimation from reliable,
high-\ac{snr} data.

We perform nonlinear regression estimation using a linear approximation
\cite{BatDou:J88}. Specifically,\foreignlanguage{english}{ for $t\in\mathcal{T}_{q,m}^{\epsilon}$,
}we take the natural logarithm on both sides of the model, resulting
in the following auxiliary weighted linear regression problem:
\begin{align}
\underset{w_{q,m}\eta_{q,m},c_{q,m}}{\text{minimize}}\quad & \sum_{t\in\mathcal{T}_{q,m}^{\epsilon}}\lambda_{t}\Big(\ln y''_{q,m,t}-\ln(w_{q,m}\label{eq:beam-pattern-estimation1}\\
 & \quad\times\exp[-\eta_{q,m}(\phi_{q,t}-c_{q,m})^{2})\Big)\nonumber 
\end{align}
\foreignlanguage{english}{where $\lambda_{t}>0$ is a weighting factor
introduced to ensure equivalence with the original pattern estimation
problem }(\ref{eq:beam-pattern-estimation})\foreignlanguage{english}{.}

Observing that\foreignlanguage{english}{ $\ln(w_{q,m}\exp[-\eta_{q,m}(\phi_{q,t}-c_{q,m})^{2})=\ln w_{q,m}-\eta_{q,m}(\phi_{q,t}-c_{q,m})^{2}$,}
denote a set of variables $b_{1}=-\eta_{q,m}$, $b_{2}=2\eta_{q,m}c_{q,m}$,
and $b_{3}=\ln w_{q,m}-\eta_{q,m}c_{q,m}^{2}$. The pattern estimation
problem (\ref{eq:beam-pattern-estimation1}) can be rewritten as
\begin{equation}
\underset{\mathbf{b}}{\text{minimize}}\quad\sum_{t\in\mathcal{T}_{q,m}^{\epsilon}}\lambda_{t}\big(\ln y_{q,m,t}''-(\phi_{q,t}^{2}b_{1}+\phi_{q,t}b_{2}+b_{3})\big)^{2}\label{eq:beam-pattern-estimation-auxiliary}
\end{equation}
and we have the following result.

\begin{myprop}

[Equivalence Condition]\label{prop:Equivalence-Condition} Suppose
that $\mathcal{T}_{q,m}^{\epsilon}=\{1,2,\dots,T\}$. Denote $B(\bm{\vartheta};\phi_{q,t})=w_{q,m}\exp\big[-\eta_{q,m}(\phi_{q,t}-c_{q,m})^{2}\big]$,
where $\bm{\vartheta}=(w_{q,m},\eta_{q,m},c_{q,m})$. If the weights
$\lambda_{t}$ satisfy 
\begin{equation}
\lambda_{t}=\frac{y_{q,m,t}''-B(\bm{\vartheta};\phi_{q,t})}{\ln y_{q,m,t}''-\ln B(\bm{\vartheta};\phi_{q,t})}B(\bm{\vartheta};\phi_{q,t})\label{eq:beam-pattern-estimation-weight}
\end{equation}
then, the auxiliary problem (\ref{eq:beam-pattern-estimation-auxiliary})
is equivalent to the pattern estimation problem (\ref{eq:beam-pattern-estimation}).
Specifically, if $\mathbf{\bm{b}}$ is the solution to (\ref{eq:beam-pattern-estimation-auxiliary}),
then 
\begin{equation}
\omega_{q,m}=\mathrm{exp}\Big(b_{3}-\frac{b_{2}^{2}}{4b_{1}}\Big),\eta_{q,m}=-b_{1},c_{q,m}=-\frac{b_{2}}{2b_{1}}\label{eq:solution-beam-pattern-from-b}
\end{equation}
is the solution to (\ref{eq:beam-pattern-estimation}).

\end{myprop}

\begin{proof}

See Appendix \ref{sec:app-prop2}.

\end{proof}

Proposition \ref{prop:Equivalence-Condition} imposes a stronger,
sufficient (but not necessary) condition by requiring the subset $\mathcal{T}_{q,m}^{\epsilon}$
to include all samples; while this may not always hold initially,
the equivalence becomes a good approximation as optimization progresses
and parameter estimates improve, especially under high \ac{snr} or
near the optimum.

It is observed that a small value $y_{qm,t,}''\ll1$ tends to receive
a small weight $\lambda_{t}$, because the term $\ln y_{q,m,t}''$
in the the denominator of (\ref{eq:beam-pattern-estimation-weight})
has a large magnitude. As a result, the auxiliary problem tends to
focus more on the data with a large $y_{q,m,t}''$, i.e., the center
of the beam, which aligns with the goal of (P2).

Proposition \ref{prop:Equivalence-Condition} provides a convenience
way to solve the pattern estimation problem (\ref{eq:beam-pattern-estimation}),
because the auxiliary problem (\ref{eq:beam-pattern-estimation-auxiliary})
is a weighted least-squares linear regression problem, and the solution
is given by 
\begin{equation}
\mathbf{\bm{b}}=\Bigg(\sum_{t\in\mathcal{T}_{q,m}^{\epsilon}}\lambda_{t}\bm{\phi}_{q,t}\bm{\phi}_{q,t}^{\text{T}}\Bigg)^{-1}\Bigg(\sum_{t\in\mathcal{T}_{q,m}^{\epsilon}}(\lambda_{t}\ln y''_{q,m,t})\bm{\phi}_{q,t}\Bigg)\label{eq:solution-beam-pattern-auxiliary-b}
\end{equation}
which is obtained by setting the derivative of (\ref{eq:beam-pattern-estimation-auxiliary})
\ac{wrt} $\mathbf{\bm{b}}$ to 0, where $\bm{\phi}_{q,t}=[\phi_{q,t}^{2},\phi_{q,t},1]^{\text{T}}$.
In addition, the solution $\bm{\vartheta}=(w_{q,m},\eta_{q,m},c_{q,m})$
is obtained from (\ref{eq:solution-beam-pattern-from-b}).

Finally, as $\lambda_{t}$ depends on the solution $\bm{\vartheta}$,
an iterative approach can applied. At first, all weights are initialized
to $1$, and solution of $\bm{\vartheta}$ is obtained from (\ref{eq:solution-beam-pattern-from-b})
and (\ref{eq:solution-beam-pattern-auxiliary-b}). Then, $\lambda_{t}$
is updated according to (\ref{eq:beam-pattern-estimation-weight}),
followed by an update of $\bm{\vartheta}$ from (\ref{eq:solution-beam-pattern-from-b})
and (\ref{eq:solution-beam-pattern-auxiliary-b}), and iteration goes
on.

The overall algorithm is summarized in Algorithm \ref{alg:P2}.

\begin{algorithm}
Initialize the path loss parameter $\{\alpha_{q}^{(0)},\beta_{q}^{(0)}\}$
using (\ref{eq:solution-path-loss-estimation-2}) and $\sigma_{q}^{(0)}$
using (\ref{eq:solution-path-loss-estimation-sigma}).

Loop for the ($i+1$)th iteration:
\begin{itemize}
\item Initialize all weights $\lambda_{t}^{(0)}$ to $1$.
\item Loop for the ($j+1$)th iteration:
\begin{itemize}
\item Update $\{w_{q,m}^{(j+1)},\eta_{q,m}^{(j+1)},c_{q,m}^{(j+1)}\}$ from
(\ref{eq:solution-beam-pattern-from-b}) and (\ref{eq:solution-beam-pattern-auxiliary-b}).
\item Update $\lambda_{t}^{(j+1)}$ using (\ref{eq:beam-pattern-estimation-weight}).
\end{itemize}
\item Repeat Until $w_{q,m}^{(j+1)}=w_{q,m}^{(j)}$, $\eta_{q,m}^{(j+1)}=\eta_{q,m}^{(j)}$,
and $c_{q,m}^{(j+1)}=c_{q,m}^{(j)}$.
\item Update $\{\alpha_{q}^{(i+1)},\beta_{q}^{(i+1)}\}$ using (\ref{eq:solution-path-loss-estimation-1})
and $\sigma_{q}^{(i+1)}$ using (\ref{eq:solution-path-loss-estimation-sigma}).
\end{itemize}
Repeat Until $\alpha_{q}^{(i+1)}=\alpha_{q}^{(i)}$, $\beta_{q}^{(i+1)}=\beta_{q}^{(i)}$,
and $\sigma_{q}^{(i+1)}=\sigma_{q}^{(i)}$.

\caption{An alternating optimization procedure for propagation parameter estimation.\label{alg:P2}}
\end{algorithm}

\subsection{Solution to (P3) for Trajectory Optimization}

\label{subsec:Solution-to-(P3)}

\subsubsection{Problem Discretization}

Solving problem (P3) is challenging due to its non-convex nature and
the high-dimensional solution space. To address this, we propose a
discretization strategy that approximates the solution to problem
(P3) with linear complexity within a discrete domain. Specifically,
we discretize the area of interest into grid locations spaced equally
by $\tau$ meters, thereby constructing the location set $\mathcal{V}$.
The area is then represented as a graph $\mathcal{G}=(\mathcal{V},\mathcal{E})$,
where edges in the set $\mathcal{E}$ connect locations that are reachable
within a maximum of $K$ hops. The parameter $K$ is calculated as
$K=\left\lceil \delta v_{\mathrm{max}}/\tau\right\rceil $ where $v_{\mathrm{max}}$
is the maximum speed attainable by the mobile user.

We first discretized $p(\mathbf{x}_{t}|\mathbf{x}_{t-1},\mathbf{x}_{t-2};\hat{\bm{\Theta}}_{\text{{m}}})$
in (P3) as
\begin{align*}
 & \mathbb{P}(\mathbf{x}_{t}|\mathbf{x}_{t-1},\mathbf{x}_{t-2};\bm{\Theta}_{\text{m}})\\
 & =\frac{p(\mathbf{x}_{t}|\mathbf{x}_{t-1},\mathbf{x}_{t-2};\bm{\Theta}_{\text{m}})}{\sum_{\mathbf{v}\in\{\mathbf{v}\mid(\mathbf{x}_{t-1},\mathbf{v})\in\mathcal{E},\mathbf{v}\in\mathcal{V}\}}p(\mathbf{v}|\mathbf{x}_{t-1},\mathbf{x}_{t-2};\bm{\Theta}_{\text{m}})},
\end{align*}
 ensuring that the sum of transition probabilities at each time slot
equals 1.

\begin{algorithm}
Initialize the parameter $\bm{\Theta}_{\text{{p}}}^{(0)}$, $\bm{\Theta}_{\text{{m}}}^{(0)}$
randomly.

Loop for the ($i+1$)th iteration:
\begin{itemize}
\item Update $\mathcal{X}_{T}^{(i+1)}$ using the gradient descent method,
initialized with the output of the Viterbi algorithm.
\item Update $\bm{\Theta}_{\text{{p}}}^{(i+1)}$ using Algorithm \ref{alg:P2}.
\item Update $\bm{\Theta}_{\text{{m}}}^{(i+1)}$ using (\ref{eq:solution-v})-(\ref{eq:solution-v-sigma})
\end{itemize}
Until $\mathcal{X}_{T}^{(i+1)}=\mathcal{X}_{T}^{(i)}$.

\caption{An alternating optimization algorithm for trajectory recovery.\label{alg:alternative-opt}}
\end{algorithm}

Then, problem (P3) can be discretized to \textcolor{black}{
\begin{align*}
(\mathrm{P3.1}):\underset{\mathcal{X}_{T}}{\mathrm{maximize}} & \;\;\sum_{t=1}^{T}\Big\{\log p(\mathbf{y}_{t}|\mathbf{x}_{t};\hat{\bm{\Theta}}_{\text{{p}}})\\
 & \qquad+\mathbb{I}(t>2)\log\mathbb{P}(\mathbf{x}_{t}|\mathbf{x}_{t-1},\mathbf{x}_{t-2};\hat{\bm{\Theta}}_{\text{{m}}})\Big\}\\
\text{subject to} & \;\;\mathbf{x}_{t}\in\mathcal{V},\;(\mathbf{x}_{t},\mathbf{x}_{t-1})\in\mathcal{E}
\end{align*}
}Our goal is to find a trajectory within a discrete space that maximizes
the log-likelihood $\log p(\mathcal{Y}_{T},\mathcal{X}_{T})$ given
the signal propagation parameters $\hat{\bm{\Theta}}_{\text{{p}}}$
and mobility model parameters $\hat{\bm{\Theta}}_{\text{{m}}}$.

\subsubsection{Algorithm and Complexity}

\label{subsec:Algorithm-Vertibi}

Problem (P3.1) follows a classical \ac{hmm} optimization form, with
the distinction that the current state depends on the previous two
states. Problem (P3.1) can be efficiently solved using a modified
version of the Viterbi algorithm with globally optimal guarantee.

At each step, there are $|\mathcal{V}|$ candidate locations considered,
but states with very low probabilities $p(\mathbf{y}_{t}|\mathbf{x}_{t};\hat{\bm{\Theta}}_{\text{p}})$
are highly unlikely to contribute to the optimal path. To improve
efficiency, states with probabilities below a threshold $\zeta$ are
pruned. Mathematically, this corresponds to retaining only the top
$n_{t}(\zeta)$ most probable locations at time slot $t$, where $n_{t}(\zeta)$
is the number of elements in the set $\{\mathbf{x}_{t}\mid p(\mathbf{y}_{t}\mid\mathbf{x}_{t};\hat{\bm{\Theta}}_{\text{p}})>\zeta,\mathbf{x}_{t}\in\mathcal{V}\}$.
Denote the maximum number of element in the set $n_{\mathrm{max}}(\zeta)=\max_{t}\{n_{t}(\zeta)\}$.

Considering the number of candidate previous states for the current
state, which is constrained by the graph structure, it is of the order
$\mathcal{O}(K^{2})$ for a square region and $\mathcal{O}(K)$ for
an unbranched road network. Thus, we have the following result.

\begin{myprop}

\label{prop:The-computational-complexity}The computational complexity
of solving problem (P3.1) is upper bounded by $\mathcal{O}(Tn_{\mathrm{max}}(\zeta)(\delta v_{\mathrm{max}}/\tau)^{2})$
and lower bounded by $\mathcal{O}(Tn_{\mathrm{max}}(\zeta)(\delta v_{\mathrm{max}}/\tau))$
.

\end{myprop}

We can solve problem (P3.1) with linear complexity, as stated in Proposition
\ref{prop:The-computational-complexity}. Given the solution $\mathcal{X}_{T}^{(0)}$
of (P3.1), we employ the gradient descent method with a learning rate
$l_{r}$ to obtain a convergent solution to problem (P3), initialized
with $\mathcal{X}_{T}^{(0)}$. The overall algorithm is summarized
in Algorithm \ref{alg:alternative-opt}. We first initialize the propagation
parameter $\bm{\Theta}_{\text{{p}}}$ and and the mobility parameter
$\bm{\Theta}_{\text{{m}}}$ randomly and then begin the alternating
update of $\mathcal{X}_{T}$, $\bm{\Theta}_{\text{{p}}}$ and $\bm{\Theta}_{\text{{m}}}$
alternatively until convergence. Since each iteration of this procedure
never decreases the objective function, which is bounded above, the
iterative process is therefore guaranteed to converge. Moreover, as
problem (\ref{eq:prob-J}) is inherently non-convex, the algorithm
is not guaranteed to find a unique global optimum, and may yield different
solutions depending on the initialization.

\section{Numerical Results}

\label{sec:Numerical-Results}

\subsection{Datasets}

\label{subsec:Datasets}

\begin{figure}
\begin{centering}
\includegraphics[width=1\columnwidth]{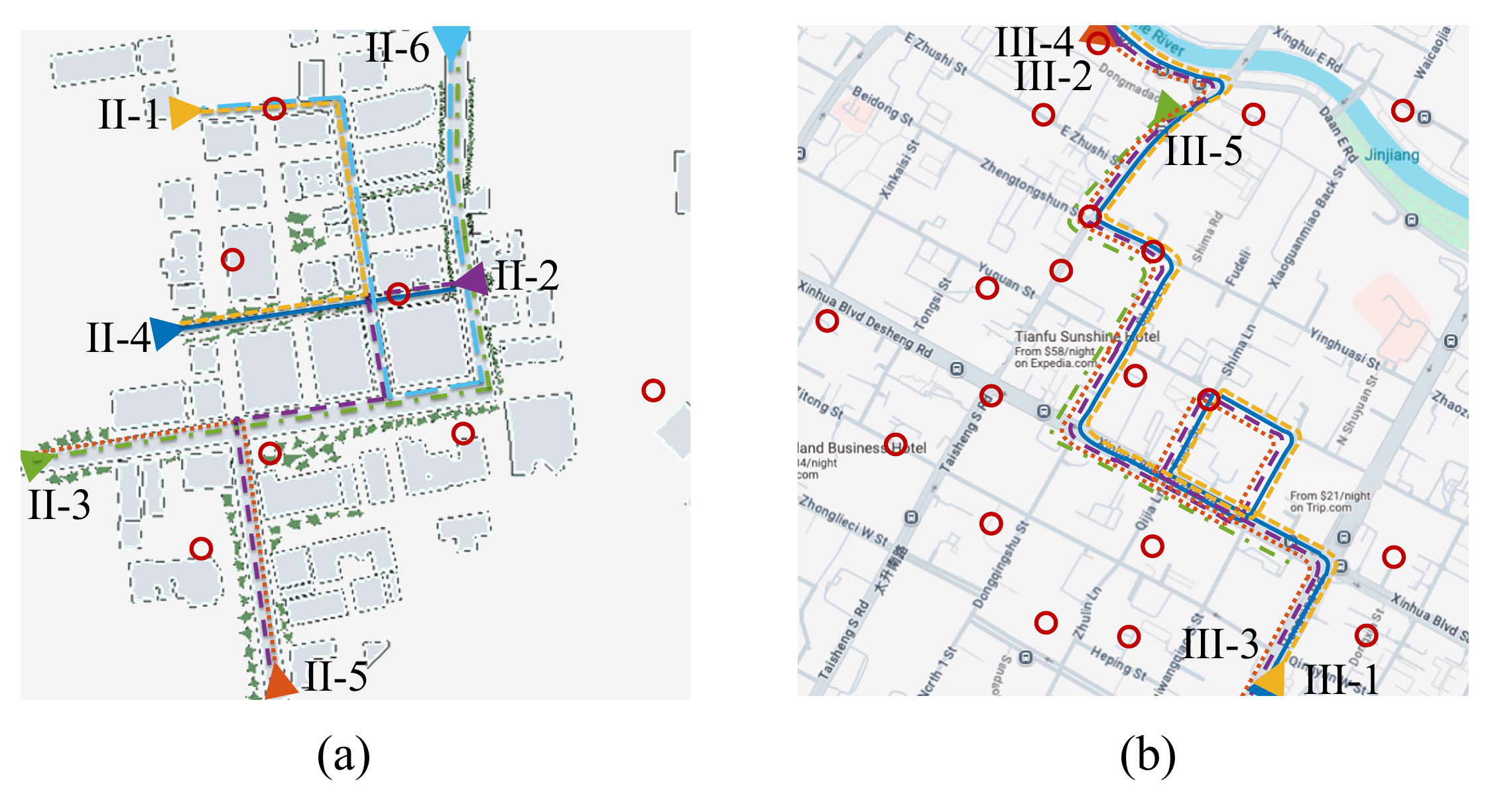}\vspace{-0.1in}
\par\end{centering}
\centering{}\caption{The data collection environment of (a) synthetic \ac{mimo} dataset
and (b) real \ac{mimo} dataset. The signal from \acpl{bs} (gray
points) is measured along the trajectory (distinct line styles and
colors, begin with triangles). \label{fig:DatasetShow}}
\vspace{-0.1in}
\end{figure}

This paper validates the proposed algorithm using three datasets:
\begin{itemize}
\item \textbf{Synthetic Single-Antenna Dataset I:} We simulate a trajectory
of length $1\times10^{5}$ meters using the mobility model defined
in (\ref{eq:mobility-model}), parameterized by $\gamma=1$, $\mathbf{v}=[10,0]^{\mathrm{T}}$
m/s, $\mathbf{x}=[0,0]^{\mathrm{T}}$ m, and $\delta=0.5$ seconds.
We consider \ac{bs} at a height of 50 meters and a mobile user at
a height of 2 meters equipped with a single antenna. The parameters
in the signal propagation model (\ref{eq:measurement-model}) are
$N_{\text{t}}=1$, $\beta_{q}=5$, $\alpha_{q}=-20$, and $\sigma_{q}=0.2,0.5,1$.
Two scenarios are considered: in Scenario 1 (\ac{bs} deployed in
a limited region, c.f., Section III-B), the number of \acpl{bs} surrounding
the trajectory is fixed at $Q=4$, 8, 12, 16, 20; in Scenario 2 (\ac{bs}
deployed in an unlimited region, c.f., Section III-C), the \acpl{bs}
in the target area follow a \ac{ppp} with densities $\kappa=1.02\times10^{-3}$,
$3.02\times10^{-3}$, $5.02\times10^{-3}$, $7.02\times10^{-3}$,
$9.02\times10^{-3}$, and $1.02\times10^{-2}$ units per $\mathrm{m}^{2}$.
The mobile user can only detect \acpl{bs} within a radius of $R=50$,
100, 200, 300, 400, 500 meters.
\item \textbf{Synthetic \ac{mimo} Dataset} \textbf{II:} We utilized Wireless
Insite\foreignlanguage{english}{$^{\circledR}$} to simulate a environment
encompassing a 700 m \texttimes{} 700 m area in San Francisco, USA,
featuring building heights ranging from 12 m to 204 m. As illustrated
in Figure \ref{fig:DatasetShow}(a), seven \acpl{bs} with a height
of 55 meters, each were manually deployed on selected rooftops to
ensure comprehensive coverage of the area of interest. Each \ac{bs}
is equipped with a 64-antenna dual-polarized \ac{mimo} array and
configured with $M=7$ beams. The antenna orientation of each \ac{bs}
spans approximately $120^{\circ}$, and the transmit power is set
to $0$ dBm. The sensing vectors are constructed using the Kronecker
product of the array responses and are utilized to compute the \ac{ssb}
\ac{rsrp} by evaluating signal strengths across various receivers
equipped with isotropic antennas along the route. We recorded the
\ac{ssb} \ac{rsrp} in receivers positioned at a height of 2 meters
long six predefined trajectories with lengths 675 m, 792 m, 1085 m,
355 m, 627 m, and 1182 m, called II-1$\sim$6. All measurements are
conducted at speeds of 6, 12, 18, 24, 30 m/s with a sampling interval
of $\text{\ensuremath{\delta}}=0.5$ s.
\item \textbf{Real \ac{mimo} Dataset III: }We conducted a driving procedure
in an urban area of China, covering a 1350 m \texttimes  1350 m region,
where we collected \ac{gps}-reported location data, as well as the
\ac{rsrp}, \ac{rssi}, and \ac{sinr} of 32 \ac{csi} beams from
the serving cell. Additionally, we measured the \ac{rsrp} of $M=8$
\ac{ssb} beams both the serving and neighboring cells using a 5G-enabled
smartphone. The receiver reliably acquired 8 beam values from the
\ac{bs} in the serving cell, while \acpl{bs} in neighboring cells
provided only 0 to 14 beam values due to device limitations and signal
propagation interference. The vehicle traversed five distinct trajectories,
labeled III-1 through III-5, and measured signals from 39 surrounding
\acpl{bs}. Specifically, Trajectory III-1 followed a 2652 m path
with speeds between 0 and 13.1 m/s and a sampling interval of $\delta=0.5$
s, yielding 692 samples. Trajectory III-2 retraced III-1 in reverse,
III-3 doubled the average speed on III-1, III-4 doubled the speed
while reversing III-1, and III-5 covered a 1253 m segment of III-1.
\end{itemize}

\subsection{Numerical Validation of the Theoretical Results}

\begin{table*}[t]
\centering{}\caption{Comparison of average localization error on synthetic \ac{mimo} dataset
II (six trajectories) and real \ac{mimo} dataset III (five trajectories).
\label{tab:traj-performance}}
\begin{tabular}{>{\raggedright}V{\linewidth}|>{\centering}m{1.2cm}>{\centering}m{1.2cm}>{\centering}m{1.2cm}>{\centering}m{1.2cm}>{\centering}m{1cm}|>{\centering}m{1.2cm}>{\centering}m{1.2cm}>{\centering}m{1.2cm}>{\centering}m{1.2cm}>{\centering}m{1cm}}
\hline 
 & \multicolumn{5}{c|}{without a road map} & \multicolumn{5}{c}{use a road map to constrain the trajectory}\tabularnewline
 & MaR\cite{IqSha:J24} & WCL\cite{MagGioKanYu:J18} & Proposed ($M=1$) & Proposed & GMA & MaR\cite{IqSha:J24} & WCL\cite{MagGioKanYu:J18} & Proposed ($M=1$) & Proposed & GMA\tabularnewline
\hline 
Dataset II & \centering{}47.6 & 41.4 & 10.7 & \underline{7.2} & \textbf{7.0} & 42.3 & 34.8 & 9.1 & \underline{7.0} & \textbf{6.3}\tabularnewline
Dataset III & 167.8 & 124.5 & 22.7 & \underline{18.7} & \textbf{17.8} & 102.2 & 86.3 & 19.2 & \underline{15.9} & \textbf{14.9}\tabularnewline
\hline 
\end{tabular}\vspace{-0.1in}
\end{table*}

\begin{figure}
\centering{}\includegraphics[width=1\columnwidth]{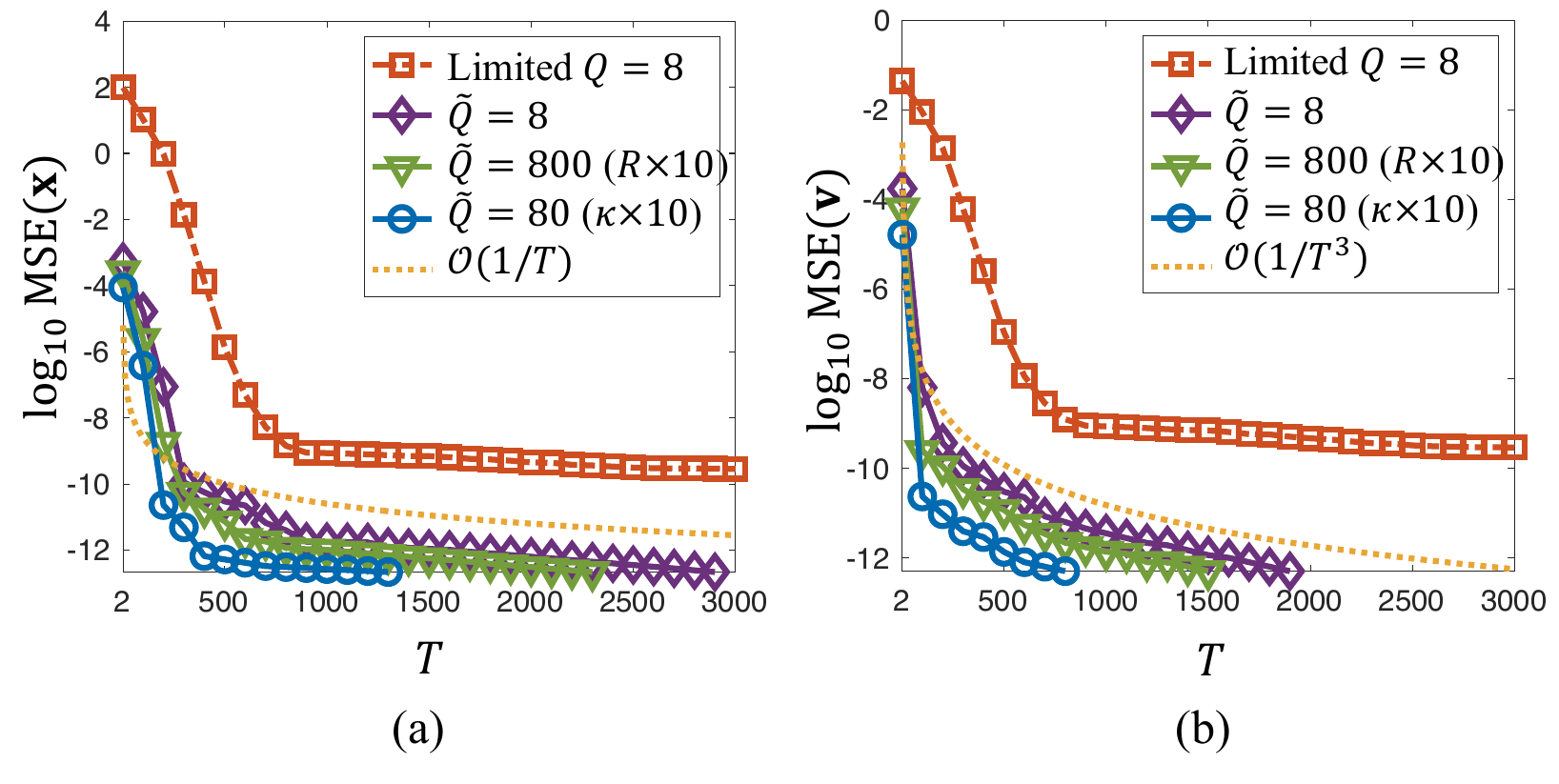}\vspace{-0.1in}
\caption{\ac{mse} of (a) $\mathbf{x}$ and (b)\textbf{ $\mathbf{v}$} with
different sample number $T$, the number of \ac{bs} $Q$, radius
$R$, and density $\kappa$.\label{fig:vx-T}}
\vspace{-0.1in}
\end{figure}

Figure \ref{fig:vx-T} illustrates the \ac{mse} defined as MSE($\mathbf{x}$)$=\|\mathbf{x}-\hat{\mathbf{x}}\|_{2}^{2}$
m$^{2}$ and MSE($\mathbf{v}$)$=\|\mathbf{v}-\hat{\mathbf{v}}\|_{2}^{2}$
on the synthetic single-antenna dataset with the parameter $\beta_{q}=5$,
$\alpha_{q}=-20$, and $\sigma_{q}=0.1$, where $\hat{\mathbf{x}}$
and $\hat{\mathbf{v}}$ are the outputs of the proposed algorithm.

In Scenario 1 of the synthetic single-antenna dataset with $Q=8$,
the \ac{mse} of $\mathbf{x}$ and $\mathbf{v}$ decreases as $T$
increases within a limited region but does not converge to zero even
when $T=20000$ in our experiments. This behavior is consistent with
Theorem \ref{thm:LB-F-x} and Theorem \ref{thm:LB-F-v}.

In Scenario 2 of the synthetic single-antenna dataset, we set $R=50$
m and $\kappa=1.02\times10^{-3}$ units per m$^{2}$, resulting in
$\text{\ensuremath{\tilde{Q}}}\approx8$. As $T$ increases, the rate
at which MSE($\mathbf{x}$) decreases follows $\mathcal{O}(1/T)$,
and the rate at which MSE($\mathbf{v}$) decreases follows $\mathcal{O}(1/T^{3})$,
which is consistent with Theorem \ref{thm:LB-F-x-un} and Theorem
\ref{thm:LB-F-v-un}. The MSE($\mathbf{x}$) and MSE($\mathbf{v}$)
for the curves $\tilde{Q}=8$ and $\tilde{Q}=800$ in Figure \ref{fig:vx-T}
both reach zero when $T>3200$. In addition, we found that increasing
the radius $R$ from 50 to 500 meters results in a lower \ac{mse},
and increasing the density $\kappa$ from $1.02\times10^{-3}$ to
$1.02\times10^{-2}$ also yields a lower \ac{mse}.

\begin{figure}
\centering{}\includegraphics[width=1\columnwidth]{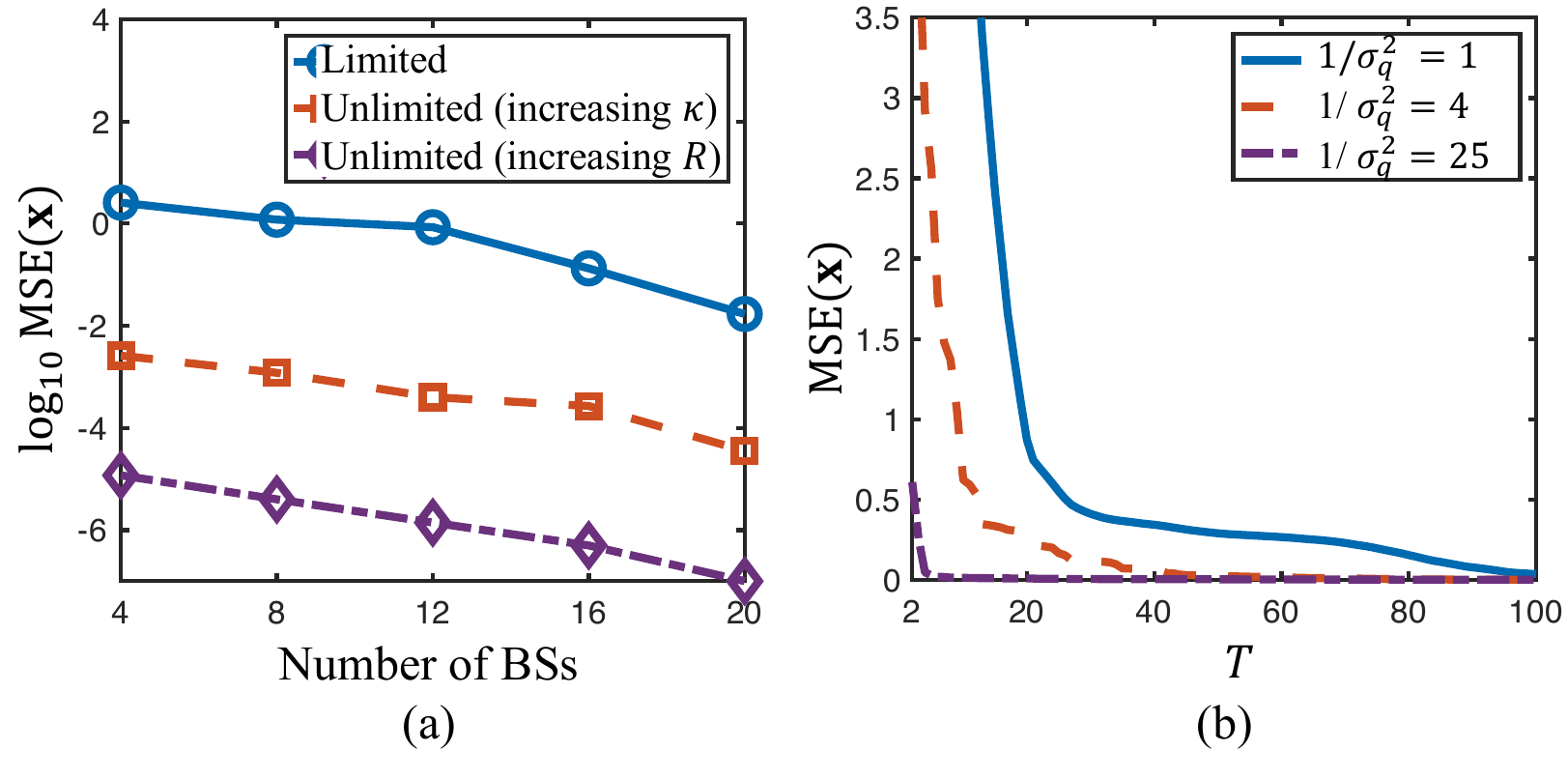}\vspace{-0.1in}
\caption{(a) The relationship between MSE($\mathbf{x}$) and the number of
\acpl{bs}. (b) MSE($\mathbf{x}$) under different noise $\sigma_{q}$.
\label{fig:Q-sigma}}
\vspace{-0.1in}
\end{figure}

We investigate the effect of the number of \acpl{bs}, $Q$ and $\tilde{Q}$,
on MSE($\mathbf{x}$) and MSE($\mathbf{v}$). The trajectory length
is set to 100 m, with $\beta_{q}=5$, $\alpha_{q}=-5$, and $\sigma_{q}=0.1$.
For Scenario 2, we consider a fixed radius $R=50$ m and incrementally
increase the density $\kappa$ to achieve $\tilde{Q}=4,8,...,20$.
Additionally, we also examine a fixed $\kappa=1.02\times10^{-3}$
units per m$^{2}$ while increasing $R$ to obtain $\tilde{Q}=4$,
8,..., 20. As shown in Figure \ref{fig:Q-sigma}(a), the unlimited
scenario consistently achieves a lower \ac{mse} compared to the limited
scenario under the same number of \acpl{bs}. This observation is
also validated in Figure \ref{fig:vx-T}, where Scenario 2 with $\text{\ensuremath{\tilde{Q}}}\approx8$
achieves a lower \ac{mse} compared to Scenario 1 with $Q=8$. Furthermore,
we found that increasing $R$ results in a lower \ac{mse} than increasing
$\kappa$ under the same number of \acpl{bs}. This is because $\tilde{\Delta}_{T,x}$
and $\tilde{\Delta}_{T,v}$ in Theorem \ref{thm:LB-F-x-un} and Theorem
\ref{thm:LB-F-v-un} is related to $\mathcal{O}(1/(\kappa\ln R)$.

We also investigate the effect of the noise variance $\sigma_{q}^{2}$
under the unlimited scenario, with $R=50$ m, $\kappa=1.02\times10^{-3}$
per m$^{2}$, $\beta_{q}=5$, $\alpha_{q}=-5$, and a trajectory length
of $500$ m. We consider $1/\sigma_{q}^{2}=1,4,25$ for all \acpl{bs}.
As shown in Figure \ref{fig:Q-sigma}(b), a larger $1/\sigma_{q}^{2}$
results in a faster convergence rate. Recall that $1/\sigma_{\max}^{2}$
is proportional to the \ac{snr}. Thus, a smaller $1/\sigma_{q}^{2}$
leads to a slower decrease in the \ac{crlb} of $\mathbf{x}$ and
$\mathbf{v}$ as $T$ increases, as stated in Theorem \ref{thm:LB-F-x-un}
and Theorem \ref{thm:LB-F-v-un}.

\subsection{Trajectory Recovery Performance}

We use the average localization error $E_{\mathrm{l}}=\frac{1}{T}\sum_{t=1}^{T}\|\mathbf{x}_{t}-\hat{\mathbf{x}}_{t}\|_{2}$
to evaluate the trajectory recovery performance of the proposed method.
Here, $\mathbf{x}_{t}$ is the data collection location at time slot
$t$, and $\hat{\mathbf{x}}_{t}$ is the $t$th location in the output
trajectory of the proposed algorithm. Four baselines are designed
for comparison:
\begin{itemize}
\item Max-\ac{rss} (MaR): At time $t$, the strongest \ac{rss} among the
surrounding \acpl{bs} is selected, and the estimated position is
determined by the location of the \ac{bs} with the strongest \ac{rss}.
\item Weighted Centroid Localization (WCL): This method computes a weighted
location $\hat{\mathbf{p}}_{t}=\sum_{q=1}^{Q}w_{t,q}\mathbf{o}_{q}$,
where $w_{t,q}=\sum_{j=1}^{M}10^{y_{q,m,t}/20}/[\sum_{j=1}^{M}\sum_{l=1}^{Q}10^{y_{q,j,l}/20}]$.
\item Proposed ($M=1$) \cite{XinChe:C24}: This method uses only the max
\ac{rsrp} among the $M$ beams of a \ac{bs} as the \ac{rss} for
that \ac{bs}, and considers only the path-loss model (\ref{eq:path-loss})
in the signal propagation probability.
\item Genius-aided map-assisted (GMA): This method utilizes the true location
information available in the real MIMO dataset III (via GPS) to fit
the propagation models described by equation (\ref{eq:measurement-model})
and assumes that the propagation model parameters are known. Only
the mobility model parameters and the trajectory are updated alternately
until convergence. The GMA method serves as an upper bound for performance
comparison.
\item Direct AI positioning: We adopt the 'fingerprinting based on channel
observation' method, as defined in \cite{Report3GPP}, for direct
AI positioning. A multi-layer perceptron (MLP) is constructed with
fully connected layers of dimensions 49, 128, 64, 16, and 2. This
method is supervised, requiring user location labels for training.
\end{itemize}
For the comparison, if the road map knowledge is given, the output
location of the comparison is projected to the nearest road as stated
in \cite{Hu:J23} to utilize the map knowledge. For the proposed method,
the target area will be the road network space if the map information
is given, and the whole target area otherwise. We set $\tau=1$ m,
$v_{\mathrm{max}}=120$ km/h, $\zeta=0.8$, $l_{r}=0.01$, $\epsilon=0.01$,
and $\gamma=0.9$ for the proposed method.

\subsubsection{Synthetic \ac{mimo} Dataset}

We first evaluate the performance of the proposed method on a synthetic
\ac{mimo} dataset. As shown in Table \ref{tab:traj-performance},
first, MaR exhibits the poorest performance, demonstrating that relying
solely on proximity to \acpl{bs} is ineffective for accurate mobile
user positioning. While the WCL method improves estimation by applying
weights to the distances from surrounding \acpl{bs}, it only marginally
outperforms the MaR method. Second, although the proposed $(M=1)$
method outperforms MaR and WCL, its performance is slightly inferior
to the proposed method that accounts for beam effects due to the absence
of beam consideration, underscoring the importance of incorporating
angular domain considerations in modeling signal propagation. Third,
the proposed method consistently outperforms the comparison methods
MaR and WCL, including the $M=1$ variant. Additionally, the proposed
method performs slightly worse than GMA. This discrepancy arises because
GMA and the proposed method solve the same optimization problem; however,
GMA assumes a known real signal propagation model, thereby establishing
an upper bound for the proposed method. Nevertheless, the gap between
them is only less than 1 meter. Finally, compared to the method with
map information, the method without map exhibits a slight increase
in localization error, suggesting that all methods perform better
when road network information is available.

\begin{table}[t]
\begin{centering}
\caption{Average localization error ($E_{\mathrm{l}}$) on the six trajectories
separately in synthetic \ac{mimo} dataset with the speed 6 m/s under
map knowledge free. \label{tab:PerformanceII}}
\par\end{centering}
\begin{centering}
\begin{tabular}{l|cccccc}
\hline 
Trajectory & II-4 & II-5 & II-3 & II-1 & \multicolumn{1}{c}{II-2} & II-6\tabularnewline
Length {[}m{]} & 355 & 627 & 1085 & 675 & \multicolumn{1}{c}{792} & 1182\tabularnewline
No. of turns & 0 & 1 & 1 & 2 & \multicolumn{1}{c}{3} & 3\tabularnewline
\hline 
$E_{\mathrm{l}}$ {[}m{]} & 4.39 & 7.41 & 6.67 & 7.80 & 8.03 & 7.88\tabularnewline
\hline 
\end{tabular}
\par\end{centering}
\centering{}\vspace{-0.1in}
\end{table}

We investigate the trajectory recovery performance by varying the
trajectory length and the complexity of the trajectory in terms of
the number of turns. As shown in Table \ref{tab:PerformanceII}, firstly,
increasing the trajectory length indeeds decreases the localization
error. This can be seen by comparing the \ac{mse} of trace II-3 and
II-5, and the \ac{mse} of trace II-2 and II-6. Second, for roughly
the same trajectory length, making a turn can slightly decreases the
localization performance. As seen from traces II-1, II-2, II-4, II-5,
the more turns, the higher the \ac{mse}.

\begin{table}[t]
\begin{centering}
\caption{Average localization error ($E_{\mathrm{l}}$) on all six trajectories
in synthetic \ac{mimo} dataset with different speeds under map knowledge
free. \label{tab:localization-error-speed}}
\par\end{centering}
\centering{}%
\begin{tabular}{l|ccccc}
\hline 
Speed m/s & 6 & 12 & 18 & 24 & 30\tabularnewline
$E_{\mathrm{l}}$ {[}m{]} & 7.29 & 8.12 & 8.36 & 8.77 & 9.49\tabularnewline
\hline 
\end{tabular}\vspace{-0.1in}
\end{table}

We also investigate the effect of moving speed on trajectory recovery
performance by varying the speed from 6 m/s to 30 m/s. As shown in
Table \ref{tab:localization-error-speed}, the error gradually increases
with increasing speed. For a fixed trajectory length, increasing the
speed results in a reduced number of samples. As stated in Theorems
\ref{thm:LB-F-x}$\sim$\ref{thm:LB-F-v-un}, the \ac{crlb} of the
\ac{mse} decreases as $T$ increases.

For comparison with the Direct AI method, we use datasets II-1, II-2,
and II-3 for training, and datasets II-4, II-5, and II-6 for testing.
Note that, the trajectories in the test datasets are subsets of that
in the training datasets. As shown in Table \ref{tab:Comparison-supervised},
the Direct AI method is able to achieve a small training error, but
the test performance is poor. the proposed method (w/o road map) outperforms
the Direct AI approach. This is primarily because our method directly
estimates the entire trajectory with mobility model constraints, whereas
the Direct AI approach predicts each location independently at each
time slot. Additionally, our method incorporates a \ac{bs}-specific
channel model for more accurate \ac{rsrp} fitting, while the Direct
AI approach does not leverage any underlying physical models.

\begin{table}[t]
\centering{}\caption{Comparison of average localization error ($E_{\mathrm{l}}$ {[}m{]})
performance between the proposed method and the supervised method
on synthetic \ac{mimo} dataset II.\label{tab:Comparison-supervised}}
\begin{tabular}{c|cccc}
\hline 
 & \multirow{2}{*}{Training error (m)} & \multicolumn{3}{c}{Test error (m)}\tabularnewline
\cline{3-5} \cline{4-5} \cline{5-5} 
 &  & II-4 & II-5 & II-6\tabularnewline
\hline 
Direct AI positioning & 6.24 & 8.35 & 10.23 & 9.59\tabularnewline
Proposed (w/o road map) & (no training) & 4.39 & 7.41 & 7.88\tabularnewline
\hline 
\end{tabular}
\end{table}

\subsubsection{Real \ac{mimo} Dataset}

\begin{figure}[t]
\begin{centering}
\includegraphics[width=1\columnwidth]{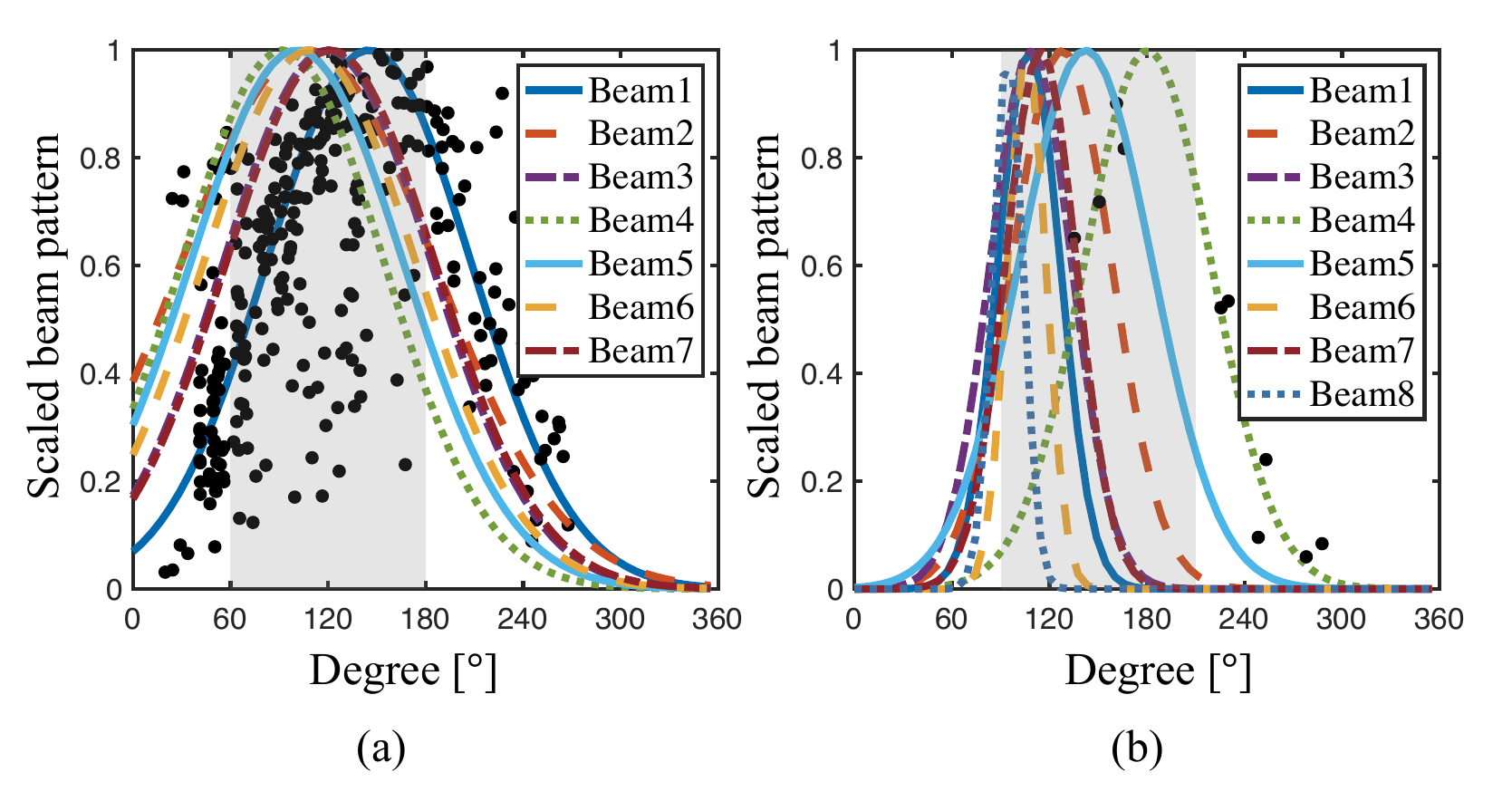}\vspace{-0.1in}
\par\end{centering}
\centering{}\caption{(a) The seven scaled patterns of the $q$th \ac{bs} in synthetic
\ac{mimo} dataset with measurements (black points) belonging to beam
1. (b) The eight scaled patterns of the $q$th \ac{bs} in real \ac{mimo}
dataset with measurements (black points) belonging to beam 4. \label{fig:example-model-plot}}
\vspace{-0.1in}
\end{figure}

Similar to the performance observed in the synthetic \ac{mimo} dataset,
the proposed method still achieves the lowest average localization
error on real \ac{mimo} datasets compared with MaR, WCL, and Proposed
($M=1$), and shows a very small gap to the GMA method. However, we
observed that the performance on the real \ac{mimo} dataset is inferior
compared to the synthetic \ac{mimo} dataset. To investigate the source
of this discrepancy, we conducted the following analysis.

Although the synthetic dataset comprises only seven \acpl{bs} compared
to 39 \acpl{bs} in the real \ac{mimo} dataset, in the real \ac{mimo}
dataset, an average of only 16 values can be measured for each beam
in the whole trajectory, whereas the synthetic \ac{mimo} dataset
allows for the measurement of approximately 1050 values for each beam.
The sparsity observed in the real \ac{mimo} dataset arises because
the mobile device can measure beam values from only six neighboring
\acpl{bs} and records only the strongest beam value for each neighboring
\ac{bs}. As illustrated in Figure \ref{fig:example-model-plot},
the amount of data used to fit beam 1 in the synthetic \ac{mimo}
dataset is significantly larger than that used to fit beam 4 in the
real \ac{mimo} dataset. This larger dataset facilitates a more accurate
estimation of the beam pattern in the synthetic scenario.

Figure \ref{fig:example-model-plot}(b) demonstrates the beam pattern
fitting, where the standard deviation $\sigma_{q}$ is found to be
0.25. The reconstructed beam patterns align with the expectations
communicated by the network operator: eight beams point in different
directions for spatial multiplexing. The beam coverage spans approximately
120 degrees; for example, beams are concentrated between $60^{\circ}$
and $180^{\circ}$ in the synthetic \ac{mimo} dataset and between
$90^{\circ}$ and $210^{\circ}$ in the real \ac{mimo} dataset. The
estimated channel model parameters vary across different \acpl{bs}
in Dataset III. This variation reflects the diversity of propagation
environments experienced by each \ac{bs}.

\subsection{Application: \ac{csi} Prediction}

\label{sec:Integrated-Loc-pred}

\selectlanguage{english}%
\begin{table*}[t]
\centering{}\centering{}\caption{\foreignlanguage{american}{Performance of \ac{csi} beam \ac{sinr} and \ac{rssi} prediction.
\label{tab:Perfor-SINR-RSSI}}}
\begin{tabular}{>{\centering}p{1.5cm}|>{\centering}p{1cm}>{\centering}p{1cm}>{\centering}p{1cm}>{\centering}p{1cm}>{\centering}p{1.6cm}|>{\centering}p{1cm}>{\centering}p{1cm}>{\centering}p{1cm}>{\centering}p{1cm}>{\centering}p{1.6cm}}
\hline 
\foreignlanguage{american}{} & \multicolumn{5}{c|}{\selectlanguage{american}%
\ac{sinr}\selectlanguage{english}%
} & \multicolumn{5}{c}{\selectlanguage{american}%
\ac{rssi}\selectlanguage{english}%
}\tabularnewline
\foreignlanguage{american}{} & \foreignlanguage{american}{MI} & \foreignlanguage{american}{AR} & \foreignlanguage{american}{CNN} & \foreignlanguage{american}{LSTM} & \foreignlanguage{american}{Proposed} & \foreignlanguage{american}{MI} & \foreignlanguage{american}{AR} & \foreignlanguage{american}{CNN} & \foreignlanguage{american}{LSTM} & \foreignlanguage{american}{Proposed}\tabularnewline
\hline 
$E_{\mathrm{q}}(1)$ & \foreignlanguage{american}{0.86} & \foreignlanguage{american}{0.68} & \foreignlanguage{american}{0.61} & \foreignlanguage{american}{0.59} & \foreignlanguage{american}{\textbf{0.42}} & \foreignlanguage{american}{0.96} & \centering{}0.70 & \foreignlanguage{american}{0.68} & \foreignlanguage{american}{0.69} & \foreignlanguage{american}{\textbf{0.43}}\tabularnewline
$E_{\mathrm{q}}(8)$ & \foreignlanguage{american}{0.70} & \foreignlanguage{american}{0.53} & \foreignlanguage{american}{0.52} & \foreignlanguage{american}{0.42} & \foreignlanguage{american}{\textbf{0.34}} & \foreignlanguage{american}{0.70} & \foreignlanguage{american}{0.50} & \foreignlanguage{american}{0.46} & \foreignlanguage{american}{0.47} & \foreignlanguage{american}{\textbf{0.35}}\tabularnewline
$E_{\mathrm{q}}(16)$ & \foreignlanguage{american}{0.47} & \foreignlanguage{american}{0.27} & \foreignlanguage{american}{0.22} & \foreignlanguage{american}{0.19} & \foreignlanguage{american}{\textbf{0.13}} & \foreignlanguage{american}{0.56} & \foreignlanguage{american}{0.37} & \foreignlanguage{american}{0.31} & \foreignlanguage{american}{0.25} & \foreignlanguage{american}{\textbf{0.16}}\tabularnewline
\hline 
$E_{\mathrm{e}}(1)$ & \foreignlanguage{american}{0.63} & \foreignlanguage{american}{0.56} & \foreignlanguage{american}{0.46} & \foreignlanguage{american}{0.41} & \foreignlanguage{american}{\textbf{0.23}} & \foreignlanguage{american}{0.76} & \foreignlanguage{american}{0.50} & \foreignlanguage{american}{0.49} & \foreignlanguage{american}{0.46} & \foreignlanguage{american}{\textbf{0.31}}\tabularnewline
$E_{\mathrm{e}}(8)$ & \foreignlanguage{american}{0.60} & \foreignlanguage{american}{0.49} & \foreignlanguage{american}{0.39} & \foreignlanguage{american}{0.30} & \foreignlanguage{american}{\textbf{0.28}} & \foreignlanguage{american}{0.61} & \foreignlanguage{american}{0.46} & \foreignlanguage{american}{0.40} & \foreignlanguage{american}{0.38} & \foreignlanguage{american}{\textbf{0.27}}\tabularnewline
$E_{\mathrm{e}}(16)$ & \foreignlanguage{american}{0.41} & \foreignlanguage{american}{0.32} & \foreignlanguage{american}{0.27} & \foreignlanguage{american}{0.21} & \foreignlanguage{american}{\textbf{0.10}} & \foreignlanguage{american}{0.54} & \foreignlanguage{american}{0.38} & \foreignlanguage{american}{0.29} & \foreignlanguage{american}{0.22} & \foreignlanguage{american}{\textbf{0.12}}\tabularnewline
\hline 
\end{tabular}\foreignlanguage{american}{\vspace{-0.1in}
}
\end{table*}

\selectlanguage{american}%
When the trajectories are recovered based on unlabeled \ac{rsrp}
measurements of the \ac{ssb} beam, a {\em radio map} can be constructed
by pairing the recovered location labels with the \ac{csi} data.
Note that such a construction can be done in an accumulative way to
keep improving the accuracy of the radio map. When \ac{csi} prediction
is needed \cite{LiuSin:C12}, a sequence of {\em sparse} and {\em coarse}
\ac{ssb} \ac{csi} data is observed over consecutive $L+1$ time
slots, and then, the full \ac{csi} data can be recovered from the
radio map using a finger-printing approach; consequently, applications
such as \ac{csi} tracking and quality-of-service prediction may follow.

Consider a constructed radio map includes signal propagation parameters,
\ac{rsrp} of \ac{ssb} beams, and \ac{rsrp}\foreignlanguage{english}{,
}\ac{rssi}\foreignlanguage{english}{ and }\ac{sinr}\foreignlanguage{english}{
of }\ac{csi}\foreignlanguage{english}{ beams}, along with the location
label. Given the current \ac{ssb} \ac{rsrp} measurement $\mathbf{y}_{t}$
and its $L$-length history $\{\mathbf{y}_{i}\}_{i=t-L}^{t-1}$, our
objective is to predict the \ac{rsrp}\foreignlanguage{english}{,
}\ac{rssi}\foreignlanguage{english}{ and }\ac{sinr}\foreignlanguage{english}{
of }\ac{csi}\foreignlanguage{english}{ beams} for the next time slot
$t+1$.

We begin by addressing the estimation of \ac{ssb} \ac{rsrp} $\mathbf{y}_{t+1}$
at time slot $t+1$ through maximizing $\sum_{i=t-L}^{t}\log p(\mathbf{y}_{i}|\mathbf{x}_{i})+\log p(\mathbf{y}_{t+1}|\bar{\mathbf{x}}_{t+1})+\sum_{j=t-L_{p}+2}^{t}\log p\{\mathbf{x}_{j}|\mathbf{x}_{j-1},\mathbf{x}_{j-2}\}$
\ac{wrt} $\mathbf{y}_{t+1}$, where the location for the next time
slot is estimated as $\bar{\mathbf{x}}_{t+1}=(1+\gamma)\mathbf{x}_{t}-\gamma\mathbf{x}_{t-1}+(1-\gamma)\delta\bar{\mathbf{v}}$.
This optimization problem is solved by alternately updating $\{\mathbf{x}_{i}\}_{i=t-L_{p}}^{t}$
and $\{\bar{\mathbf{v}},\sigma_{v}^{2}\}$ based on Algorithm \ref{alg:alternative-opt}.
Subsequently, the estimated location $\bar{\mathbf{x}}_{t+1}$ is
calculated, and the estimated \ac{ssb} \ac{rsrp} $\hat{\mathbf{y}}_{t+1}$
is obtained using the measurement model in (\ref{eq:measurement-model}).
Finally, the \ac{rsrp}\foreignlanguage{english}{, }\ac{rssi}\foreignlanguage{english}{
and }\ac{sinr}\foreignlanguage{english}{ of }\ac{csi}\foreignlanguage{english}{
beams} at time slot $t+1$ are predicted by matching $\hat{\mathbf{y}}_{t+1}$
with the constructed radio map through a nearest \ac{ssb} \ac{rsrp}
search.

We construct the radio map using III-1 in the real \ac{mimo} dataset
and evaluate \ac{csi} prediction performance on III-2 to III-5. We
compare our method with four baselines: Mean Inference (MI), which
computes the average beam variation over a window of length $L$ and
adds this average to the current measurement $\mathbf{y}_{t}$ to
estimate $\hat{\mathbf{y}}_{t+1}$; AutoRegressive (AR), a linear
order-$L$ model trained via gradient descent; Convolutional Neural
Network (CNN), featuring three $3\times3$ convolutional layers with
ReLU activations and a fully connected layer; and Long Short-Term
Memory Network (LSTM), comprising three LSTM layers with dropout and
a fully connected output. All models input sequences of past normalized
\ac{ssb} \ac{rsrp} measurements to predict the \ac{ssb} \ac{rsrp}
of the next time slot, with CNN and LSTM trained using the Adam optimizer
for up to 1000 epochs at a learning rate of 0.001, and $L$ set to
12 for all methods.

\begin{figure}[t]
\begin{centering}
\includegraphics[width=1\columnwidth]{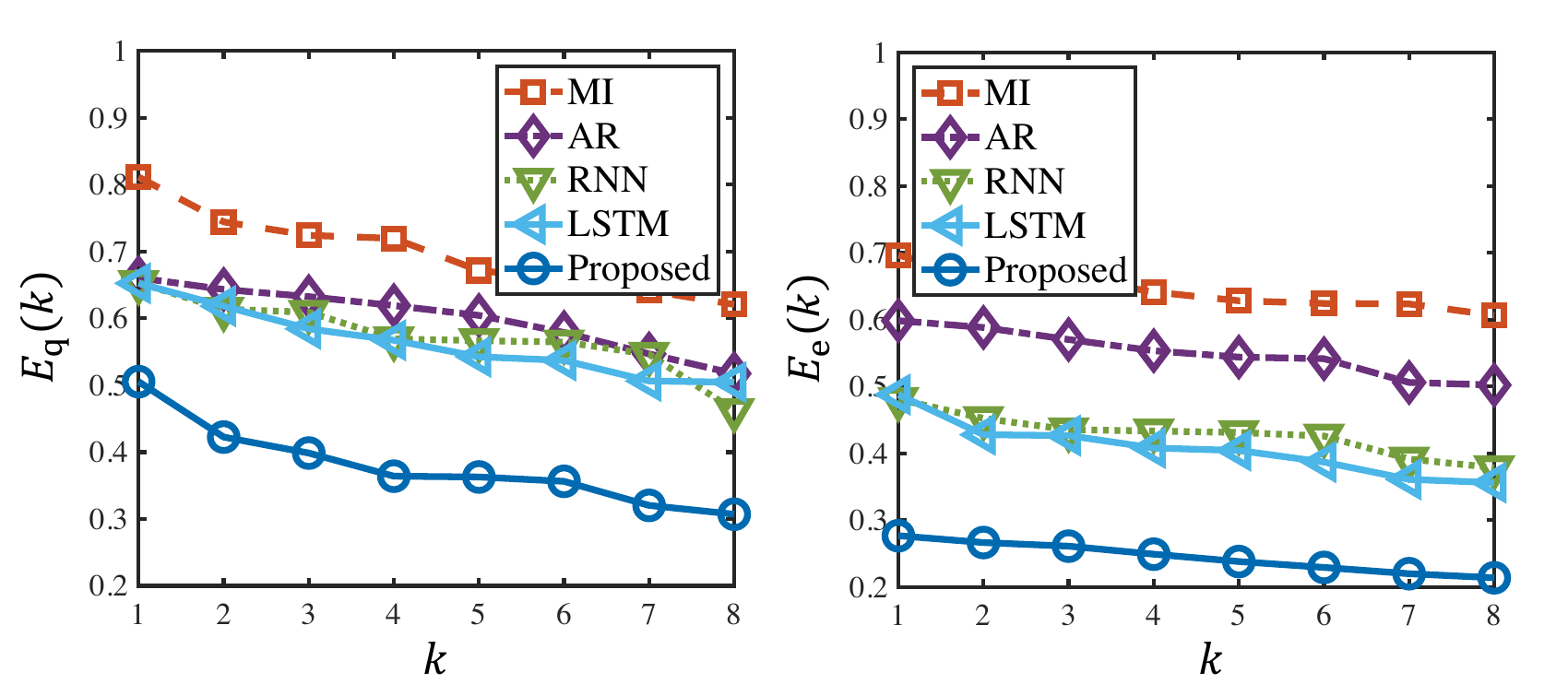}\vspace{-0.1in}
\par\end{centering}
\caption{Performance of \ac{csi} beam \ac{rsrp} prediction. \label{fig:Per_CSI_Prediction_RSRP}}
\vspace{-0.1in}
\end{figure}

\selectlanguage{english}%
We employ three metrics to evaluate the accuracy and reliability of
prediction models in estimating future \foreignlanguage{american}{\ac{csi}}
values based on current and historical measurements. Let $\{\hat{\mathbf{h}}_{n}\}_{n=1}^{N}$
represent the predicted \foreignlanguage{american}{\ac{csi}} beam
information, and $\{\mathbf{h}_{n}\}_{n=1}^{N}$ denote the real measurements.
The first metric, \emph{average quantity deviation error} on the $k$-strongest
beams, \foreignlanguage{american}{$E_{\mathrm{q}}(k)=\frac{1}{NQ}\sum_{n=1}^{N}\sum_{q=1}^{Q}\Big(1-|\hat{\mathcal{B}}_{q,n,k}\cap\mathcal{B}_{q,n,k}|/k\Big)$},
measures the average mismatch in top-$k$ beam selection across all
base stations and time indices, where $\hat{\mathcal{B}}_{q,n,k}$
denotes the set of indices corresponding to the $k$ strongest beams
in $\hat{\mathbf{h}}_{n}$ for BS $q$, and $\mathcal{B}_{q,n,k}$
denotes the same for $\mathbf{h}_{n}$. The quantity $|\hat{\mathcal{B}}_{q,n,k}\cap\mathcal{B}_{q,n,k}|$
represents the number of beams correctly predicted within the top-$k$
set for BS $q$. Lower values of $E_{\mathrm{q}}(k)$ indicate more
accurate beam predictions. The second metric, \emph{average relative
energy deviation} on the $k$-strongest beams, \foreignlanguage{american}{$E_{\mathrm{e}}(k)=\frac{1}{NQ}\sum_{n=1}^{N}\sum_{q=1}^{Q}|e_{q,n,k}-\hat{e}_{q,n,k}|/e_{q,n,k}$},
assesses the average relative difference in the total energy of the
predicted versus actual top-$k$ beams, where $\hat{e}_{q,n,k}$ and
$e_{q,n,k}$ denote the total energy of the $k$ strongest beams in
the prediction and actual measurements, respectively, for \foreignlanguage{american}{\ac{bs}}
$q$. Lower values of $E_{\mathrm{e}}(k)$ reflect higher accuracy
in the energy allocation of the predictions. The third metric is the
\foreignlanguage{american}{\ac{mae}} of the predicted maximum beam
\foreignlanguage{american}{\ac{csi}}, defined as $E_{\mathrm{a}}=\frac{1}{NQ}\sum_{n=1}^{N}\sum_{q=1}^{Q}|\nu_{n,q}-\hat{\nu}_{n,q}|$
where $\hat{\nu}_{n,q}$ and $\nu_{n,q}$ denote the predicted and
real \foreignlanguage{american}{\ac{csi}} measurements of the strongest
beam, respectively for \foreignlanguage{american}{\ac{bs}} $q$.

\selectlanguage{american}%
Figure \ref{fig:Per_CSI_Prediction_RSRP} presents the \ac{csi} beam
\ac{rsrp} prediction performance of the proposed radio map-assisted
method in comparison with the baseline methods. The proposed radio
map-assisted method consistently demonstrates superior performance
over the baselines. As the parameter $k$ increases, both $E_{\mathrm{q}}(k)$
and \foreignlanguage{english}{$E_{\mathrm{e}}(k)$} decrease, indicating
enhanced prediction accuracy. The proposed radio map-assisted method
achieves average energy deviation errors of 21\% and average quantity
deviation error of 31\% for predicting the eight strongest \ac{rsrp}
of 32 \ac{csi} beams. \foreignlanguage{english}{Table \ref{tab:Performance-fmaxRSRP}
presents the $E_{\mathrm{a}}$ performance for CSI beam maximum }\ac{rsrp}\foreignlanguage{english}{
prediction. The proposed method achieves a prediction error of less
than 4.44 dB.}

\begin{table}
\centering{}\caption{Performance for CSI beam maximum RSRP prediction.\label{tab:Performance-fmaxRSRP}}
\begin{tabular}{c|ccccc}
\hline 
MAE (dB) & MI & AR & CNN & LSTM & Proposed\tabularnewline
\hline 
\selectlanguage{english}%
$E_{\mathrm{a}}$\selectlanguage{american}%
 & 11.16 & 8.93 & 8.80 & 7.49 & 4.44\tabularnewline
\hline 
\end{tabular}
\end{table}

Table \ref{tab:Perfor-SINR-RSSI} displays the prediction performance
for \ac{csi} beam \ac{sinr} and \ac{rssi}. The proposed radio map-assisted
method always outperforms the comparison methods. Specifically, the
proposed radio map-assisted method achieves average energy deviation
errors of 28\%, 27\% and and average quantity deviation error of 34\%,
35\% for predicting the eight strongest \ac{rssi}\foreignlanguage{english}{
and }\ac{sinr} beams of \ac{csi}, respectively. For the strongest
$k=16$ beams prediction, the $E_{\mathrm{q}}$ of the proposed radio
map-assisted method is only 0.13 and 0.16 for \ac{sinr} and \ac{rssi},
respectively, which are reductions of 0.07 and 0.09 compared to the
LSTM method. Additionally, the $E_{\mathrm{e}}(16)$ of the proposed
method is 0.10 and 0.12 for \ac{sinr} and \ac{rssi}, respectively,
which are decreases of 0.11 and 0.10 relative to the LSTM method.

We investigate the effect of $L$ on the \ac{csi} prediction performance.
As shown in Figure \ref{fig:Per-Various-L}, as $L$ increases, the
proposed radio map-assisted method exhibits a decreasing trend in
$E_{\mathrm{e}}(1)$ and $E_{\mathrm{q}}(1)$. This is because incorporating
more historical \ac{rsrp} measurements enables more accurate estimation
of the current location, resulting in improved \ac{csi} prediction
accuracy. Furthermore, the proposed radio map-assisted method consistently
outperforms the comparisons.

\begin{figure}[t]
\begin{centering}
\includegraphics[width=1\columnwidth]{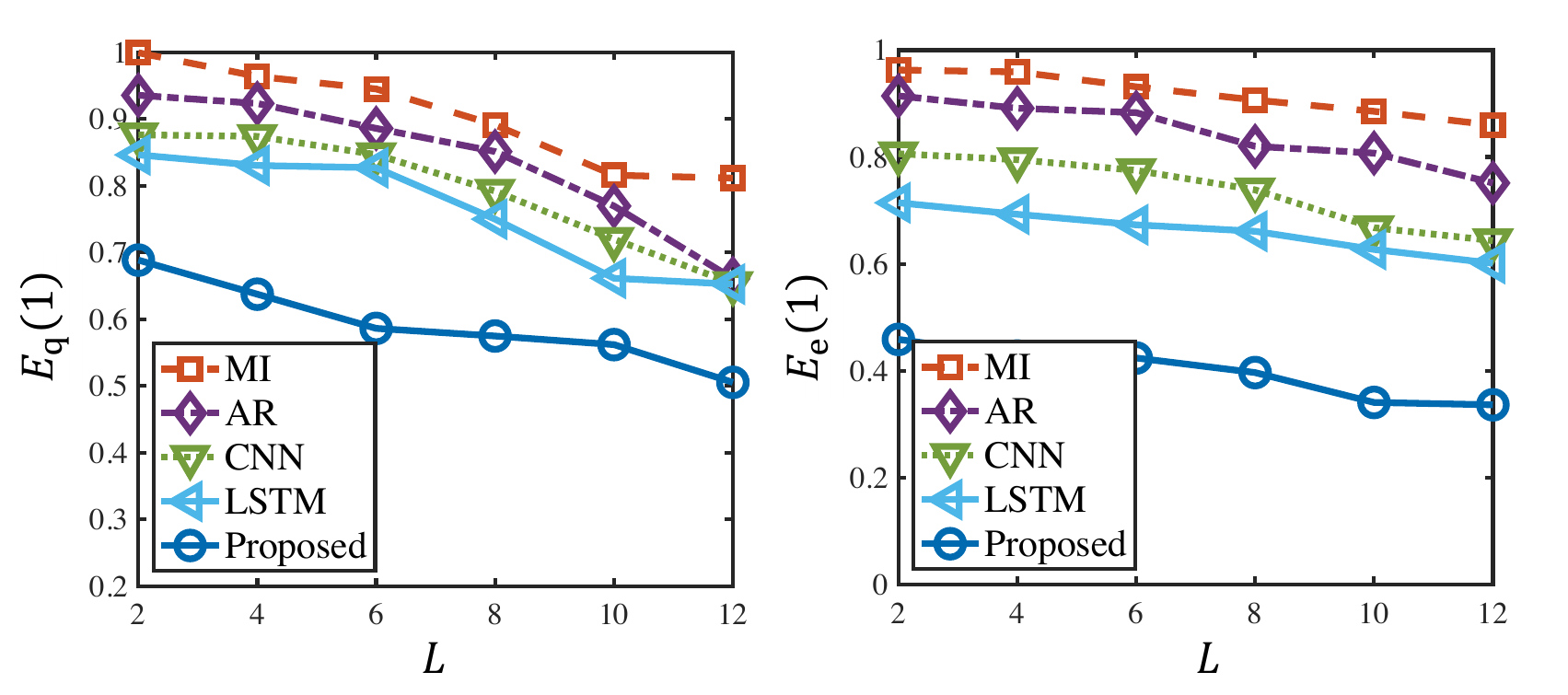}\vspace{-0.1in}
\par\end{centering}
\caption{Performance of \ac{csi} beam \ac{rsrp} prediction with different
$L$. \label{fig:Per-Various-L}}
\vspace{-0.1in}
\end{figure}

\begin{table}[t]
\begin{centering}
\caption{Performance (\foreignlanguage{english}{$E_{\mathrm{q}}(8)$}) of \ac{csi}
beam \ac{rsrp} prediction on different testing trajectories in the
real \ac{mimo} dataset. \label{tab:Perfo-differ-traj}}
\par\end{centering}
\centering{}%
\begin{tabular}{>{\raggedright}p{1.6cm}|>{\centering}p{1.1cm}>{\centering}p{1.1cm}>{\centering}p{1.1cm}>{\centering}p{1.1cm}}
\hline 
 & III-2 & III-3 & III-4 & III-5\tabularnewline
\hline 
\foreignlanguage{english}{MI} & 0.51 & 0.61 & 0.81 & 0.46\tabularnewline
AR & \centering{}0.42 & \centering{}0.51 & \centering{}0.68 & \centering{}0.36\tabularnewline
CNN & \centering{}0.38 & \centering{}0.45 & \centering{}0.60 & \centering{}0.32\tabularnewline
LSTM & \centering{}0.41 & \centering{}0.49 & \centering{}0.66 & \centering{}0.35\tabularnewline
Proposed & \centering{}0.18 & \centering{}0.35 & \centering{}0.44 & \centering{}0.16\tabularnewline
\hline 
\end{tabular}\vspace{-0.1in}
\end{table}

We investigate the performance of the proposed radio map-assisted
method on different testing trajectories as shown in Table \ref{tab:Perfo-differ-traj}.
First, we compare the performance on Trajectory III-5 with that on
Trajectories III-2 to III-4. The proposed radio map-assisted method,
along with the baseline methods, achieves the best performance on
Trajectory III-5. This superior performance is attributed to the fact
that Trajectory III-5 is included in the training set, resulting in
prediction errors that are closely aligned with the training error.
Second, we assess the performance on Trajectory III-2 in comparison
with Trajectories III-3 and III-4. The proposed radio map-assisted
method and the baseline methods demonstrate better performance on
Trajectory III-2. The only difference between Trajectory III-2 and
the training trajectory is the direction of movement, with Trajectory
III-2 moving in the opposite direction. Although reversing the direction
introduces some increase in prediction error, the increase remains
limited. Third, we examine the performance on Trajectories III-3 and
III-4 relative to Trajectories III-2 and III-5. We observe that doubling
the speed results in an increase in prediction error significantly.

\section{Conclusion}

\label{sec:Conclusion}

This paper presents a method for constructing an angular power map
without the need for location labels. We begin by proposing a mobility
model for the mobile user and modeling the signal propagation of each
beam. Subsequently, we introduce a novel \ac{hmm}-based \ac{rsrp}
embedding technique to recover the data collection trajectory of \ac{csi}
sequences in massive \ac{mimo} networks. As a result, an angular
power map is constructed without requiring calibration efforts. We
establish theoretical results demonstrating that under uniform rectilinear
mobility and \ac{ppp} \acpl{bs}, the \ac{crlb} of the localization
error can asymptotically approach zero at any \ac{snr}. Moreover,
if the \ac{bs} are deployed only within a limited region, the localization
error cannot approach zero even with an infinite amount of independent
measurement data. Experiments conducted in a real commercial 5G network
confirm the effectiveness of our method, achieving a mean localization
error below 18 meters based on sparse \ac{ssb} \ac{rsrp} measurements.

\appendices{}

\section{Proof of Theorem \ref{thm:LB-F-x}}

\label{sec:app-theo1}

From (\ref{eq:F-Txtil}), the FIM $\mathbf{F}_{T,x}$, as the upper
diagonal block in $\mathbf{F}_{T,\psi}$, can be expressed as
\begin{align}
\mathbf{F}_{T,x} & =\sum_{t=1}^{T}\sum_{q=1}^{Q}\frac{\alpha_{q}^{2}}{\sigma_{q}^{2}d_{t,q}^{4}(\mathbf{x},\mathbf{v})}(\bm{l}_{q}(\mathbf{x})+t\mathbf{v})(\bm{l}_{q}(\mathbf{x})+t\mathbf{v})^{\mathrm{T}}.\label{eq:F-Tv}
\end{align}

In the following text, we simplify the notation by writing $\bm{l}_{q}(\mathbf{x})$
as $\bm{l}_{q}$ and $d_{t,q}(\mathbf{x},\mathbf{v})$ as $d_{t,q}$.
Denoting $\alpha_{\max}^{2}=\max_{q}\{\alpha_{q}^{2}\}\geq\alpha_{q}^{2}$,
$\sigma_{\min}^{2}=\min_{q}\{\sigma_{q}^{2}\}\leq\sigma_{q}^{2}$,
we have
\begin{align}
\mathbf{F}_{T,x} & \preceq\frac{\alpha_{\max}^{2}}{\sigma_{\min}^{2}}\sum_{t=1}^{T}\sum_{q=1}^{Q}\frac{1}{d_{t,q}^{4}}(\bm{l}_{q}+t\mathbf{v})(\bm{l}_{q}+t\mathbf{v})^{\mathrm{T}}\label{eq:UB-F-Tx}\\
 & =C_{0}\bigg[\sum_{q=1}^{Q}s_{T,q}^{(0)}\bm{l}_{q}\bm{l}_{q}^{\mathrm{T}}+\sum_{q=1}^{Q}s_{T,q}^{(1)}(\bm{l}_{q}\mathbf{v}^{\mathrm{T}}+\mathbf{v}\bm{l}_{q}^{\mathrm{T}})\nonumber \\
 & \qquad+\sum_{q=1}^{Q}s_{T,q}^{(2)}\mathbf{v}\mathbf{v}^{\mathrm{T}}\bigg]=C_{0}\mathbf{A}_{T,x}\nonumber 
\end{align}
where $C_{0}=\frac{\alpha_{\max}^{2}}{\sigma_{\min}^{2}}$, $s_{T,q}^{(n)}=\sum_{t=1}^{T}\frac{t^{n}}{d_{t,q}^{4}}$,
and 
\begin{equation}
\mathbf{A}_{T,x}=\sum_{q=1}^{Q}s_{T,q}^{(0)}\bm{l}_{q}\bm{l}_{q}^{\mathrm{T}}+\sum_{q=1}^{Q}s_{T,q}^{(1)}(\bm{l}_{q}\mathbf{v}^{\mathrm{T}}+\mathbf{v}\bm{l}_{q}^{\mathrm{T}})+\sum_{q=1}^{Q}s_{T,q}^{(2)}\mathbf{v}\mathbf{v}^{\mathrm{T}}\label{eq:A-Tx}
\end{equation}
and equality (\ref{eq:UB-F-Tx}) can be achieved when $\sigma_{\mathrm{\min}}^{2}=\sigma_{q}^{2}$
and $\alpha_{\mathrm{max}}^{2}=\alpha_{q}^{2}$.

It is observed that (\ref{eq:UB-F-Tx}) $\mathbf{A}_{T,x}\prec\mathbf{A}_{T+1,x}$
because each component $\sum_{q}^{Q}1/d_{t,q}^{4}(\bm{l}_{q}+t\mathbf{v})(\bm{l}_{q}+t\mathbf{v})^{\mathrm{T}}$
in (\ref{eq:UB-F-Tx}) is a positive definite matrix since at least
two vectors in $\{\bm{l}_{1},\bm{l}_{2},\dots,\bm{l}_{Q},\mathbf{v}\}$
are linear independent. Therefore, $\mathrm{tr}\{\mathbf{F}_{T,x}^{-1}\}\geq\bar{\Delta}_{T,x}\triangleq\mathrm{tr}\{(C_{0}\mathbf{A}_{T,x})^{-1}\}$
is strictly decreasing in $T$.
\begin{lem}
\label{lem:A-Tx-lam}Suppose $d_{\min,q}>0$. The sequence $s_{T,q}^{(n)}$
is bounded for $n<3$ and divergent as $s_{T,q}^{(n)}\rightarrow\infty$
as $T\rightarrow\infty$ for $n\geq3$. In addition, $s_{T,q}^{(n+1)}/s_{T,q}^{(n)}\rightarrow\infty$
as $T\rightarrow\infty$ for $n>3$.

\begin{proof}Recall $d_{t,q}=\|\bm{l}_{q}+t\mathbf{v}\|_{2}$ which
is asymptotically a linear function in $t$ for large $t$. Then,
there exits a sufficiently small $\rho_{1}>0$, such that $d_{t,q}>\rho_{1}t$
for all $t\geq1$. As a result,
\[
0<s_{T,q}^{(n)}=\sum_{t=1}^{T}\frac{t^{n}}{d_{t,q}^{4}}\leq\sum_{t=1}^{T}\frac{t^{n}}{(\rho_{1}t)^{4}}=\frac{1}{\rho_{1}^{4}}\sum_{t=1}^{T}\frac{1}{t^{4-n}}<\infty
\]
for $n<3$. This is because $\sum_{t=1}^{T}\frac{1}{t^{r}}$ is convergent
for all $r>1$.

In addition, there exists a large enough $\rho_{2}<\infty$, such
that $d_{t,q}<\rho_{2}t$ for all $t\geq1$. As a result,
\[
s_{T,q}^{(n)}=\sum_{t=1}^{T}\frac{t^{n}}{d_{t,q}^{4}}\geq\sum_{t=1}^{T}\frac{t^{n}}{(\rho_{2}t)^{4}}=\frac{1}{\rho_{2}^{4}}\sum_{t=1}^{T}\frac{1}{t^{4-n}}\rightarrow\infty
\]
as $T\rightarrow\infty$ if $4-n\leq1$, i.e., $n\geq3$.

Moreover, for $n\geq3$, we have
\[
\frac{s_{T,q}^{(n+1)}}{s_{T,q}^{(n)}}\geq\frac{\sum_{t=1}^{T}\frac{t^{n}}{(\rho_{2}t)^{4}}}{\sum_{t=1}^{T}\frac{t^{n}}{(\rho_{1}t)^{4}}}=\frac{\frac{1}{\rho_{2}^{4}}\sum_{t=1}^{T}\frac{1}{t^{4-n-1}}}{\frac{1}{\rho_{1}^{4}}\sum_{t=1}^{T}\frac{1}{t^{4-n}}}\rightarrow\infty
\]
as $T\rightarrow\infty$.\end{proof}
\end{lem}
Using Lemma \ref{lem:A-Tx-lam}, since $s_{T,q}^{(n)}$ are bounded
for $n<3$ and $Q$ is finite, we have $\mathbf{A}_{T,x}$ bounded.
Thus, $\bar{\Delta}_{T,x}$ converges to a strictly positive number
as $T\rightarrow\infty$.

\section{Proof of Theorem \ref{thm:LB-F-v}}

\label{sec:app-theo2}

From (\ref{eq:F-Txtil}), the FIM $\mathbf{F}_{T,v}$, as the lower
diagonal block in $\mathbf{F}_{T,\psi}$ can be expressed as
\begin{align}
\mathbf{F}_{T,v} & =\sum_{t=1}^{T}\sum_{q=1}^{Q}\frac{\alpha_{q}^{2}t^{2}}{\sigma_{q}^{2}d_{t,q}^{4}}(\bm{l}_{q}+t\mathbf{v})(\bm{l}_{q}+t\mathbf{v})^{\mathrm{T}}.\label{eq:F-Tv-1}
\end{align}

Similar to (\ref{eq:UB-F-Tx}), we have,
\begin{align}
\mathbf{F}_{T,v} & \preceq\sum_{t,q}\frac{\alpha_{\mathrm{max}}^{2}}{\sigma_{\mathrm{min}}^{2}}\frac{t^{2}}{d_{t,q}^{4}}(\bm{l}_{q}+t\mathbf{v})(\bm{l}_{q}+t\mathbf{v})^{\mathrm{T}}\label{eq:UB-F-Tv}\\
 & =C_{0}\cdot\Big[\sum_{q=1}^{Q}s_{T,q}^{(2)}\bm{l}_{q}\bm{l}_{q}^{\mathrm{T}}+\sum_{q=1}^{Q}s_{T,q}^{(3)}(\bm{l}_{q}\mathbf{v}^{\mathrm{T}}+\mathbf{v}\bm{l}_{q}^{\mathrm{T}})\nonumber \\
 & \qquad+\sum_{q=1}^{Q}s_{T,q}^{(4)}\mathbf{v}\mathbf{v}^{\mathrm{T}}\Big]=C_{0}\mathbf{A}_{T,v}\nonumber 
\end{align}
where
\[
\mathbf{A}_{T,v}=\sum_{q=1}^{Q}s_{T,q}^{(2)}\bm{l}_{q}\bm{l}_{q}^{\mathrm{T}}+\sum_{q=1}^{Q}s_{T,q}^{(3)}(\bm{l}_{q}\mathbf{v}^{\mathrm{T}}+\mathbf{v}\bm{l}_{q}^{\mathrm{T}})+\sum_{q=1}^{Q}s_{T,q}^{(4)}\mathbf{v}\mathbf{v}^{\mathrm{T}}
\]
and the equality is achieved when $\alpha_{\mathrm{max}}^{2}=\alpha_{q}^{2}$,
$\sigma_{\mathrm{min}}^{2}=\sigma_{q}^{2}$.
\begin{lem}
\label{lem:A-Tv-lam}The eigenvalues of $\mathbf{A}_{T,v}$ satisfies
\[
\lambda_{\mathrm{min}}(\mathbf{A}_{T,v})\rightarrow\sum_{q=1}^{Q}s_{T,q}^{(2)}\mathbf{v}_{\perp}^{\mathrm{T}}\bm{l}_{q}\bm{l}_{q}^{\mathrm{T}}\mathbf{v}_{\perp}
\]
and $\lambda_{\mathrm{max}}(\mathbf{A}_{T,v})\rightarrow\sum_{q=1}^{Q}s_{T,q}^{(4)}\|\mathbf{v}\|^{2}$
as $T\rightarrow\infty$.
\end{lem}
\begin{proof}Using Lemma \ref{lem:A-Tx-lam}, $s_{T,q}^{(2)}$ is
bounded, and $s_{T,q}^{(3)}$ and $s_{T,q}^{(4)}$ are divergent,
where $s_{T,q}^{(4)}$ dominates $s_{T,q}^{(3)}$ for asymptotically
large $T$. Therefore, the term $\sum_{q=1}^{Q}s_{t,q}^{(4)}\mathbf{v}\mathbf{v}^{T}$
dominates $\mathbf{A}_{T,v}$ for a sufficiently large $T$. Thus,
for a sufficiently large $T$, the larger eigenvalue satisfies
\begin{align}
\lambda_{\mathrm{max}}(\mathbf{A}_{T,v}) & =\underset{\|\mathbf{u}\|=1}{\mathrm{max}}\:\mathbf{u}^{\mathrm{T}}\mathbf{A}_{T,v}\mathbf{u}\nonumber \\
 & =\underset{\|\mathbf{u}\|=1}{\mathrm{max}}\:\mathbf{u}^{\mathrm{T}}\Big[\sum_{q=1}^{Q}s_{T,q}^{(2)}\bm{l}_{q}\bm{l}_{q}^{\mathrm{T}}\nonumber \\
 & \qquad+\sum_{q=1}^{Q}s_{T,q}^{(3)}(\bm{l}_{q}\mathbf{v}^{\mathrm{T}}+\mathbf{v}\bm{l}_{q}^{\mathrm{T}})+\sum_{q=1}^{Q}s_{T,q}^{(4)}\mathbf{v}\mathbf{v}^{\mathrm{T}}\Big]\mathbf{u}\nonumber \\
 & \approx\underset{\|\mathbf{u}\|=1}{\mathrm{max}}\:\sum_{q=1}^{Q}s_{T,q}^{(4)}\cdot\mathbf{u}^{\mathrm{T}}(\mathbf{v}\mathbf{v}^{\mathrm{T}})\mathbf{u}\label{eq:prob-uv}
\end{align}
where the solution to (\ref{eq:prob-uv}) is $\mathbf{u}=\mathbf{v}/\|\mathbf{v}\|$,
and $\lambda_{\mathrm{max}}(\mathbf{A}_{T,v})\rightarrow\sum_{q=1}^{Q}s_{T,q}^{(4)}\|\mathbf{v}\|^{2}$.
As a result, asymptotically, the larger eigenvector is $\mathbf{v}/\|\mathbf{v}\|\in\mathbb{R}^{2}$,
and hence, the smaller eigenvector is denoted as $\mathbf{v}_{\perp}$,
which satisfies $\mathbf{v}_{\perp}^{\mathrm{T}}\mathbf{v}=0$.

Consequentially, we have
\begin{align*}
\lambda_{\mathrm{min}}(\mathbf{A}_{T,v}) & =\mathbf{v}_{\perp}^{\mathrm{T}}\Big[\sum_{q=1}^{Q}s_{T,q}^{(2)}\bm{l}_{q}\bm{l}_{q}^{\mathrm{T}}+\sum_{q=1}^{Q}s_{T,q}^{(3)}(\bm{l}_{q}\mathbf{v}^{\mathrm{T}}+\mathbf{v}\bm{l}_{q}^{\mathrm{T}})\\
 & \qquad\qquad+\sum_{q=1}^{Q}s_{T,q}^{(4)}\mathbf{v}\mathbf{v}^{\mathrm{T}}\Big]\mathbf{v}_{\perp}\\
 & =\mathbf{v}_{\perp}^{\mathrm{T}}\sum_{q=1}^{Q}s_{T,q}^{(2)}\bm{l}_{q}\bm{l}_{q}^{\mathrm{T}}\mathbf{v}_{\perp}+\mathbf{v}_{\perp}^{\mathrm{T}}\Big(\sum_{q=1}^{Q}s_{T,q}^{(3)}(\bm{l}_{q}\mathbf{v}^{\mathrm{T}}\\
 & \qquad\qquad+\mathbf{v}\bm{l}_{q}^{\mathrm{T}})\Big)\mathbf{v}_{\perp}+\mathbf{v}_{\perp}^{\mathrm{T}}\sum_{q=1}^{Q}s_{T,q}^{(4)}\mathbf{v}\mathbf{v}^{\mathrm{T}}\mathbf{v}_{\perp}\\
 & =\sum_{q=1}^{Q}s_{T,q}^{(2)}\mathbf{v}_{\perp}^{\mathrm{T}}\bm{l}_{q}\bm{l}_{q}^{\mathrm{T}}\mathbf{v}_{\perp}
\end{align*}
\end{proof}

From $\mathbf{F}_{T,v}\preceq C_{0}\mathbf{A}_{T,v}$, since both
$\mathbf{F}_{T,\mathbf{v}}$ and $\mathbf{A}_{T,\mathbf{v}}$ are
P.S.D., we have
\begin{equation}
\lambda_{\text{min}}(\mathbf{F}_{T,v})\leq C_{0}\lambda_{\text{min}}(\mathbf{A}_{T,v}),\lambda_{\text{max}}(\mathbf{F}_{T,v})\leq C_{0}\lambda_{\text{max}}(\mathbf{A}_{T,v}).\label{eq:lam-min-ineq}
\end{equation}

Denoting the \ac{evd} of $\mathbf{F}_{T,\mathbf{v}}$ as $\mathbf{F}_{T,\mathbf{v}}=\mathbf{u}_{T,v}\bm{\varLambda}_{T,v}\mathbf{u}_{T,v}^{-1}$,
where
\[
\bm{\varLambda}_{T,v}=\left[\begin{array}{cc}
\lambda_{\mathrm{max}}(\mathbf{A}_{T,\mathbf{v}}) & 0\\
0 & \lambda_{\mathrm{min}}(\mathbf{A}_{T,\mathbf{v}})
\end{array}\right],
\]
 we have
\begin{align}
\mathrm{tr}\{\mathbf{F}_{T,v}^{-1}\} & =\mathrm{tr}\{(\mathbf{u}_{T,v}\bm{\varLambda}_{T,v}\mathbf{u}_{T,v}^{-1})^{-1}\}=\mathrm{tr}\{\mathbf{u}_{T,v}\bm{\varLambda}_{T,v}^{-1}\mathbf{u}_{T,v}^{-1}\}\nonumber \\
 & =\lambda_{\mathrm{max}}^{-1}(\mathbf{F}_{T,v})+\lambda_{\mathrm{min}}^{-1}(\mathbf{F}_{T,v})\nonumber \\
 & \geq\frac{1}{C_{0}}\lambda_{\mathrm{max}}^{-1}(\mathbf{A}_{T,v})+\frac{1}{C_{0}}\lambda_{\mathrm{min}}^{-1}(\mathbf{A}_{T,v})\label{eq:tfF-ineq1}\\
 & \geq\frac{1}{C_{0}}\lambda_{\mathrm{min}}^{-1}(\mathbf{A}_{T,v})\triangleq\bar{\Delta}_{T,v}\label{eq:tfF-ineq2}
\end{align}
where (\ref{eq:tfF-ineq1}) is due to (\ref{eq:lam-min-ineq}) with
equality achieved when $\alpha_{\mathrm{max}}^{2}=\alpha_{q}^{2}$,
$\sigma_{\mathrm{min}}^{2}=\sigma_{q}^{2}$ and (\ref{eq:tfF-ineq2})
is due to the fact that $C_{0}\lambda_{\mathrm{max}}(\mathbf{A}_{T,\mathbf{v}})>0$
and equality can be asymptotically achieved at large $T$ as $\lambda_{\mathrm{\max}}^{-1}(\mathbf{A}_{T,v})\rightarrow1/(\sum_{q=1}^{Q}s_{T,q}^{(4)}\|\mathbf{v}\|^{2})$
which converges to zero.

Using Lemma \ref{lem:A-Tv-lam}, as $T\rightarrow\infty$, we have
\[
\bar{\Delta}_{T,v}\rightarrow C_{v}=\Bigg(C_{0}\sum_{q=1}^{Q}s_{T,q}^{(2)}\mathbf{v}_{\perp}^{\mathrm{T}}\bm{l}_{q}\bm{l}_{q}^{\mathrm{T}}\mathbf{v}_{\perp}\Bigg)^{-1},
\]
which is strictly positive. Since the orthogonal projection $\mathbf{v}_{\perp}^{\mathrm{T}}\bm{l}_{q}=\bm{l}_{q}-(\bm{l}_{q}^{\mathrm{T}}\mathbf{v}/\|\mathbf{v}\|^{2})\mathbf{v}=(\mathbf{I}-\mathbf{v}\mathbf{v}^{\mathrm{T}}/\|\mathbf{v}\|^{2})\bm{l}_{q}$,
we have
\begin{align*}
 & \sum_{q=1}^{Q}s_{T,q}^{(2)}\mathbf{v}_{\perp}^{\mathrm{T}}\bm{l}_{q}\bm{l}_{q}^{\mathrm{T}}\mathbf{v}_{\perp}=\sum_{q=1}^{Q}s_{T,q}^{(2)}\|\bm{l}_{q}^{\mathrm{T}}\mathbf{v}_{\bot}\|^{2}\\
 & =\sum_{q=1}^{Q}s_{T,q}^{(2)}\|(\mathbf{I}-\mathbf{v}\mathbf{v}^{\mathrm{T}}/\|\mathbf{v}\|^{2})\bm{l}_{q}\|^{2}=\sum_{q=1}^{Q}s_{T,q}^{(2)}\|\mathbf{P}_{v}^{\bot}\bm{l}_{q}\|^{2}
\end{align*}
where $\mathbf{P}_{v}^{\bot}=\mathbf{I}-\mathbf{vv^{T}/\|v}\|^{2}$
is orthogonal projector, and $\mathbf{P}_{v}^{\bot}\bm{l}_{q}$ is
to project the vector $\bm{l}_{q}$ onto the null space spanned by
$\mathbf{v}_{\bot}$ of $\mathbf{v}$.

Thus, we have $C_{v}=\left(C_{0}\sum_{q=1}^{Q}s_{T,q}^{(2)}\|\mathbf{P}_{v}^{\bot}\bm{l}_{q}\|^{2}\right)^{-1}$,
where $s_{T,q}^{(2)}$ is bounded as stated in Lemma \ref{lem:A-Tx-lam}.
Recall $\rho>0$ is sufficiently small such that $d_{t,q}(\mathbf{x},\mathbf{v})>\rho t$
for all $t\geq1$, we have
\begin{align*}
s_{\infty,q}^{(2)} & =\lim_{T\to\infty}\sum_{t=1}^{T}\frac{t^{2}}{d_{t,q}^{4}}<\lim_{T\to\infty}\sum_{t=1}^{T}\frac{t^{2}}{(\rho t)^{4}}\\
 & =\frac{1}{\rho^{4}}\lim_{T\to\infty}\sum_{t=1}^{T}\frac{1}{t^{2}}\approx\frac{\pi^{2}}{6\rho^{4}}.
\end{align*}
Thus, the element $s_{\infty,q}^{(2)}$ is upper bounded by $\frac{\pi^{2}}{6\rho^{4}}$.

\section{Proof of Theorem \ref{thm:LB-F-x-un}}

\label{sec:app-theo3}

Denote $\mathcal{Q}_{t}=\{q|d_{t,q}\leq R\}$ as the set of \acpl{bs}
that are within a range of $R$ from the mobile user at time slot
$t$. Based on the FIM $\mathbf{F}_{T,\psi}$ in (\ref{eq:F-Tv}),
we have
\begin{align}
 & \mathbf{F}_{T,x}\nonumber \\
 & \succeq\frac{\alpha_{\min}^{2}}{\sigma_{\max}^{2}}\sum_{t=1}^{T}\mathbb{E}\bigg\{\sum_{q\in\mathcal{Q}_{t}}\frac{1}{d_{t,q}^{4}}(\bm{l}_{q}+t\mathbf{v})(\bm{l}_{q}+t\mathbf{v})^{\mathrm{T}}\Bigg\}\label{eq:ineq-F-tx}\\
 & =\tilde{C}_{0}\sum_{t=1}^{T}\mathbb{E}\bigg\{\sum_{q\in\mathcal{Q}_{t}}\left[\frac{\bm{l}_{q}\bm{l}_{q}^{\mathrm{T}}}{d_{t,q}^{4}}+\frac{t(\bm{l}_{q}\mathbf{v}^{\mathrm{T}}+\mathbf{v}\bm{l}_{q}^{\mathrm{T}})}{d_{t,q}^{4}}+\frac{t^{2}\mathbf{v}\mathbf{v}^{\mathrm{T}}}{d_{t,q}^{4}}\right]\Bigg\}\nonumber \\
 & =\tilde{C}_{0}\tilde{\mathbf{A}}_{T,x}\nonumber 
\end{align}
where $\tilde{C}_{0}=\frac{\alpha_{\min}^{2}}{\sigma_{\max}^{2}}$
and
\begin{align}
\tilde{\mathbf{A}}_{T,x} & =\sum_{t=1}^{T}\mathbb{E}\Bigg\{\sum_{q\in\mathcal{Q}_{t}}\left[\frac{\bm{l}_{q}\bm{l}_{q}^{\mathrm{T}}}{d_{t,q}^{4}}+\frac{t(\bm{l}_{q}\mathbf{v}^{\mathrm{T}}+\mathbf{v}\bm{l}_{q}^{\mathrm{T}})}{d_{t,q}^{4}}+\frac{t^{2}\mathbf{v}\mathbf{v}^{\mathrm{T}}}{d_{t,q}^{4}}\right]\Bigg\}\nonumber \\
 & =\sum_{t=1}^{T}\mathbb{E}\Bigg\{\sum_{q\in\mathcal{Q}_{t}}\frac{\bm{l}_{q}\bm{l}_{q}^{\mathrm{T}}}{d_{t,q}^{4}}\Bigg\}+\sum_{t=1}^{T}t\mathbb{E}\Bigg\{\sum_{q\in\mathcal{Q}_{t}}\frac{(\bm{l}_{q}\mathbf{v}^{\mathrm{T}}+\mathbf{v}\bm{l}_{q}^{\mathrm{T}})}{d_{t,q}^{4}}\Bigg\}\nonumber \\
 & \quad\quad+\sum_{t=1}^{T}t^{2}\mathbb{E}\Bigg\{\sum_{q\in\mathcal{Q}_{t}}\frac{\mathbf{v}\mathbf{v}^{\mathrm{T}}}{d_{t,q}^{4}}\Bigg\}.\label{eq:A-tilde-T-x}
\end{align}
The equality in (\ref{eq:ineq-F-tx}) can be achieved when $\sigma_{\mathrm{\max}}^{2}=\sigma_{q}^{2}$
and $\alpha_{\min}^{2}=\alpha_{q}^{2}$.

Since $\mathbf{F}_{T,x}$ and $\tilde{\mathbf{A}}_{T,x}$ are $2\times2$
symmetric and positive semi-definite, their eigenvalues are real and
non-negative. From $\mathbf{F}_{T,x}\succeq\tilde{C}_{0}\tilde{\mathbf{A}}_{T,x}$,
we have
\[
\lambda_{\text{min}}(\mathbf{F}_{T,x})\geq\tilde{C}_{0}\lambda_{\text{min}}(\tilde{\mathbf{A}}_{T,x}),\lambda_{\text{max}}(\mathbf{F}_{T,x})\geq\tilde{C}_{0}\lambda_{\text{max}}(\tilde{\mathbf{A}}_{T,x}).
\]

Since $\mathrm{tr}\{\mathbf{F}_{T,x}^{-1}\}=\lambda_{\mathrm{max}}^{-1}(\mathbf{F}_{T,x})+\lambda_{\mathrm{min}}^{-1}(\mathbf{F}_{T,x})$,
we have
\begin{equation}
\mathrm{tr}\{\mathbf{F}_{T,x}^{-1}\}\leq2\lambda_{\mathrm{min}}^{-1}(\mathbf{F}_{T,x})\leq2\left(\tilde{C}_{0}\lambda_{\mathrm{min}}(\tilde{\mathbf{A}}_{T,x})\right)^{-1}\triangleq\tilde{\Delta}_{T,x}.\label{eq:tr-F-ineq}
\end{equation}

\begin{lem}
\label{lem:A-Tx-lam-1}Assume that $d_{t,q}\geq r_{0}$ \textup{for
all $t$ and $q$.}. The eigenvalue of $\tilde{\mathbf{A}}_{T,x}$
satisfies
\[
\frac{1}{T}\lambda_{\mathrm{min}}(\tilde{\mathbf{A}}_{T,x})\rightarrow\kappa\pi\ln(R/r_{0})
\]
as $T\to\infty$\foreignlanguage{english}{.}
\end{lem}
\begin{proof}The term $\sum_{t=1}^{T}t^{2}\mathbb{E}\{\sum_{q\in\mathcal{Q}_{t}}\mathbf{v}\mathbf{v}^{\mathrm{T}}/d_{t,q}^{4}\}$
in (\ref{eq:A-tilde-T-x}) dominates $\tilde{\mathbf{A}}_{T,x}$ for
a sufficiently large $T$, because $t^{2}$ increases quadratically.
Thus, as $T\rightarrow\infty$, the larger eigenvalue satisfies: 
\begin{align}
\frac{1}{T}\lambda_{\mathrm{max}}(\tilde{\mathbf{A}}_{T,x}) & =\frac{1}{T}\underset{\|\mathbf{u}\|=1}{\mathrm{max}}\:\mathbf{u}^{\mathrm{T}}\tilde{\mathbf{A}}_{T,x}\mathbf{u}\nonumber \\
 & \rightarrow\underset{\|\mathbf{u}\|=1}{\mathrm{max}}\:\frac{1}{T}\sum_{t=1}^{T}t^{2}\mathbb{E}\Bigg\{\sum_{q\in\mathcal{Q}_{t}}\frac{1}{d_{t,q}^{4}}\mathbf{u}^{\mathrm{T}}(\mathbf{v}\mathbf{v}^{\mathrm{T}})\mathbf{u}\Bigg\}\label{eq:prob-uv-1}
\end{align}
where the solution to (\ref{eq:prob-uv-1}) is $\mathbf{u}=\mathbf{v}/\|\mathbf{v}\|_{2}$.

As a result, asymptotically, the larger eigenvector is $\mathbf{v}/\|\mathbf{v}\|_{2}\in\mathbb{R}^{2}$,
and hence, the smaller eigenvector is $\mathbf{v}_{\perp}$, which
satisfies $\mathbf{v}^{\mathrm{T}}\mathbf{v}_{\perp}=0$. Consequently,
from (\ref{eq:A-tilde-T-x}), as $T\to\infty$, we have: 
\begin{align}
 & \frac{1}{T}\lambda_{\mathrm{min}}(\tilde{\mathbf{A}}_{T,x})\label{eq:lambda-min-A-tilde}\\
 & \rightarrow\mathbf{v}_{\perp}^{\mathrm{T}}\Big[\frac{1}{T}\sum_{t=1}^{T}\mathbb{E}\Big\{\sum_{q\in\mathcal{Q}_{t}}\frac{1}{d_{t,q}^{4}}\bm{l}_{q}\bm{l}_{q}^{\mathrm{T}}\Big\}+\frac{1}{T}\sum_{t=1}^{T}t\mathbb{E}\Big\{\sum_{q\in\mathcal{Q}_{t}}\frac{1}{d_{t,q}^{4}}\nonumber \\
 & \qquad\times(\bm{l}_{q}\mathbf{v}^{\mathrm{T}}+\mathbf{v}\bm{l}_{q}^{\mathrm{T}})\Big\}+\frac{1}{T}\sum_{t=1}^{T}t^{2}\mathbb{E}\Big\{\sum_{q\in\mathcal{Q}_{t}}\frac{1}{d_{t,q}^{4}}\mathbf{v}\mathbf{v}^{\mathrm{T}}\Big\}\Big]\mathbf{v}_{\perp}\nonumber \\
 & =\frac{1}{T}\sum_{t=1}^{T}\mathbb{E}\left\{ \sum_{q\in\mathcal{Q}_{t}}\frac{1}{d_{t,q}^{4}}\|\mathbf{P}_{v}^{\bot}\bm{l}_{q}\|^{2}\right\} .\label{eq:E-stq0}
\end{align}

\begin{figure}[t]
\begin{centering}
\includegraphics[width=0.6\columnwidth]{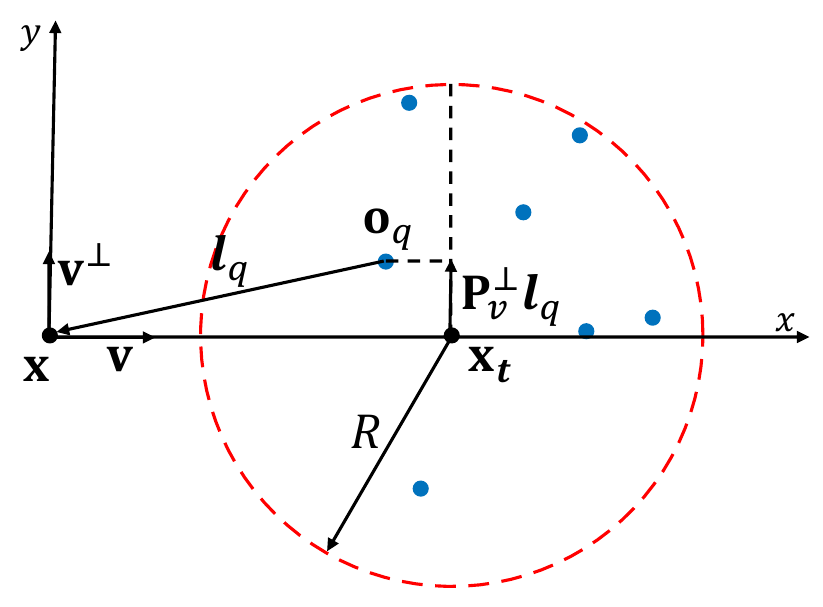}\vspace{-0.1in}
\par\end{centering}
\centering{}\caption{Illustration of the \acpl{bs} follow \ac{ppp} within a radius of
$R$ from the user location $\mathbf{x}_{t}$. \label{fig:Possion}}
\vspace{-0.1in}
\end{figure}

To compute the expectation in (\ref{eq:E-stq0}), we note that as
the \acpl{bs} follow a Poisson distribution within a radius of $R$
from the user location $\mathbf{x}_{t}$, the expected number of the
\acpl{bs} is $\kappa\pi R^{2}$. In addition, given the number of
the \acpl{bs}, the \acpl{bs} are independently and uniformly distributed.
As a result, consider a coordinate system with the initial position
$\mathbf{x}$ as the origin and the direction $\mathbf{v}$ as the
$x$-axis as shown in Figure \ref{fig:Possion}, and then, $\mathbf{P}_{v}^{\bot}\bm{l}_{q}$
is simply to project the vector $\bm{l}_{q}=\mathbf{x}-\mathbf{o}_{q}$
onto the $y$-axis. Denote $\bm{l}_{q}=(l_{q,x},l_{q,y})$ and it
follows that $\mathbf{P}_{v}^{\bot}\bm{l}_{q}=l_{q,y}$. We have
\begin{align}
 & \mathbb{E}\left\{ \sum_{q\in\mathcal{Q}_{t}}\frac{\|\mathbf{P}_{v}^{\bot}\bm{l}_{q}\|^{2}}{d_{t,q}^{4}}\right\} =\mathbb{E}\left\{ \frac{l_{q,y}^{2}}{(l_{q,x}^{2}+l_{q,y}^{2})^{2}}\right\} \kappa\pi R^{2}\nonumber \\
 & =\kappa\pi R^{2}\frac{1}{\pi R^{2}}\int_{-R}^{R}\int_{-\sqrt{R^{2}-x^{2}}}^{\sqrt{R^{2}-x^{2}}}\frac{x^{2}}{(x^{2}+y^{2})^{2}}\,dy\,dx\label{eq:remove-T}\\
 & =\kappa\pi\ln(R/r_{0}).\label{eq:E-pl}
\end{align}
where (\ref{eq:remove-T}) is due to the prior condition that $d_{t,q}>r_{0}$
for all $t$ and $q$.

As a result, from (\ref{eq:E-stq0}) and (\ref{eq:E-pl}), we have
$\frac{1}{T}\lambda_{\mathrm{min}}(\tilde{\mathbf{A}}_{T,x})\rightarrow\kappa\pi\ln(R/r_{0})$
as $T\to\infty$.\end{proof}

According Lemma \ref{lem:A-Tx-lam-1} and from (\ref{eq:tr-F-ineq}),
we have
\begin{align*}
T\tilde{\Delta}_{T,x} & \rightarrow\frac{2\sigma_{\max}^{2}}{\alpha_{\min}^{2}\kappa\pi\ln(R/r_{0})}
\end{align*}
as $T\to\infty$.

\section{Proof of Theorem \ref{thm:LB-F-v-un}}

\label{sec:app-theo4}

Denote $\mathcal{Q}_{t}=\{q|d_{t,q}\leq R\}$ as the set of \acpl{bs}
that are within a range of $R$ from the mobile user at time slot
$t$. Based on the FIM $\mathbf{F}_{T,\psi}$ in (\ref{eq:F-Tv}),
we have
\begin{align}
\mathbf{F}_{T,v} & \succeq\frac{\alpha_{\min}^{2}}{\sigma_{\max}^{2}}\sum_{t=1}^{T}\mathbb{E}\left\{ \sum_{q\in\mathcal{Q}_{t}}\frac{t^{2}}{d_{t,q}^{4}}(\bm{l}_{q}+t\mathbf{v})(\bm{l}_{q}+t\mathbf{v})^{\mathrm{T}}\right\} \label{eq:ineq-Ftv}\\
 & =\tilde{C}_{0}\sum_{t=1}^{T}\mathbb{E}\Bigg\{\sum_{q\in\mathcal{Q}_{t}}\frac{t^{2}}{d_{t,q}^{4}}\bm{l}_{q}\bm{l}_{q}^{\mathrm{T}}+\sum_{q\in\mathcal{Q}_{t}}\frac{t^{3}}{d_{t,q}^{4}}(\bm{l}_{q}\mathbf{v}^{\mathrm{T}}+\mathbf{v}\bm{l}_{q}^{\mathrm{T}})\nonumber \\
 & \qquad\qquad+\sum_{q\in\mathcal{Q}_{t}}\frac{t^{4}}{d_{t,q}^{4}}\mathbf{v}\mathbf{v}^{\mathrm{T}}\Bigg\}=\tilde{C}_{0}\tilde{\mathbf{A}}_{T,v}\nonumber 
\end{align}
where
\begin{align}
\tilde{\mathbf{A}}_{T,v} & =\sum_{t=1}^{T}t^{2}\mathbb{E}\Bigg\{\sum_{q\in\mathcal{Q}_{t}}\frac{\bm{l}_{q}\bm{l}_{q}^{\mathrm{T}}}{d_{t,q}^{4}}\Bigg\}+\sum_{t=1}^{T}t^{3}\mathbb{E}\Bigg\{\sum_{q\in\mathcal{Q}_{t}}\frac{(\bm{l}_{q}\mathbf{v}^{\mathrm{T}}+\mathbf{v}\bm{l}_{q}^{\mathrm{T}})}{d_{t,q}^{4}}\Bigg\}\nonumber \\
 & \quad\quad+\sum_{t=1}^{T}t^{4}\mathbb{E}\Bigg\{\sum_{q\in\mathcal{Q}_{t}}\frac{\mathbf{v}\mathbf{v}^{\mathrm{T}}}{d_{t,q}^{4}}\Bigg\}\label{eq:A-tilde-T-v}
\end{align}
and equality in (\ref{eq:ineq-Ftv}) can be achieved when $\sigma_{\mathrm{\max}}^{2}=\sigma_{q}^{2}$
and $\alpha_{\min}^{2}=\alpha_{q}^{2}$.

Since $\mathbf{F}_{T,v}$ is symmetric and positive semi-definite,
its eigenvalues are real and non-negative. Similar to (\ref{eq:tr-F-ineq}),
we have
\begin{equation}
\mathrm{tr}\{\mathbf{F}_{T,v}^{-1}\}\leq2\lambda_{\mathrm{min}}^{-1}(\mathbf{F}_{T,v})\leq2[\tilde{C}_{0}\lambda_{\mathrm{min}}(\tilde{\mathbf{A}}_{T,v})]^{-1}\triangleq\tilde{\Delta}_{T,v}.\label{eq:tr-F-ineq-1}
\end{equation}

\begin{lem}
\label{lem:A-Tx-lam-1-1}Under the same condition in Lemma \ref{lem:A-Tx-lam-1},
the eigenvalue of $\tilde{\mathbf{A}}_{T,v}$ satisfies
\[
\frac{\lambda_{\mathrm{min}}(\tilde{\mathbf{A}}_{T,x})}{T(T+1)(2T+1)}\rightarrow\frac{1}{6}\kappa\pi\ln(R/r_{0})
\]
as $T\to\infty$\foreignlanguage{english}{.}
\end{lem}
\begin{proof}Similar to the derivation of Lemma \ref{lem:A-Tx-lam-1},
the asymptotic larger eigenvector of $\tilde{\mathbf{A}}_{T,v}$ is
$\mathbf{u}=\mathbf{v}/\|\mathbf{v}\|_{2}$, because the last term
in (\ref{eq:A-tilde-T-v}) dominates when $T$ is large.

As a result, asymptotically, the smaller eigenvector is $\mathbf{v}_{\perp}$,
which satisfies $\mathbf{v}^{\mathrm{T}}\mathbf{v}_{\perp}=0$. Consequently,
from (\ref{eq:A-tilde-T-v}), as $T\to\infty$, we have
\begin{align}
 & \sum_{t=1}^{T}t^{2}\mathbb{E}\Bigg\{\sum_{q\in\mathcal{Q}_{t}}\frac{\mathbf{v}_{\perp}^{\mathrm{T}}\bm{l}_{q}\bm{l}_{q}^{\mathrm{T}}\mathbf{v}_{\perp}}{d_{t,q}^{4}}\Bigg\}+\sum_{t=1}^{T}t^{4}\mathbb{E}\Bigg\{\sum_{q\in\mathcal{Q}_{t}}\frac{\mathbf{v}_{\perp}^{\mathrm{T}}\mathbf{v}\mathbf{v}^{\mathrm{T}}\mathbf{v}_{\perp}}{d_{t,q}^{4}}\Bigg\}\nonumber \\
 & \quad\quad+\sum_{t=1}^{T}t^{3}\mathbb{E}\Bigg\{\sum_{q\in\mathcal{Q}_{t}}\frac{\mathbf{v}_{\perp}^{\mathrm{T}}(\bm{l}_{q}\mathbf{v}^{\mathrm{T}}+\mathbf{v}\bm{l}_{q}^{\mathrm{T}})\mathbf{v}_{\perp}}{d_{t,q}^{4}}\Bigg\}\nonumber \\
 & =\sum_{t=1}^{T}t^{2}\mathbb{E}\left\{ \sum_{q\in\mathcal{Q}_{t}}\frac{1}{d_{t,q}^{4}}\|\mathbf{P}_{v}^{\bot}\bm{l}_{q}\|^{2}\right\} \nonumber \\
 & =\frac{1}{6}T(T+1)(2T+1)\kappa\pi\ln(R/r_{0})\label{eq:E-plv}
\end{align}

Thus, we have
\begin{align*}
\frac{\lambda_{\mathrm{min}}(\tilde{\mathbf{A}}_{T,v})}{T(T+1)(2T+1)} & \rightarrow\frac{1}{6}\kappa\pi\ln(R/r_{0})
\end{align*}
as $T\to\infty$.\end{proof}

According Lemma \ref{lem:A-Tx-lam-1-1} and from (\ref{eq:tr-F-ineq-1}),
we have
\begin{align*}
 & T(T+1)(2T+1)\tilde{\Delta}_{T,v}\rightarrow\frac{12\sigma_{\max}^{2}}{\alpha_{\min}^{2}\kappa\pi\ln(R/r_{0})}
\end{align*}
as $T\to\infty$.

\section{Proof of Proposition \ref{prop:sparability}}

\label{sec:app-prop1}

Following the notations defined for the solutions (\ref{eq:solution-path-loss-estimation-2})–(\ref{eq:solution-path-loss-estimation-1})
in Section \ref{subsec:Path-Loss-Parameters}, we further define a
matrix $\mathbf{Y}'_{q}\in\mathbb{R}^{M\times T},$ where the $t$th
column is a collection of variables $y'_{q,m,t}$ for $m=1,2,\dots,M$
of the measurements at $\mathbf{x}_{t}$ over all the $M$ beams such
that $\mathbf{\bm{y}}_{q}'=\text{vec}(\mathbf{Y}_{q}')$. In addition,
denote $\bar{y}'_{q,t}=\sum_{m}y'_{q,m,t}$ and $\bar{\mathbf{y}}_{q}'\in\mathbb{R}^{T}$
as the collection of variables $\bar{y}'_{q,t}$ along the trajectory
$\mathbf{x}_{t}$. Then, we have $\mathbf{1}^{\text{T}}\mathbf{Y}_{q}'=(\bar{\mathbf{y}}_{q}')^{\text{T}}$.
Using the condition (\ref{eq:condition-uniform-distritubted-beam-pattern}),
we have 
\[
\bar{y}'_{q,t}=\sum_{m}y'_{q,m,t}=\bar{y}_{q,t}=\sum_{m}y_{q.m.t}-\bar{C}_{q}=\bar{y}_{q,t}-\bar{C}_{q}
\]
and consequently, in the vector form, $\bar{\mathbf{y}}_{q}'=\bar{\mathbf{y}}_{q}-\bar{C}_{q}\mathbf{1}$.

Due to the property of matrix operation with Kronecker product \cite{VanCha:J00},
$(\mathbf{A}\otimes\mathbf{B})(\mathbf{C}\otimes\mathbf{D})=(\mathbf{AC})\otimes(\mathbf{BD})$
and $(\mathbf{B}^{\text{T}}\otimes\mathbf{A})\mbox{vec}(\mathbf{X})=\mbox{vec}(\mathbf{A}\mathbf{X}\mathbf{B})$,
one can easily verify that 
\begin{align}
\tilde{\bm{\mathrm{D}}}_{q}^{\text{T}}\tilde{\bm{\mathrm{D}}}_{q} & =(\bm{\mathrm{D}}_{q}\otimes\bm{1})^{\text{T}}(\bm{\mathrm{D}}_{q}\otimes\bm{1})\nonumber \\
 & =(\bm{\mathrm{D}}_{q}^{\text{T}}\otimes\bm{1}^{\text{T}})(\bm{\mathrm{D}}_{q}\otimes\bm{1})=\bm{\mathrm{D}}_{q}^{\text{T}}\bm{\mathrm{D}}_{q}\label{eq:property-DqDq}
\end{align}
and 
\begin{align}
\tilde{\bm{\mathrm{D}}}_{q}^{\text{T}}\mathbf{y}_{q}' & =(\bm{\mathrm{D}}_{q}^{\text{T}}\otimes\bm{1}^{\text{T}})\mbox{vec}(\mathbf{Y}'_{q})=\text{\mbox{vec}}(\mathbf{1}^{\text{T}}\mathbf{Y}_{q}'\bm{\mathrm{D}}_{q})\nonumber \\
 & =\text{\mbox{vec}}((\bar{\mathbf{y}}_{q}')^{\text{T}}\bm{\mathrm{D}}_{q})=\bm{\mathrm{D}}_{q}^{\text{T}}\bar{\mathbf{y}}_{q}'\label{eq:property-Dqyq}
\end{align}
where the last equality is due to the fact that $(\bar{\bm{y}}_{q}')^{\text{T}}\bm{D}_{q}$
is a row vector.

Applying (\ref{eq:property-DqDq}) and (\ref{eq:property-Dqyq}) to
(\ref{eq:solution-path-loss-estimation-1}), the solution $\bm{\theta}_{q}^{(1)}=[\alpha_{q}^{(1)},\beta_{q}^{(1)}]^{\text{T}}$
to (\ref{eq:path-loss-parameter-estimation}) is written as 
\begin{align}
\hat{\bm{\theta}}_{q}^{(1)} & =(\tilde{\bm{\mathrm{D}}}_{q}^{\text{T}}\tilde{\bm{\mathrm{D}}}_{q})^{-1}\tilde{\bm{\mathrm{D}}}_{q}^{\text{T}}\mathbf{y}_{q}'\nonumber \\
 & =(\bm{\mathrm{D}}_{q}^{\text{T}}\bm{\mathrm{D}}_{q})^{-1}\bm{\mathrm{D}}_{q}^{\text{T}}\bar{\mathbf{y}}'_{q}\nonumber \\
 & =(\bm{\mathrm{D}}_{q}^{\text{T}}\bm{\mathrm{D}}_{q})^{-1}\bm{\mathrm{D}}_{q}^{\text{T}}(\bar{\mathbf{y}}_{q}-\bar{C}_{q}\mathbf{1})\nonumber \\
 & =(\bm{\mathrm{D}}_{q}^{\text{T}}\bm{\mathrm{D}}_{q})^{-1}\bm{\mathrm{D}}_{q}^{\text{T}}\bar{\mathbf{y}}_{q}-(\bm{\mathrm{\mathrm{D}}}_{q}^{\text{T}}\bm{\mathrm{D}}_{q})^{-1}\bm{\mathrm{D}}_{q}^{\text{T}}\bar{C}_{q}\mathbf{1}\label{eq:solution-path-loss-estimation-1b}\\
 & =[\alpha_{q}^{(2)},\beta_{q}^{(2)}]^{\text{T}}-(\bm{\mathrm{D}}_{q}^{\text{T}}\bm{\mathrm{D}}_{q})^{-1}\bm{\mathrm{D}}_{q}^{\text{T}}\bar{C}_{q}\mathbf{1}
\end{align}

While the first term in (\ref{eq:solution-path-loss-estimation-1b})
is identical to (\ref{eq:solution-path-loss-estimation-2}), we compute
the second term $(\bm{\mathrm{D}}_{q}^{\text{T}}\bm{\mathrm{D}}_{q})^{-1}\bm{\mathrm{D}}_{q}^{\text{T}}\bar{C}_{q}\mathbf{1}$
as follows. From the definition $\bm{\mathrm{D}}_{q}=[\bm{\mathrm{d}}_{q},\mathbf{1}]$
and the matrix inversion formula for a $2\times2$ matrix, the second
term in (\ref{eq:solution-path-loss-estimation-1b}) can be computed
as 
\begin{align*}
 & \frac{\bar{C}_{q}}{T\sum_{t}\tilde{d}_{q,t}^{2}-(\sum_{t}\tilde{d}_{q,t})^{2}}\left[\begin{array}{cc}
T & -\sum_{t}\tilde{d}_{q,t}\\
-\sum_{t}\tilde{d}_{q,t} & \sum_{t}\tilde{d}_{q,t}^{2}
\end{array}\right]\left[\begin{array}{c}
\bm{\mathrm{d}}_{q}^{\text{T}}\\
\mathbf{1}^{\text{T}}
\end{array}\right]\mathbf{1}\\
 & =\frac{\bar{C}_{q}}{T\sum_{t}\tilde{d}_{q,t}^{2}-(\sum_{t}\tilde{d}_{q,t})^{2}}\left[\begin{array}{c}
-\sum_{t}\tilde{d}_{q,t}+T\bm{\mathrm{d}}_{q}^{\text{T}}\\
\sum_{t}\tilde{d}_{q,t}^{2}-(\sum_{t}\tilde{d}_{q,t})\bm{\mathrm{d}}_{q}^{\text{T}}
\end{array}\right]\mathbf{1}\\
 & =\left[0,\bar{C}_{q}\right]
\end{align*}
which justifies that $\alpha_{q}^{(1)}=\alpha_{q}^{(2)}$ and $\beta_{q}^{(1)}=\beta_{q}^{(2)}-\bar{C}_{q}$.

\section{Proof of Proposition \ref{prop:Equivalence-Condition}}

\label{sec:app-prop2}

Since $\mathcal{T}_{q,m}^{\epsilon}=\{1,2,\dots,T\}$, consider the
following problem
\begin{equation}
\underset{\bm{\vartheta}}{\text{minimize}}\quad\sum_{t}\lambda_{t}\Big[\ln y_{q,m,t}''-\ln B(\bm{\vartheta};\phi_{q,t})\Big]^{2}.\label{eq:beam-pattern-estimation-logscale}
\end{equation}
Problems (\ref{eq:beam-pattern-estimation}) and (\ref{eq:beam-pattern-estimation-logscale})
have the same set of solutions if the derivative of the objective
functions are identical, i.e., 
\begin{align}
 & -\sum_{t}2(y'_{q,m,t}-B(\bm{\vartheta};\phi_{q,t}))\frac{\partial B(\bm{\vartheta};\phi_{q,t})}{\partial\bm{\vartheta}}\label{eq:equivalence-condition}\\
 & =-\sum_{t}2(\ln y'_{q,m,t}-\ln B(\bm{\vartheta};\phi_{q,t}))\frac{\lambda_{t}}{B(\bm{\vartheta};\phi_{q,t})}\frac{\partial B(\bm{\vartheta};\phi_{q,t})}{\partial\bm{\vartheta}}\nonumber 
\end{align}
where the first line is the derivative of (\ref{eq:beam-pattern-estimation})
and the second line is the derivative of (\ref{eq:beam-pattern-estimation-logscale}).
A sufficient condition to ensure (\ref{eq:equivalence-condition})
is to enforce equality on each term in (\ref{eq:equivalence-condition}),
resulting in the equation
\[
y'_{q,m,t}-B(\bm{\vartheta};\phi_{q,t})=(\ln y'_{q,m,t}-\ln B(\bm{\vartheta};\phi_{q,t}))\frac{\lambda_{t}}{B(\bm{\vartheta};\phi_{q,t})}
\]
which leads to the condition (\ref{eq:beam-pattern-estimation-weight}).

Finally, by substituting the variables (\ref{eq:solution-beam-pattern-from-b})
into the log-scale problem (\ref{eq:beam-pattern-estimation-logscale}),
one arrives at (\ref{eq:beam-pattern-estimation-auxiliary}). This
confirms the equivalence between (\ref{eq:beam-pattern-estimation})
and (\ref{eq:beam-pattern-estimation-auxiliary}) under condition
(\ref{eq:beam-pattern-estimation-weight}).

\bibliographystyle{IEEEtran}
\bibliography{my_ref}

\begin{IEEEbiography}[{\includegraphics[width=1in,height=1.25in,clip,keepaspectratio]{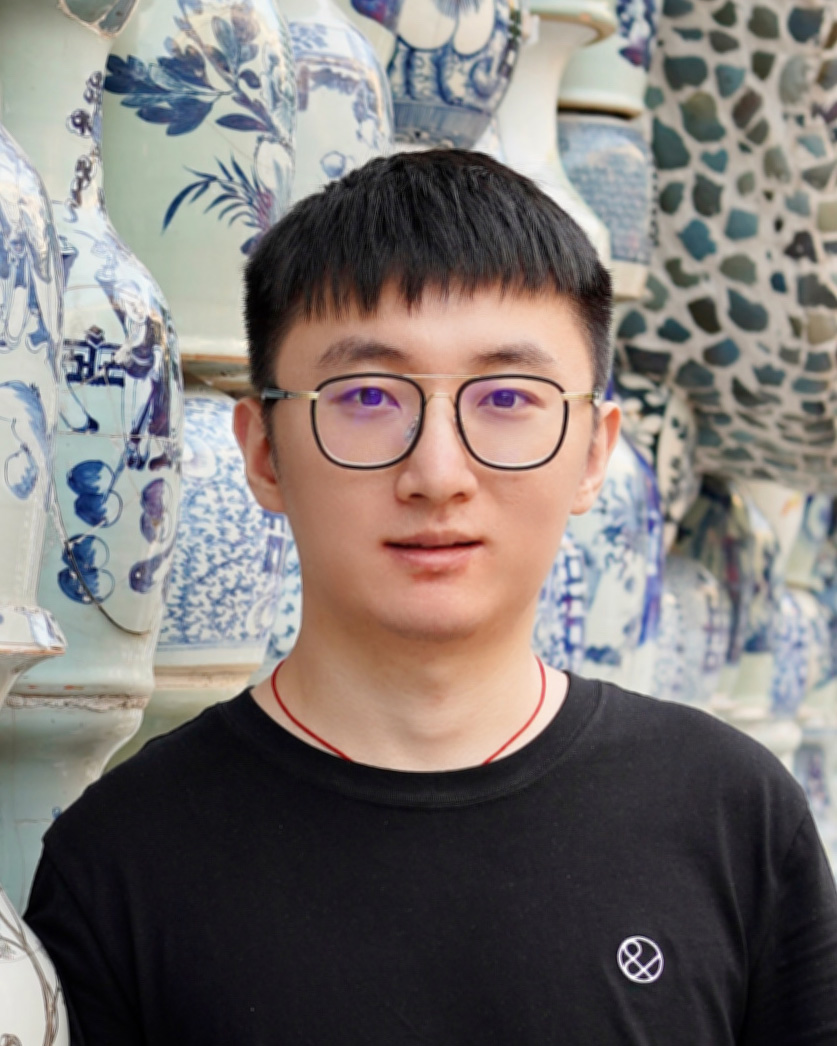}}]{Zheng Xing} (Member, IEEE) received the Ph.D. degree from The Chinese University of Hong Kong, Shenzhen, China, in 2025, the M.S. degree from Beihang University, Beijing, China, in 2020, and the B.S. degree from Ocean University of China, Qingdao, China, in 2017. He is currently an Assistant Professor with the College of Computer Science and Software Engineering, Shenzhen University. His research interests encompass optimization, mobile computing, and machine learning.

\end{IEEEbiography} 
\begin{IEEEbiography}[{\includegraphics[width=1in,height=1.25in,clip,keepaspectratio]{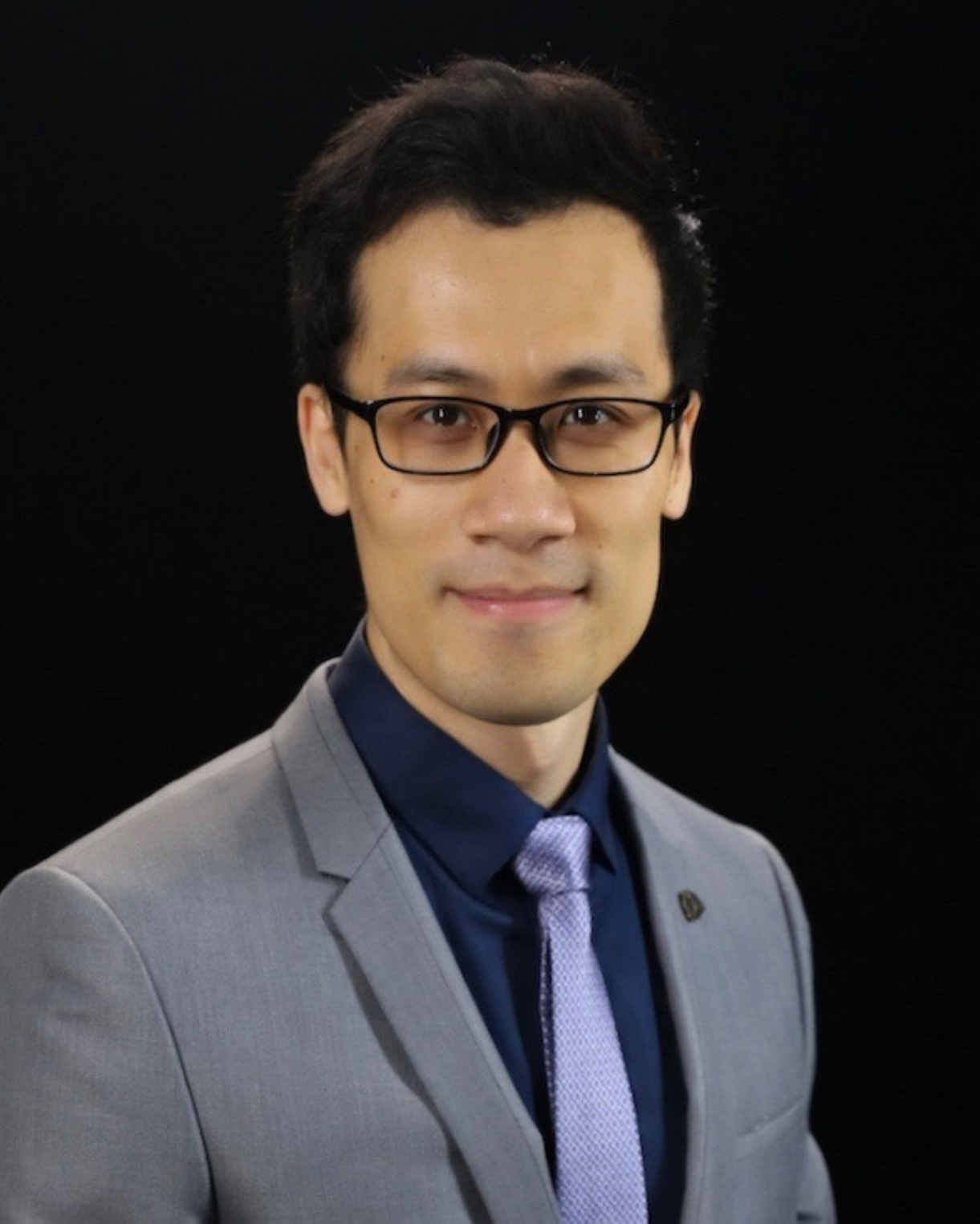}}]{Junting Chen} (S'11--M'16) received the Ph.D.\ degree in Electronic and Computer Engineering from the Hong Kong University of Science and Technology (HKUST), Hong Kong SAR China, in 2015, and the B.Sc.\ degree in Electronic Engineering from Nanjing University, Nanjing, China, in 2009. From 2014--2015, he was a visiting student with the Wireless Information and Network Sciences Laboratory at MIT, Cambridge, MA, USA.  
	
	He is an Assistant Professor with the School of Science and Engineering and the Future Network of Intelligence Institute (FNii) at The Chinese University of Hong Kong, Shenzhen (CUHK--Shenzhen), Guangdong, China. Prior to joining CUHK--Shenzhen, he was a Postdoctoral Research Associate with the Ming Hsieh Department of Electrical Engineering, University of Southern California (USC), Los Angeles, CA, USA, from 2016--2018, and with the Communication Systems Department of EURECOM, Sophia--Antipolis, France, from 2015--2016. His research interests include channel estimation, MIMO beamforming, machine learning, and optimization for wireless communications and localization. His current research focuses on radio map sensing, construction, and application for wireless communications. Dr. Chen was a recipient of the HKTIIT Post-Graduate Excellence Scholarships in 2012. He was nominated as the Exemplary Reviewer of {\scshape IEEE Wireless Communications Letters} in 2018. His paper received the Charles Kao Best Paper Award from WOCC 2022.
	
\end{IEEEbiography}

\end{document}